\documentclass{amsart}

\usepackage{amsmath,amssymb}
\usepackage{amsthm}
\usepackage{indentfirst}
\usepackage{mathrsfs}
\usepackage{graphicx}
\usepackage{hyperref}
\usepackage{xcolor}
\usepackage{fourier}
\usepackage[margin=1.25in]{geometry}

\numberwithin{equation}{section}

\theoremstyle{plain}
\newtheorem{theorem}{Theorem}[section]
\newtheorem{lemma}[theorem]{Lemma}
\newtheorem{prop}[theorem]{Proposition}

\theoremstyle{definition}
\newtheorem{defn}[theorem]{Definition}

\theoremstyle{remark}
\newtheorem*{remark}{Remark}
\newtheorem*{example}{Example}

\DeclareMathOperator{\KL}{D_{\mathsf{KL}}}
\DeclareMathOperator{\TV}{D_{\mathsf{TV}}}
\DeclareMathOperator{\BL}{D_{\mathsf{BL}}}
\DeclareMathOperator{\FI}{FI}

\DeclareMathOperator{\Law}{Law}
\DeclareMathOperator{\Cov}{Cov}
\DeclareMathOperator{\polylog}{polylog}

\DeclareMathOperator{\argmax}{arg\,max}
\DeclareMathOperator{\Tr}{Tr}

\newcommand{\R}{\mathbb{R}}

\newcommand{\E}{\mathbb{E}}
\newcommand{\Pp}{\mathbb{P}}
\newcommand{\dd}{\,\mathrm{d}}
\newcommand{\Id}{I}
\newcommand{\cF}{\mathcal{F}}

\newcommand{\cX}{\mathcal{X}}
\newcommand{\cA}{\mathcal{A}}

\newcommand{\mc}[1]{\mathcal{#1}}

\newcommand{\wh}[1]{\widehat{#1}}

\newcommand{\eps}{\varepsilon}

\newcommand{\deltaerr}{\varepsilon}
\newcommand{\norm}[1]{\left\lVert #1 \right\rVert}
\newcommand{\abs}[1]{\left\lvert #1 \right\rvert}
\newcommand{\ip}[2]{\left\langle #1,#2 \right\rangle}
\newcommand{\pdata}{p_{\mathsf{data}}}
\newcommand{\sstar}{\mathsf{s}^{\star}}
\newcommand{\score}{\mathsf{s}}
\newcommand{\dstar}{\mathsf{D}^{\star}}
\newcommand{\denoiser}{\mathsf{D}}
\newcommand{\epsinit}{\varepsilon_{\mathsf{init}}}
\newcommand{\epsdisc}{\varepsilon_{\mathsf{disc}}}

\newcommand{\epsscorek}{\varepsilon_{k,\mathsf{score}}}
\newcommand{\epsSE}{\varepsilon_{\mathsf{SE}}}
\newcommand{\normal}[2]{\mathcal{N}\!\left(#1,#2\right)}

\newcommand{\Poi}{\mathsf{Poisson}}

\newcommand{\tO}{\widetilde O}

\newcommand{\AlgFORS}{\textsf{FORS}}

\title{A Mathematical Introduction to Diffusion Models}
\author{Jianfeng Lu}
\address{Mathematics Department, Duke University, Box 90320, Durham, NC 27705 USA.}
\email{jianfeng@math.duke.edu}
\thanks{Lecture notes for the John Tukey Summer Graduate School on Mathematics of Generative Models at SLMath (June 22nd, 2026 -- July 2nd, 2026), co-organized and co-taught with Eric Vanden-Eijnden. We thank SLMath for the hospitality and students in the summer school for helpful feedback. The work is supported in part by the National Science Foundation via awards DMS-2309378 and IIS-2403276.}
\date{July 2nd, 2026}

\begin{document}

\begin{abstract}
These notes give a proof-oriented introduction to diffusion models from the
viewpoint of sampling, tracing a single arc from classical sampling dynamics to
modern diffusion samplers, their error analysis, and inference-time control.
They proceed in five movements.  The first develops the sampling language and
the Langevin toolkit---target distributions and discrepancies, Markov kernels,
Fokker--Planck evolution, and entropy dissipation---and turns it into
convergence guarantees for Langevin diffusion, the unadjusted Langevin
algorithm, and its Metropolis-adjusted correction.  The second builds
continuous-time score-based diffusion models through Gaussian noising, Tweedie's
identity, the reverse-time SDE, and the probability flow ODE, and revisits the
same Gaussian channel through stochastic localization and Polchinski flow.  The
third discretizes these continuous dynamics into implementable DDPM samplers---
exact and Gaussian reverse kernels, the denoising-score equivalence---and
carries out a sampling-error analysis that separates early stopping, KL
telescoping, and score error, treated across Euler--Maruyama, Hessian control,
and high-accuracy first-order rejection sampling.  The fourth develops discrete
diffusion on finite state spaces, where continuous-time Markov chains replace
the noising SDE and the reverse kernel and its error analysis are recast in
finite-state form.  The fifth turns to inference-time steering of a trained
model: guidance, reward tilting, path-space control, and inference-time
reinforcement learning.
Throughout, the material is layered into core definitions and identities proved
in full, representative estimates proved under simplifying assumptions, and
research-level theorems stated with a proof roadmap.  The intended audience is
beginning graduate students with a background in probability but no prior
exposure to stochastic differential equations, stochastic numerics, or diffusion models.
\end{abstract}

\maketitle

\setcounter{tocdepth}{1}
\tableofcontents

\subsection*{How to read these notes}

The intended audience is beginning graduate students who have seen probability,
linear algebra, and multivariable calculus, but may not have prior exposure to
stochastic differential equations or diffusion models.  The notes therefore
separate three levels of material.
\begin{itemize}
    \item \emph{Core definitions and identities} are proved in detail.  These
    run through every section and include the Fokker--Planck equation and
    entropy dissipation, Tweedie's identity and the reverse SDE, the
    probability flow ODE, the exact and Gaussian DDPM reverse kernels, the
    denoising-score equivalence, the finite-state reverse kernel of discrete
    diffusion, and the reward-tilting identity.
    \item \emph{Representative estimates} are proved under simplifying
    assumptions, to show a mechanism in its cleanest form.  For example, we
    prove KL contraction under a log-Sobolev inequality, compute the ULA bias
    exactly for a Gaussian target, and bound the one-step discretization and
    score-error contributions to the diffusion-model KL.
    \item \emph{Research-level theorems} are stated in simplified form with a
    proof roadmap.  These include the ULA and MALA convergence guarantees,
    high-accuracy diffusion sampling via first-order rejection sampling and the
    discrete-diffusion error analysis.  The goal is not to reproduce every
    technical lemma of the original papers, but to make clear what each theorem
    is saying, why the assumptions appear, and how the proof is organized.
\end{itemize}

Two appendices collect the analytic tools used repeatedly: It\^o calculus and
the Girsanov change-of-measure formula in Appendix~\ref{app:ito-girsanov}, and
the Gaussian identities in Appendix~\ref{app:gaussian-toolbox}.  A reader
comfortable with these may skip them and refer back as needed.

These notes are self-contained but deliberately selective, and they overlap
with several excellent treatments that a reader may wish to consult in
parallel.  Yuansi Chen's lecture notes for the ETH course \emph{Computational
and Statistical Aspects of Diffusion Models}~\cite{chen2026lecturenotes} cover
much of the same ground with a stronger emphasis on convergence proofs,
guidance, and discrete diffusion, and are a useful companion to the later
sections here.  For the Langevin and log-concave sampling side,
Chewi's book~\cite{chewi2024logconcave} is the standard reference. For stochastic calculus and numerical analysis of SDEs and CTMCs, we refer to the book~\cite{e2019appliedstochastic} by E, Li, and Vanden-Eijnden.  Throughout,
when a result first appeared in a research paper we cite the original; the
lecture notes above are flagged here once, as a unified entry point to several
of the topics that follow.

These notes do not try to survey the full modern generative-modeling
landscape.  We focus on diffusion and Langevin mechanisms for which reverse-time
SDEs, probability-flow ODEs, score-estimation identities, and inference-time
control are the central actors.  The notes treat learned-score error as an input
to sampler guarantees, but do not discuss statistical learning theory for estimating scores from finite data.  We also do
not give detailed treatments of likelihood-based normalizing flows, adversarial
and variational models, autoregressive architectures, large-scale
latent-diffusion engineering, or several newer continuous-time alternatives.
For entry points to the latter, see flow matching
\cite{lipman2022flowmatching}, rectified flows \cite{liu2022rectifiedflow},
stochastic interpolants \cite{albergo2023stochasticinterpolants},
consistency models \cite{song2023consistency}, and recent one-step
formulations such as mean flow \cite{geng2025meanflows}.  These
methods share much of the same language of probability paths, transport
equations, denoisers, and ODE/SDE samplers.

\subsection*{Minimal probability background}

We use the following conventions repeatedly.  If $X$ is a random variable, then
$\Law(X)$ denotes its distribution.  If $X$ has density $p$, expectations are
written either as $\E[f(X)]$ or $\int f(x)p(x)\dd x$.  A conditional expectation
$\E[Y\mid X=x]$ is best thought of as the best prediction of $Y$ as a function of the observed value $x$.  For
a Markov chain, $P(x,\dd y)$ denotes the distribution of the next state given
the current state $x$.  For an SDE, all formal differentiations can be justified
under smoothness and decay assumptions; the notes use this formal calculus as a
way to keep the main ideas visible.

We often use the same symbol for a probability law and its density when the
meaning is clear from context.  For example, $\pdata$ may denote the data law,
the density of that law, or the measure $\pdata(\dd x)$.  Similarly, $p_t$ may
denote either $\Law(X_t)$ or its density.  This common abuse of notation keeps
the formulas readable, but when a density value is needed we write expressions
such as $\nabla\log p_t(x)$.

\subsection*{Recurring notation}

The following symbols are used across several sections.  More local notation is
introduced where it is needed.
\begin{center}
\begin{tabular}{c|p{0.70\textwidth}}
Symbol & Meaning \\ \hline
$\pdata$ & data distribution, the target of generative modeling.\\
$X_0$ & clean data sample, $X_0\sim\pdata$.\\
$X_t,p_t$ & continuous-time forward-noised variable and its law/density.\\
$B_t$ & Brownian motion driving the forward noising SDE.\\
$a_t,\sigma_t$ & continuous Gaussian marginal parameters:
$X_t\mid X_0\sim\normal{a_tX_0}{\sigma_t^2\Id}$.\\
$\sstar_t,\score_t$ & true and learned continuous-time scores.\\
$Y^{\leftarrow}_s,B^{\leftarrow}_s$ & reverse-time SDE sample path and its
reverse Brownian motion, in reverse time $s=T-t$ with $Y^{\leftarrow}_0\sim p_T$.\\
$Z^{\leftarrow}_s$ & reverse-time probability-flow ODE sample path,
$Z^{\leftarrow}_0\sim p_T$.\\
$X_k$ & forward-noised variable at discrete time $k$.\\
$p_k$ & density/law of $X_k$.\\
$\alpha_k,\eta_k$ & one-step noising parameters in
$X_{k+1}\mid X_k\sim\normal{\alpha_kX_k}{\alpha_k^2\eta_k\Id}$; $\eta_k$ is the
one-step variance parameter.\\
$h_k$ & reverse-step time increment $t_{k+1}-t_k$ in the discretization
analysis; equals $\eta_k$ in the variance-exploding case.\\
$a_k,\sigma_k$ & marginal parameters:
$X_k\mid X_0\sim\normal{a_kX_0}{\sigma_k^2\Id}$.\\
$P_k^{\to}(x \to x')$ & forward transition density
$\Law(X_{k+1}\mid X_k=x)$.\\
$P_k^{\leftarrow}(x' \to \cdot)$ & exact backward transition density
$\Law(X_k\mid X_{k+1}=x')$.\\
$\sstar_k$ & true score $\nabla\log p_k$.\\
$\score_k$ & learned score estimate.\\
$\epsscorek^2$ & mean-square score error
$\E_{p_k}\norm{\score_k-\sstar_k}^2$.\\
\end{tabular}
\end{center}

\subsection*{Algorithm acronyms}

The following acronyms name algorithms or algorithmic discretizations used in
several sections.
\begin{center}
\begin{tabular}{c|p{0.70\textwidth}}
Acronym & Meaning \\ \hline
MCMC & Markov chain Monte Carlo, the strategy of sampling by running a Markov
chain whose long-run law is the target.\\
ULA & unadjusted Langevin algorithm, the Euler--Maruyama discretization of
overdamped Langevin diffusion.\\
MALA & Metropolis-adjusted Langevin algorithm, ULA with an accept-reject
correction using density ratios.\\
DDPM & denoising diffusion probabilistic model, the discrete Gaussian noising
and learned reverse-sampling framework.\\
DDIM & denoising diffusion implicit model, a deterministic or partially
stochastic DDPM-compatible sampler closely related to probability-flow
discretization.\\
EM & Euler--Maruyama, the basic time-discretization rule for SDEs.\\
\AlgFORS & first-order rejection sampling, a score-only correction mechanism
for Gaussian tilts.\\
\end{tabular}
\end{center}

\subsection*{AI usage disclosure}

Large-language-model tools were used in the preparation of these notes to help
with drafting, editing, reorganization, and consistency checks.  The
mathematical content, exposition choices, references, and any remaining errors
are the responsibility of the author.

\section{Introduction to Sampling and Langevin Dynamics}

Sampling is the problem of producing random points that look as if they were
drawn from a given distribution.  This is easy to state but often hard to do:
the distribution may be known only up to a normalizing constant, or only
through data, and drawing from it directly may be infeasible.  The strategy
running through these notes is indirect.  Rather than sample the hard target
in one shot, we build a random process that is easy to simulate and whose
distribution gradually drifts toward the target, and we read off a sample once
the process has run long enough.  Making this idea precise calls for one habit
above all: instead of following a single trajectory, track the whole
distribution as it moves through the algorithm.  This first section
deliberately goes slowly, fixing that habit and the language we use to measure
how far one distribution is from another.

\subsection{The sampling problem}\label{subsec:sampling-problem}

Let $\pi$ be a probability distribution on $\R^d$.  In most applications,
$\pi$ is not available through exact samples.  Instead we may have access to:
\begin{itemize}
    \item the unnormalized density $\pi(x)\propto e^{-U(x)}$;
    \item the gradient $\nabla\log\pi(x)=-\nabla U(x)$;
    \item samples from a noisy distribution related to $\pi$;
\end{itemize}
The goal is to construct a random variable $\wh X$ with law $\wh\pi$ such that
$D(\wh\pi,\pi)\leq \deltaerr$, where $D$ is an appropriate discrepancy.  There
is no single best choice of $D$. Let $\mu$ and $\nu$ be two probability laws on $\R^d$.
We recall three basic discrepancies.
\begin{enumerate}
\item Total variation asks whether every event has nearly the right
probability:
\[
    \TV(\mu,\nu)=\sup_A\abs{\mu(A)-\nu(A)}
    =\frac12\int\abs{\frac{\dd\mu}{\dd\lambda}
    -\frac{\dd\nu}{\dd\lambda}}\dd\lambda.
\]
Here $\lambda$ is any measure such that both $\mu$ and $\nu$ are absolutely
continuous with respect to $\lambda$, written $\mu,\nu\ll\lambda$; for example,
we may take $\lambda=\mu+\nu$.  The notation $\dd\mu/\dd\lambda$ denotes the
Radon--Nikodym derivative, the density of $\mu$ with respect to $\lambda$,
characterized by
\[
    \mu(A)=\int_A \frac{\dd\mu}{\dd\lambda}\,\dd\lambda
\]
for measurable sets $A$.  The same applies to $\dd\nu/\dd\lambda$, and the
value of the total-variation distance does not depend on which such reference
measure $\lambda$ is chosen.
\item Wasserstein-$2$ distance asks whether probability mass can be
transported a short geometric distance:
\[
    W_2^2(\mu,\nu)=\inf_{\gamma\in\Pi(\mu,\nu)}
    \int \norm{x-y}^2\,\gamma(\dd x,\dd y).
\]
Here $\Pi(\mu,\nu)$ is the set of couplings of $\mu$ and $\nu$, meaning
probability measures $\gamma$ on $\R^d\times\R^d$ whose first marginal is
$\mu$ and whose second marginal is $\nu$.
\item KL divergence asks whether $\mu$ can be encoded efficiently using $\nu$
as a reference:
\[
    \KL(\mu\|\nu)
    =
    \int \log\!\left(\frac{\dd\mu}{\dd\nu}\right)\dd\mu,
\]
when $\mu$ is absolutely continuous with respect to $\nu$, and
$+\infty$ otherwise.  Note that the KL divergence is asymmetric:
$\KL(\mu\|\nu) \neq \KL(\nu\|\mu)$.  In particular, KL is not a metric.
\end{enumerate}

KL is often the main bookkeeping divergence we will use in this course.  The other metrics translate this information into statements
about events, geometric displacement, or weak test functions.
The basic comparison begins with the Csisz\'ar--Kullback--Pinsker inequality;
see, for example, Bakry, Gentil, and Ledoux~\cite{bakry2014analysis}:
\[
    \TV(\mu,\nu)
    \leq
    \sqrt{\frac12\KL(\mu\|\nu)}.
\]
Thus a KL guarantee immediately gives a TV guarantee, and hence also a weak
test-function guarantee.  Another recurring rule is \emph{data processing}:
applying the same observation map to two random objects cannot increase KL\@.  If
$T$ is measurable and $T_\#\mu$ denotes the law of $T(X)$ for $X\sim\mu$, then
\[
    \KL(T_\#\mu\|T_\#\nu)\leq \KL(\mu\|\nu).
\]
The same monotonicity holds for total variation.  We use this principle
repeatedly to pass from path-space comparisons to endpoint laws; a proof is
given in Lemma~\ref{lem:data-processing}.

Wasserstein distance is more geometric.  If $\mu$
and $\nu$ are supported on a set of diameter $R$, then a maximal coupling gives
\[
    W_2^2(\mu,\nu)
    \leq
    R^2\TV(\mu,\nu)
    \leq
    R^2\sqrt{\frac12\KL(\mu\|\nu)}.
\]
Without a bounded-diameter or moment/functional-inequality assumption, there
is no universal comparison in the other direction: small $W_2$ need not imply
small TV or KL, and small TV or KL need not control $W_2$ if a tiny amount of
mass can move very far away.

A more powerful comparison is available relative to a fixed reference law.  A
probability law $\nu$ satisfies a Talagrand $T_2$ transport-entropy inequality
with constant $C_T$ if
\[
    W_2^2(\mu,\nu)\leq 2C_T\KL(\mu\|\nu)
    \qquad\text{for all }\mu\ll\nu.
\]
One standard route to this inequality is through log-Sobolev.  In the convention
of these notes, if $\nu$ satisfies
\[
    \KL(\mu\|\nu)
    \leq
    \frac{C_{\mathsf{LSI}}}{2}\FI(\mu\|\nu),
\]
where
\[
    \FI(\mu\|\nu)
    =
    \int \norm{\nabla\log\frac{\dd\mu}{\dd\nu}}^2\dd\mu
\]
when the derivative is well defined, then the Otto--Villani theorem gives
the transport bound
\[
    W_2^2(\mu,\nu)\leq 2C_{\mathsf{LSI}}\KL(\mu\|\nu).
\]
This is one reason log-Sobolev and transport inequalities appear naturally in
sampling theory; see Villani~\cite{villani2009optimal} and Bakry, Gentil, and
Ledoux~\cite{bakry2014analysis}.

\begin{example}[Three distances see different things]
Let $\mu=\delta_0$ and $\nu=\delta_\epsilon$ on $\R$.  Then
$\TV(\mu,\nu)=1$ for every $\epsilon\neq0$, because the two point masses live
on disjoint sets.  On the other hand, $W_2(\mu,\nu)=\epsilon$ and
$\KL(\mu\|\nu)=\KL(\nu\|\mu)=+\infty$.  Thus TV and KL are too strict for
comparing a point mass to a tiny displacement of itself, while Wasserstein
sees that the two laws are geometrically close.
\end{example}

\subsection{Markov chain Monte Carlo}

Markov chain Monte Carlo (MCMC) turns sampling into the problem of running a
Markov chain.  Instead of drawing directly from $\pi$, we choose a transition
rule whose repeated application should have $\pi$ as its equilibrium law.  A
Markov kernel $P$ maps a current state $x$ to a distribution $P(x,\cdot)$.  If
$\pi P=\pi$, then $\pi$ is invariant, and the ideal MCMC chain runs
\[
    X_{n+1}\sim P(X_n,\cdot),\qquad n=0,1,\ldots,
\]
with the hope that $\Law(X_n)$ approaches $\pi$ as $n\to\infty$.

This viewpoint immediately separates two questions.  First, is the transition
kernel designed so that the right law is invariant?  Second, how long must the
chain run before it is close to equilibrium?  For a genuine metric $D$, such as
TV or $W_2$, a typical algorithmic decomposition is
\[
    D(\Law(X_n),\pi)
    \leq
    \underbrace{D(\Law(X_n),\pi_h)}_{\text{mixing to the algorithm's invariant law}}
    +
    \underbrace{D(\pi_h,\pi)}_{\text{bias from approximation}},
\]
where $\pi_h$ is the invariant distribution of the implementable algorithm.
This is a triangle-inequality argument, so it applies to genuine metrics but
not to KL\@.  In general there is no inequality of the form
$\KL(\mu\|\rho)\leq\KL(\mu\|\nu)+\KL(\nu\|\rho)$.

This decomposition separates the error caused by not running the algorithm long
enough from the error caused by using a transition rule whose invariant law is
not exactly $\pi$.  For Langevin sampling, that second error appears when a
continuous diffusion is replaced by a time-stepping rule.  The discretization
may lead to an implementable chain with $\pi_h\neq\pi$.  In exact MCMC methods,
such as Metropolis-adjusted Langevin, the transition is constructed so that
$\pi_h=\pi$.

\begin{defn}[Invariant and reversible kernels]
A probability distribution $\pi$ is \emph{invariant} for $P$ if
\[
    \int \pi(\dd x)P(x,A)=\pi(A)
    \qquad\text{for every measurable set }A.
\]
It is \emph{reversible} for $P$ if the detailed balance identity
\[
    \pi(\dd x)P(x,\dd y)=\pi(\dd y)P(y,\dd x)
\]
holds as measures on pairs $(x,y)$.  Reversibility implies invariance by
integrating both sides over $x$.
\end{defn}

\paragraph{Caution: invariance is not convergence.}
Invariance is only a fixed-point statement: if $X_0\sim\pi$, then one step of
the chain still has law $\pi$.  It does not by itself say that a chain started
from another law will approach $\pi$.  For example, the identity kernel
$P(x,\cdot)=\delta_x$ leaves every distribution invariant, but it never mixes.
Even uniqueness of the invariant law is not quite enough without excluding
periodic behavior: on the two-point space $\{0,1\}$, the deterministic flip
$0\mapsto1$, $1\mapsto0$ has the uniform distribution as an invariant reversible
law, but a chain started from $0$ oscillates forever.

Thus detailed balance is a convenient way to verify invariance, not a
convergence theorem.  To justify MCMC, one also needs an ergodicity mechanism,
such as irreducibility and aperiodicity in finite state spaces, or the
appropriate Harris recurrence and minorization conditions in general state
spaces.  Quantitative sampling bounds then add still more structure: spectral
gaps, conductance, contraction, or functional inequalities.  See Meyn and
Tweedie~\cite{meyn2009markov} for a classical treatment of these Markov-chain
conditions.

A common design principle for MCMC is therefore to start one level above the
implementable chain.  First design a continuous-time Markov process whose
invariant law is $\pi$ and whose convergence mechanism can be analyzed.  Then
discretize that process to obtain a computable Markov chain, possibly adding a
Metropolis correction if one wants to remove discretization bias.  Langevin
dynamics is the canonical example of this principle: the continuous diffusion
has $\pi$ as its invariant law, while ULA and MALA are two different ways to
turn that diffusion into an algorithm.

\subsection{Overdamped Langevin diffusion}

Assume $\pi(x)\propto e^{-U(x)}$ with $U:\R^d\to\R$ smooth.  The overdamped
Langevin diffusion is the continuous-time Markov process
\begin{equation}\label{eq:langevin}
    \dd X_t = -\nabla U(X_t)\dd t+\sqrt{2}\,\dd B_t
    = \nabla\log\pi(X_t)\dd t+\sqrt{2}\,\dd B_t.
\end{equation}
The drift
$-\nabla U=\nabla\log\pi$ points toward regions where $\pi$ is larger, much as
gradient descent moves toward lower potential.  The diffusion coefficient in
the Brownian term is tuned so that the ensemble keeps precisely the spread
prescribed by $\pi$.

This is the basic difference between optimization and sampling.  Optimization
asks for a point $x_\star$ with small objective value, whereas sampling asks for
a law.  If $\pi$ is bimodal, with half of its mass near each of two separated
modes, an optimizer may correctly return one mode, but a sampler must visit both
modes with the correct frequencies.

There are two complementary ways to read \eqref{eq:langevin}.  Pathwise, each
particle is pulled down the potential landscape by $-\nabla U$ while Brownian
motion keeps the ensemble spread out.  Distributionally, the density $q_t$ of
$X_t$ evolves by a deterministic PDE\@.  The corresponding Fokker--Planck
equation is
\begin{equation}\label{eq:fokker-planck}
    \partial_t q_t
    = \nabla\cdot(q_t\nabla U)+\Delta q_t
    = \nabla\cdot\!\left(q_t\nabla\log\frac{q_t}{\pi}\right).
\end{equation}
At this stage the point of \eqref{eq:fokker-planck} is simply that the
stochastic particle system has a deterministic law-level description.  The SDE
picture is good for intuition and simulation.  The PDE picture is good for
proving invariance and entropy decay.

To derive the Fokker--Planck equation, we recall It\^o calculus in the form needed for our
calculation, more details can be found in Appendix~\ref{app:ito-girsanov}.  If
\[
    \dd X_t=b(X_t)\dd t+\sigma\,\dd B_t
\]
with constant diffusion matrix $\sigma$, then for a smooth test function
$\varphi$,
\[
    \dd \varphi(X_t)
    =
    \ip{\nabla\varphi(X_t)}{b(X_t)}\dd t
    +\frac12\Tr\!\left(\sigma\sigma^\top\nabla^2\varphi(X_t)\right)\dd t
    +\ip{\nabla\varphi(X_t)}{\sigma\,\dd B_t}.
\]
The last term is a martingale increment, so it disappears after taking
expectations.  For \eqref{eq:langevin}, $b=-\nabla U$ and
$\sigma=\sqrt2\,\Id$.
Applying the preceding formula to
\eqref{eq:langevin} gives
\[
    \frac{\dd}{\dd t}\E[\varphi(X_t)]
    =
    \E[-\ip{\nabla U(X_t)}{\nabla\varphi(X_t)}+\Delta\varphi(X_t)].
\]
If $X_t$ has density $q_t$, then
\[
    \frac{\dd}{\dd t}\int \varphi q_t\dd x
    =
    \int\left[-\ip{\nabla U}{\nabla\varphi}+\Delta\varphi\right]q_t\dd x.
\]
Integrating by parts, assuming boundary terms vanish,
\[
    \int -q_t\ip{\nabla U}{\nabla\varphi}\dd x
    =
    \int \varphi\,\nabla\cdot(q_t\nabla U)\dd x,
    \qquad
    \int q_t\Delta\varphi\dd x
    =
    \int \varphi\,\Delta q_t\dd x.
\]
Since this holds for all test functions $\varphi$, we obtain
\[
    \partial_tq_t=\nabla\cdot(q_t\nabla U)+\Delta q_t.
\]
Finally, $\pi\propto e^{-U}$ implies $\nabla U=-\nabla\log\pi$, and
\[
    \nabla\cdot(q_t\nabla U)+\Delta q_t
    =
    \nabla\cdot\left(q_t\nabla\log\frac{q_t}{\pi}\right).
\]

The differential operator that appeared in the test-function calculation,
\[
    \mathcal L\varphi
    =
    -\ip{\nabla U}{\nabla\varphi}+\Delta\varphi,
\]
is called the infinitesimal generator of the Langevin diffusion.  The
Fokker--Planck equation is the adjoint equation on densities.  In weak form,
the calculation above says
\[
    \frac{\dd}{\dd t}\int \varphi q_t\dd x
    =
    \int (\mathcal L\varphi) q_t\dd x
    =
    \int \varphi(\mathcal L^\ast q_t)\dd x,
\]
so $\partial_t q_t=\mathcal L^\ast q_t$, where
\[
    \mathcal L^\ast q
    =
    \nabla\cdot(q\nabla U)+\Delta q.
\]

\begin{prop}[Stationarity]
The target density $\pi\propto e^{-U}$ satisfies
\[
    \mathcal L^\ast \pi=0.
\]
Consequently, if $q_0=\pi$, then the solution of the Fokker--Planck equation
remains $q_t=\pi$ for all $t$.  Thus $\pi$ is an invariant law of
\eqref{eq:langevin}.
\end{prop}

\begin{proof}
Since $\nabla\log\pi=-\nabla U$, we have $\nabla\pi=-\pi\nabla U$.  Therefore
\[
    \mathcal L^\ast\pi
    =
    \nabla\cdot(\pi\nabla U)+\Delta\pi
    =
    \nabla\cdot(\pi\nabla U+\nabla\pi)
    =
    0.
\]
Thus the right-hand side of the adjoint equation
$\partial_t q_t=\mathcal L^\ast q_t$ vanishes at $q_t=\pi$, so the density
$\pi$ is stationary.
\end{proof}

\begin{example}[Ornstein--Uhlenbeck process]
Take $\pi=\normal{0}{\Id}$, so $U(x)=\norm{x}^2/2$.  Langevin dynamics becomes
\[
    \dd X_t=-X_t\dd t+\sqrt2\,\dd B_t.
\]
The explicit solution is
\[
    X_t=e^{-t}X_0+\sqrt2\int_0^t e^{-(t-s)}\dd B_s.
\]
If $X_0$ is deterministic, then
\[
    X_t\sim\normal{e^{-t}X_0}{(1-e^{-2t})\Id}.
\]
This example is worth remembering: the mean contracts exponentially and the
variance fills in to the target variance.
\end{example}

\subsection{Entropy dissipation}

Differentiating KL along the flow gives
\begin{equation}\label{eq:entropy-dissipation}
    \frac{\dd}{\dd t}\KL(q_t\|\pi)
    =
    -\int \norm{\nabla\log\frac{q_t}{\pi}}^2 q_t\,\dd x
    =-\FI(q_t\|\pi),
\end{equation}
where $\FI(q\|\pi)$ is the relative Fisher information.

\begin{proof}[Proof of \eqref{eq:entropy-dissipation}]
We have
\[
    \KL(q_t\|\pi)=\int q_t\log\frac{q_t}{\pi}\dd x.
\]
Because $\int \partial_t q_t\dd x=0$, differentiating gives
\[
    \frac{\dd}{\dd t}\KL(q_t\|\pi)
    =
    \int \partial_tq_t\,\log\frac{q_t}{\pi}\dd x.
\]
Using the equivalent form of the Fokker--Planck equation,
\[
    \partial_tq_t
    =
    \nabla\cdot\!\left(q_t\nabla\log\frac{q_t}{\pi}\right),
\]
and integrating by parts,
\[
    \int
    \nabla\cdot\!\left(q_t\nabla\log\frac{q_t}{\pi}\right)
    \log\frac{q_t}{\pi}\dd x
    =
    -\int q_t\norm{\nabla\log\frac{q_t}{\pi}}^2\dd x.
\]
This is exactly $-\FI(q_t\|\pi)$.
\end{proof}

The identity just proved becomes a quantitative convergence estimate once it is
combined with a functional inequality.  If $\pi$ satisfies a log-Sobolev
inequality
\[
    \KL(q\|\pi)\leq \frac{C_{\mathsf{LSI}}}{2}\FI(q\|\pi),
\]
then \eqref{eq:entropy-dissipation} implies
\[
    \frac{\dd}{\dd t}\KL(q_t\|\pi)
    \leq
    -\frac{2}{C_{\mathsf{LSI}}}\KL(q_t\|\pi).
\]
By Gr\"onwall's inequality,
\[
    \KL(q_t\|\pi)
    \leq e^{-2t/C_{\mathsf{LSI}}}\KL(q_0\|\pi).
\]
Thus the continuous-time process reduces KL error exponentially.

In particular, if the potential $U$ is $m$-strongly convex, that is
$\nabla^2U\succeq m\Id$ with $m>0$, then by the Bakry--\'Emery
criterion~\cite{bakry2014analysis} the target satisfies the log-Sobolev
inequality with $C_{\mathsf{LSI}}\leq 1/m$, and the estimate above becomes the
clean exponential rate
\[
    \KL(q_t\|\pi)\leq e^{-2mt}\KL(q_0\|\pi).
\]

The main lesson is the mechanism: KL
divergence is a Lyapunov function, the relative Fisher information is its
dissipation rate, and a structural inequality for the target converts
dissipation into convergence.  Sampling arguments often follow this pattern:
choose a discrepancy, compute its derivative along the dynamics, and use an
inequality for the target law to turn the derivative into a rate.

The same identity also has a geometric reading.  In optimal-transport language,
overdamped Langevin is the $W_2$-gradient flow of $\KL(\cdot\|\pi)$, and
\eqref{eq:entropy-dissipation} is its energy-dissipation identity; see
Ambrosio, Gigli, and Savar\'e~\cite{ambrosio2008gradientflows}.  For background
on diffusion semigroups, entropy dissipation, and log-Sobolev inequalities, see
Bakry, Gentil, and Ledoux~\cite{bakry2014analysis}; for the transport viewpoint
behind Wasserstein geometry, see Villani~\cite{villani2009optimal}.

\section{Convergence of Langevin Diffusion and ULA}

The previous section established continuous-time convergence through the entropy identity and log-Sobolev inequality.  We now ask what remains of that story after replacing the diffusion by an implementable Markov chain.  The main new issue is numerical error: a time-stepping rule may converge as a Markov chain, but its invariant law need not be the target law.  For a modern treatment of log-concave sampling algorithms, including Langevin algorithms and their Metropolis-adjusted variants, see the excellent book by Chewi~\cite{chewi2024logconcave}.

\subsection{Unadjusted Langevin algorithm}

Euler--Maruyama with step size $h>0$ gives the unadjusted Langevin algorithm (ULA)
\begin{equation}\label{eq:ula}
    X_{n+1}=X_n-h\nabla U(X_n)+\sqrt{2h}\,\xi_n,
    \qquad \xi_n\sim\normal{0}{\Id}.
\end{equation}
ULA is often the first sampler one considers because it is simple, gradient-based, and cheap.  Its limitation is equally important: \eqref{eq:ula} is not an exact transition of \eqref{eq:langevin}, and in general does not preserve $\pi$.

It is tempting to view ULA as ``almost Langevin'' and therefore assume that it must have almost the right stationary law.  Already for a Gaussian target one can see exactly what ``almost'' means.  For small step size the bias is small, but it does not vanish at fixed $h$ no matter how long the chain is run.  This is the difference between mixing error, which decreases with more iterations, and discretization bias, which is built into the transition rule itself.

\begin{example}[ULA bias for a standard Gaussian]
Let $U(x)=x^2/2$ in one dimension, so the target is $\pi=\normal{0}{1}$.
ULA becomes the Markov chain
\[
    X_{n+1}=(1-h)X_n+\sqrt{2h}\xi_n.
\]
Recall from the Gaussian toolbox, Lemma~\ref{lem:gaussian-affine-sum}, that
affine transformations of Gaussians are Gaussian, and that independent Gaussian
variances add.  In stationarity, $X_n$ and $X_{n+1}$ have the same law.  Thus,
if the invariant law is centered Gaussian with variance $v_h$, then
\[
    X_n\sim\normal{0}{v_h}
    \quad\Longrightarrow\quad
    X_{n+1}\sim
    \normal{0}{(1-h)^2v_h+2h},
\]
because $\xi_n\sim\normal{0}{1}$ is independent of $X_n$.  Equality of the
stationary input and output variances gives the identity
\[
    v_h=(1-h)^2v_h+2h.
\]
Solving,
\[
    v_h=\frac{2h}{1-(1-h)^2}=\frac{2}{2-h}.
\]
Thus the invariant law of ULA is $\pi_h=\normal{0}{2/(2-h)}$, not $\pi=\normal{0}{1}$ unless $h=0$.  For small $h$, the variance bias is $2/(2-h)-1=h/(2-h)=O(h)$.  The same bias can also be expressed as divergences between the two invariant laws:
\[
    \KL(\pi_h\|\pi)
    =
    \frac12\left(v_h-1-\log v_h\right),
    \qquad
    \KL(\pi\|\pi_h)
    =
    \frac12\left(\frac1{v_h}-1+\log v_h\right),
\]
and
\[
    W_2^2(\pi_h,\pi)=(\sqrt{v_h}-1)^2.
\]
All three quantities are $h^2/16+O(h^3)$ as $h\downarrow0$.  This example is
the cleanest way to see why unadjusted discretization creates bias.
\end{example}

\subsection{One-step KL discretization error}

Let $P_h(x,\cdot)$ be the exact Langevin transition for time $h$, and let
$\wh P_h(x,\cdot)$ be the ULA transition.  To compare them in KL, interpolate
one ULA step over the interval $[0,h]$ by
\[
    \dd \wh X_s=-\nabla U(x)\dd s+\sqrt2\,\dd B_s,\qquad \wh X_0=x.
\]
The exact Langevin law started from the same point has drift field $y\mapsto-\nabla U(y)$.  Since the KL below is $\Law(\wh X_{[0,h]})$ relative to $\Law(X_{[0,h]})$, the Girsanov formula evaluates both drift fields along the first path, namely at $\wh X_s$. Thus the formula reviewed in Appendix~\ref{app:ito-girsanov} gives the path-space comparison
\[
    \KL\!\left(\Law(\wh X_{[0,h]})\middle\|\Law(X_{[0,h]})\right)
    =
    \frac14
    \int_0^h
    \E\norm{\nabla U(\wh X_s)-\nabla U(x)}^2\dd s.
\]
By the data-processing inequality reviewed in Lemma~\ref{lem:data-processing}, the KL between the endpoint laws is no larger.  If $\nabla U$ is $L$-Lipschitz, then
\[
    \KL\!\left(\wh P_h(x,\cdot)\middle\|P_h(x,\cdot)\right)
    \leq
    \frac{L^2}{4}
    \int_0^h \E\norm{\wh X_s-x}^2\dd s.
\]
Since $\wh X_s=x-s\nabla U(x)+\sqrt2\,B_s$,
\[
    \E\norm{\wh X_s-x}^2=s^2\norm{\nabla U(x)}^2+2ds,
\]
and therefore
\begin{equation}\label{eq:ula-local-kl}
    \KL\!\left(\wh P_h(x,\cdot)\middle\|P_h(x,\cdot)\right)
    \leq
    \frac{L^2}{4}
    \left(dh^2+\frac{h^3}{3}\norm{\nabla U(x)}^2\right).
\end{equation}
This is the KL analogue of a local truncation error: freezing the drift for one short interval costs order $dh^2$, plus a term depending on the local drift size.

\subsection{Direct KL recursion for ULA}

The one-step estimate explains the scale of the discretization error, but it is not by itself a convergence theorem to $\pi$.  Indeed, a comparison such as $\KL(\wh q_k\|q_{kh})$ cannot simply be added to $\KL(q_{kh}\|\pi)$, because KL has no triangle inequality.  The modern KL-based analysis instead tracks the target-relative quantity $\KL(\wh q_k\|\pi)$ directly.

The following result is a translation of the KL theorem of Vempala and Wibisono~\cite{vempala2019rapidula} into our notation.

\begin{theorem}[ULA convergence in KL]\label{thm:ula-kl-vw}
Assume that $\pi(\dd x)\propto e^{-U(x)}\dd x$ satisfies the log-Sobolev
inequality with constant $C_{\mathsf{LSI}}$ in the convention
\[
    \KL(q\|\pi)\leq \frac{C_{\mathsf{LSI}}}{2}\FI(q\|\pi),
\]
and that $\nabla U$ is $L$-Lipschitz.  Let $\wh q_k$ be the law of ULA
\eqref{eq:ula} with step size
\[
    0<h\leq \frac{1}{4C_{\mathsf{LSI}}L^2}.
\]
If $\KL(\wh q_0\|\pi)<\infty$, then
\begin{equation}\label{eq:ula-vw-kl}
    \KL(\wh q_k\|\pi)
    \leq
    e^{-kh/C_{\mathsf{LSI}}}\KL(\wh q_0\|\pi)
    +8C_{\mathsf{LSI}}L^2dh.
\end{equation}
\end{theorem}

No convexity is assumed in this theorem.  The log-Sobolev inequality supplies
the global mixing mechanism, while the bounded Hessian assumption supplies the
local control needed to discretize the diffusion.

Thus, to reach KL target
$\KL(\wh q_k\|\pi)\leq \deltaerr^2$, it suffices to choose
\[
    h\lesssim
    \min\left\{
    \frac{1}{C_{\mathsf{LSI}}L^2},
    \frac{\deltaerr^2}{C_{\mathsf{LSI}}L^2d}
    \right\}
\]
and run for
\[
    k
    \gtrsim
    \frac{C_{\mathsf{LSI}}}{h}
    \log\frac{\KL(\wh q_0\|\pi)}{\deltaerr^2}
\]
iterations.  Equivalently, the iteration complexity is
\[
    k=\tO\!\left(
    \frac{C_{\mathsf{LSI}}^2L^2d}{\deltaerr^2}
    \right).
\]
The dependence on the ambient dimension is linear in this KL analysis.  By the
Bakry--\'Emery criterion for log-Sobolev inequalities (see Bakry, Gentil, and
Ledoux~\cite{bakry2014analysis}), under $m$-strong convexity one has
$C_{\mathsf{LSI}}\leq 1/m$, so the bound becomes
$\tO(\kappa^2d/\deltaerr^2)$ complexity, where $\kappa=L/m$ is the condition number.

\begin{proof}[Proof sketch]
Fix one step and start from law $\wh q_k$.  Interpolate the ULA update by
running, for elapsed time $t\in[0,h]$, the frozen-drift diffusion
\[
    \dd \bar X_t=-\nabla U(\bar X_0)\dd t+\sqrt2\,\dd B_t,
    \qquad \bar X_0\sim \wh q_k.
\]
Write $\nu_t=\Law(\bar X_t)$, so $\nu_h=\wh q_{k+1}$.  Differentiating
$\KL(\nu_t\|\pi)$ along this interpolated process gives the same negative
Fisher-information term as in the continuous Langevin calculation, plus the
error caused by freezing the drift:
\[
    \frac{\dd}{\dd t}\KL(\nu_t\|\pi)
    =
    -\FI(\nu_t\|\pi)
    +
    \E\ip{\nabla U(\bar X_t)-\nabla U(\bar X_0)}
    {\nabla\log \frac{\nu_t}{\pi}(\bar X_t)}.
\]
Young's inequality gives, directly for the inner product above,
\begin{align*}
    \E\ip{\nabla U(\bar X_t)-\nabla U(\bar X_0)}
    {\nabla\log \frac{\nu_t}{\pi}(\bar X_t)}
    &\leq
    \frac12\E\norm{\nabla\log \frac{\nu_t}{\pi}(\bar X_t)}^2
    +
    \frac12\E\norm{\nabla U(\bar X_t)-\nabla U(\bar X_0)}^2 \\
    &=
    \frac12\FI(\nu_t\|\pi)
    +
    \frac12\E\norm{\nabla U(\bar X_t)-\nabla U(\bar X_0)}^2.
\end{align*}
By $L$-smoothness,
$\norm{\nabla U(\bar X_t)-\nabla U(\bar X_0)}
\leq L\norm{\bar X_t-\bar X_0}$, so
\[
    \frac{\dd}{\dd t}\KL(\nu_t\|\pi)
    \leq
    -\frac12\FI(\nu_t\|\pi)
    +
    \frac{L^2}{2}\E\norm{\bar X_t-\bar X_0}^2.
\]
The log-Sobolev inequality turns the first term into contraction:
$\FI(\nu_t\|\pi)\geq 2C_{\mathsf{LSI}}^{-1}\KL(\nu_t\|\pi)$.  Thus
\[
    \frac{\dd}{\dd t}\KL(\nu_t\|\pi)
    \leq
    -\frac{1}{C_{\mathsf{LSI}}}\KL(\nu_t\|\pi)
    +
    \frac{L^2}{2}\E\norm{\bar X_t-\bar X_0}^2.
\]
It remains to bound the displacement of the frozen-drift interpolation.  Since
\[
    \E\norm{\bar X_t-\bar X_0}^2
    =
    t^2\E_{\wh q_k}\norm{\nabla U}^2+2dt,
\]
the Brownian part contributes the local $d h^2$ term after integration over
$t\in[0,h]$.  For the drift part, use an optimal $W_2$ coupling between
$\wh q_k$ and $\pi$.  Along such a coupling, smoothness gives
$\norm{\nabla U(x)}^2\leq
2\norm{\nabla U(y)}^2+2L^2\norm{x-y}^2$, and averaging gives
\[
    \E_{\wh q_k}\norm{\nabla U}^2
    \leq
    2\E_\pi\norm{\nabla U}^2+2L^2W_2^2(\wh q_k,\pi).
\]
The first term is at most $2Ld$, because integration by parts under
$\pi\propto e^{-U}$ gives
$\E_\pi\norm{\nabla U}^2=\E_\pi\Delta U\leq Ld$.  The second term is bounded by
Talagrand's comparison
$W_2^2(\wh q_k,\pi)\leq 2C_{\mathsf{LSI}}\KL(\wh q_k\|\pi)$ from
Subsection~\ref{subsec:sampling-problem}.  Hence the drift contribution is a
lower-order $d h^2$ term plus a multiple of
$C_{\mathsf{LSI}}L^4h^3\KL(\wh q_k\|\pi)$, and the stated step-size condition
makes this last piece small enough to be absorbed into the contraction.

The resulting one-step inequality of Vempala and
Wibisono~\cite{vempala2019rapidula}, in our notation, is
\[
    \KL(\wh q_{k+1}\|\pi)
    \leq
    e^{-h/C_{\mathsf{LSI}}}\KL(\wh q_k\|\pi)
    +6L^2dh^2.
\]
Iterating this recursion gives
\[
    \KL(\wh q_k\|\pi)
    \leq
    e^{-kh/C_{\mathsf{LSI}}}\KL(\wh q_0\|\pi)
    +
    6L^2dh^2
    \sum_{j=0}^{k-1} e^{-jh/C_{\mathsf{LSI}}}.
\]
Finally,
$\sum_{j=0}^{k-1}e^{-jh/C_{\mathsf{LSI}}}
\leq (1-e^{-h/C_{\mathsf{LSI}}})^{-1}
 \leq 4C_{\mathsf{LSI}}/(3h)$ under the same small-step assumptions.  The
last term is therefore at most $8C_{\mathsf{LSI}}L^2dh$, which gives
\eqref{eq:ula-vw-kl}.  This is the
discrete analogue of entropy dissipation: log-Sobolev supplies the contraction,
while smoothness controls the Euler--Maruyama error.
\end{proof}

Let us take a look at the result of the Theorem. The dependence on
dimension is linear, but the dependence on accuracy is polynomial:
to make the KL error of order $\deltaerr^2$, ULA uses
$\tO(C_{\mathsf{LSI}}^2L^2d/\deltaerr^2)$ steps.  This happens because the
Euler discretization has a stationary bias of order $h$, so high accuracy
forces a small step size.  It is therefore natural to ask whether one can keep
the Langevin proposal but remove the discretization bias by an exact correction.
Classical Metropolis adjustment does precisely this when density ratios are
available.

\subsection{Metropolis-adjusted Langevin algorithm}

The Metropolis-adjusted Langevin algorithm (MALA) corrects the ULA proposal with
an accept-reject step.  The distinction from the gradient-only setting is central:
density evaluations enable the classical high-accuracy rejection ideas, whereas
gradient information alone does not provide the density ratios those ideas
require.

Let $\pi(\dd x)\propto e^{-U(x)}\dd x$.  From a current state $x$, MALA first
draws the ULA proposal
\[
    y\sim q_h(x,\cdot)
    =
    \normal{x-h\nabla U(x)}{2h\Id}.
\]
It then accepts $y$ with probability
\begin{equation}\label{eq:mala-acceptance}
    a_h(x,y)
    =
    1\wedge
    \frac{\pi(y)q_h(y,x)}{\pi(x)q_h(x,y)}.
\end{equation}
If the proposal is rejected, the chain stays at $x$.  The Metropolis ratio
\eqref{eq:mala-acceptance} enforces detailed balance:
\[
    \pi(\dd x)P_h^{\mathsf{MALA}}(x,\dd y)
    =
    \pi(\dd y)P_h^{\mathsf{MALA}}(y,\dd x).
\]
Consequently, $\pi$ is exactly stationary for MALA\@.  This is the basic
advantage over ULA: the discretization bias is removed, at the price of needing
the density ratio $\pi(y)/\pi(x)$.

MALA is the cleanest example of a correction mechanism.  The
proposal uses only gradient information, just like ULA, but the accept-reject
step asks whether the proposed move has the right probability under the target
density.  This single density-ratio check changes the invariant distribution
from an approximation $\pi_h$ back to the exact target $\pi$.

Exact stationarity, however, is only the invariance part of the story.  To turn
MALA into a quantitative sampling algorithm, one still has to bound how many
accepted-or-rejected proposals are needed before the chain is close to $\pi$.
The modern high-accuracy theory separates this analysis into two tasks:
sampling efficiently once a good initialization is available, and producing
that initialization in the first place.  Both are phrased relative to a warm
start.  A standard warm-start condition is the following: an initial law
$\mu_0$ is $M$-warm with respect to
$\pi$ if
\[
    \mu_0(A)\leq M\pi(A)
    \qquad\text{for every measurable set }A,
\]
or equivalently
\begin{equation}\label{eq:warmstartinfty}
    \left\|\frac{\dd\mu_0}{\dd\pi}\right\|_{L^\infty(\pi)}
    \leq M.
\end{equation}
Some warm-start preparation results use the
finite-order R\'enyi version:
\[
    D_q(\mu_0\|\pi)
    =
    \frac{1}{q-1}
    \log
    \int
    \left(\frac{\dd\mu_0}{\dd\pi}\right)^q
    \dd\pi
    =O(1)
\]
for some $q>1$. Note that in the R\'enyi notation, \eqref{eq:warmstartinfty} can be written as
$D_\infty(\mu_0\|\pi)\leq \log M$.

For the first task, fix a target total-variation accuracy
$\deltaerr\in(0,1)$.  Assuming an $M$-warm start, Wu, Schmidler, and
Chen~\cite{wu2022minimaxmala} prove that, for $m$-strongly log-concave
and $L$-smooth targets in $\R^d$, MALA mixes in
\[
    O\!\left(
        \kappa\sqrt d\,
        \log^3\!\max\left\{\kappa,d,\frac{M}{\deltaerr}\right\}
    \right),
    \qquad \kappa=L/m,
\]
iterations, up to universal constants.  Chen and
Gatmiry~\cite{chen2023simplemala} extend this warm-start MALA picture under
smoothness and isoperimetry, recovering the same leading $\kappa\sqrt d$
behavior in the strongly log-concave case, again with logarithmic dependence
on $M/\deltaerr$.

For the second task, Altschuler and
Chewi~\cite{altschuler2023warmstarts} show that kinetic Langevin, also called
underdamped Langevin, can produce the required finite-order R\'enyi warm start
with the same $\tO(\sqrt d)$ dimension dependence.  At the level of
dimension dependence, the high-accuracy picture is: use kinetic Langevin to
prepare a warm start, then use MALA as the exact Metropolis-corrected sampler.

The discussion above also shows the limitation of classical adjusted MCMC\@.  MALA removes discretization bias when density evaluations are
available, but its sharp convergence guarantees require structural assumptions
such as strong log-concavity, log-Sobolev or isoperimetric inequalities,
smoothness, and warm starts.  A general data distribution may be multimodal,
singular, or available only through samples, so these assumptions are not a
natural starting point.  This motivates a different question: if we are given
data rather than an explicit target density, how should we design an algorithm
to generate new samples?

\section{Score-Based Diffusion Models}

The previous section treated sampling from an explicit density: the algorithm
could use $\nabla U$, and MALA could even use density ratios to remove
discretization bias.  Diffusion models begin from a different premise.  The
target law is represented by data, not by a tractable formula, so we do not try
to run Langevin dynamics directly on the data distribution.  Instead we add
noise to the data, forming a path of smoother laws.  At each positive noise
level, the score of the noised law is the local vector field that guides a
reverse dynamics back toward the data.

The denoising-score connection underlying this approach goes back in part to
the DDPM formulation of Ho, Jain, and Abbeel~\cite{ho2020ddpm}.  The
continuous-time SDE and probability-flow formulation used here follows Yang Song,
Sohl-Dickstein, Kingma, Kumar, Ermon, and Poole~\cite{song2021score}.  The
basic picture is simple: at time $t$, the noised distribution $p_t$ is a
blurred version of the data, and $\nabla\log p_t(x)$ points toward nearby
regions where this blurred density is larger.  Diffusion models learn this
time-indexed field of local denoising directions.

\subsection{Continuous-time forward noising}

The cleanest conceptual starting point is continuous time.  Let
$X_0\sim\pdata$ and let $(X_t)_{t\in[0,T]}$ solve a forward noising SDE
\begin{equation}\label{eq:forward-sde}
    \dd X_t=f_t(X_t)\dd t+g_t\dd B_t,
    \qquad X_0\sim\pdata.
\end{equation}
Here $B_t$ is Brownian motion, $f_t$ is a drift field, and $g_t$ is a scalar
diffusion coefficient.  We write $p_t$ for the density of $X_t$.
Most diffusion models use linear Gaussian forward processes for which, for
deterministic functions $a_t$ and $\sigma_t$,
\begin{equation}\label{eq:continuous-marginal}
    X_t\mid X_0\sim\normal{a_tX_0}{\sigma_t^2\Id}.
\end{equation}
\begin{example}[Two common continuous schedules]
For the variance-exploding heat flow,
\[
    X_t=X_0+\sqrt t\,Z,\qquad p_t=\pdata*\normal{0}{t\Id},
\]
with Gaussian transition kernel
\[
    X_t\mid X_0\sim\normal{X_0}{t\Id}.
\]
For the Ornstein--Uhlenbeck or variance-preserving flow with constant rate,
\[
    \dd X_t=-\frac12 X_t\dd t+\dd B_t,
\]
one has
\[
    X_t\mid X_0\sim\normal{e^{-t/2}X_0}{(1-e^{-t})\Id}.
\]
Both examples fit \eqref{eq:continuous-marginal}; only $a_t$ and
$\sigma_t$ differ.  The variance-preserving convention has
$a_t^2+\sigma_t^2=1$, so the total marginal variance is preserved when the data
have unit scale.
\end{example}

The two scalar functions $a_t$ and $\sigma_t$ summarize the signal-to-noise
ratio.  The coefficient $a_t$ says how much of the original sample remains
visible in the conditional mean, while $\sigma_t$ says how much independent
Gaussian uncertainty has been added.  Early in the forward process,
$\sigma_t$ is small and the score may be complicated because it reflects the
fine structure of the data.  At large noise, $p_t$ is smoother and closer to a
simple reference law, so the reverse sampler has an easier starting point.

\subsection{Continuous-time Tweedie identity}\label{subsec:continuous-tweedie}

The most natural way to undo the forward noising is to denoise: given a noisy
observation $X_t=x$, estimate the clean sample that produced it, namely the
posterior mean $\E[X_0\mid X_t=x]$.  The object that encodes this denoising is
the \emph{score} of the noised law,
\begin{equation}\label{eq:continuous-score}
    \sstar_t(x)=\nabla\log p_t(x),
\end{equation}
the true score at time $t$, approximated in practice by a learned model
$\score_t(x)\approx\sstar_t(x)$.  The Gaussian marginal formula
\eqref{eq:continuous-marginal} makes the link between the score and denoising
precise through a continuous-time form of Tweedie's identity:
\begin{equation}\label{eq:continuous-tweedie}
    \sstar_t(x)
    =
    \frac{1}{\sigma_t^2}
    \E[a_tX_0-X_t\mid X_t=x].
\end{equation}
Equivalently, whenever $a_t\neq0$, define the optimal denoiser
\[
    \dstar_t(x)
    :=
    \E[X_0\mid X_t=x]
    =
    a_t^{-1}\bigl(x+\sigma_t^2\sstar_t(x)\bigr).
\]
A learned score $\score_t$ similarly determines a learned denoiser
\[
    \denoiser_t(x)
    :=
    a_t^{-1}\bigl(x+\sigma_t^2\score_t(x)\bigr).
\]
Thus denoising and score estimation are two ways of describing the same
posterior information.

\begin{proof}
The density of $X_t$ is the Gaussian mixture
\[
    p_t(x)=
    \int \frac{1}{(2\pi\sigma_t^2)^{d/2}}
    \exp\left(-\frac{\norm{x-a_tx_0}^2}{2\sigma_t^2}\right)
    \pdata(\dd x_0).
\]
Write the Gaussian kernel in this integral as
\[
    \varphi_t(x\mid x_0)
    =
    \frac{1}{(2\pi\sigma_t^2)^{d/2}}
    \exp\left(-\frac{\norm{x-a_tx_0}^2}{2\sigma_t^2}\right).
\]
For fixed $x_0$,
\[
    \nabla_x\varphi_t(x\mid x_0)
    =
    \frac{a_tx_0-x}{\sigma_t^2}\,\varphi_t(x\mid x_0).
\]
Thus, differentiating under the integral,
\[
    \nabla p_t(x)
    =
    \frac{1}{\sigma_t^2}
    \int (a_tx_0-x)\,\varphi_t(x\mid x_0)\,\pdata(\dd x_0).
\]
On the other hand, Bayes' rule gives the conditional law of the clean sample
given the noisy observation:
\[
    \Pp(X_0\in\dd x_0\mid X_t=x)
    =
    \frac{\varphi_t(x\mid x_0)\,\pdata(\dd x_0)}{p_t(x)}.
\]
Dividing the previous display by $p_t(x)$ therefore gives
\[
    \nabla\log p_t(x)
    =
    \frac{1}{\sigma_t^2}
    \E[a_tX_0-x\mid X_t=x],
\]
which is the same as \eqref{eq:continuous-tweedie}, since the conditioning
sets $X_t=x$.
\end{proof}

This Tweedie identity explains why the score can be learned from empirical data.
Fix $t>0$.
Under the Gaussian corruption $X_t=a_tX_0+\sigma_t Z$, with
$Z\sim\normal{0}{\Id}$ independent of $X_0$, the random vector inside this
conditional expectation becomes
\[
    \frac{a_tX_0-X_t}{\sigma_t^2}
    =
    -\frac{Z}{\sigma_t}.
\]
Thus the conditional mean of $-Z/\sigma_t$ given $a_tX_0+\sigma_t Z=x$ is
exactly $\sstar_t(x)$, which motivates the regression population loss:
\[
    \E\norm{
        \score_t(a_tX_0+\sigma_t Z)+\frac{Z}{\sigma_t}
    }^2.
\]
Replacing the expectation over
$X_0\sim\pdata$ by an empirical average over data points, while resampling the
Gaussian noise, gives the basic denoising score-matching objective.

This is why the learned field has a denoising interpretation rather than being
an arbitrary vector field.  The regression target is noisy for each individual
corruption, but its conditional average points from the observation toward the
posterior mean $\dstar_t(x)$ of the clean sample.  Equivalently, the score
records the Bayes correction encoded by the noised data distribution.

\subsection{Continuous-time reverse SDE}

The forward SDE \eqref{eq:forward-sde} is designed to move data toward a simple
law.  For sampling, the object we need is not a pathwise inverse of the
Brownian motion.  It is enough to construct a Markov process whose one-time
marginals run through the same densities in the opposite order: if
$(p_t)_{0\leq t\leq T}$ are the forward noising marginals, then the reverse
sampler $(Y^{\leftarrow}_s)_{0\leq s\leq T}$ should satisfy
\[
    \Law(Y^{\leftarrow}_s)=p_{T-s}
    \qquad\text{for every }s\in[0,T].
\]
This requirement is deliberately distributional.  It says that the reverse
sampler has the right snapshots, not that it retraces individual forward
Brownian paths.  We therefore first match the evolution of densities.
Set $q_s=p_{T-s}$ and write $t=T-s$.  The forward density satisfies
\[
    \partial_t p_t
    =
    -\nabla\cdot(f_tp_t)
    +\frac12g_t^2\Delta p_t.
\]
Therefore, with $t=T-s$,
\[
    \partial_s q_s
    =
    \nabla\cdot(f_t p_t)
    -\frac12g_t^2\Delta p_t.
\]
The key step uses the identity $\Delta p_t=\nabla\cdot(p_t\nabla\log p_t)$,
which brings in precisely the score $\sstar_t=\nabla\log p_t$ from
\eqref{eq:continuous-score}.  This rewrites the density evolution as
\[
    \partial_s q_s
    =
    -\nabla\cdot\left[
    \left(-f_t+g_t^2\sstar_t\right)p_t
    \right]
    +\frac12g_t^2\Delta p_t.
\]
This is the Fokker--Planck equation for the reverse-time diffusion
\begin{equation}\label{eq:reverse-sde}
    \dd Y^{\leftarrow}_s
    =
    \left[-f_{T-s}(Y^{\leftarrow}_s)
    +g_{T-s}^2\,\sstar_{T-s}(Y^{\leftarrow}_s)\right]\dd s
    +g_{T-s}\dd B^{\leftarrow}_s,
    \qquad Y^{\leftarrow}_0\sim p_T.
\end{equation}
Here $B^{\leftarrow}_s$ is Brownian motion in the reverse sampling time.  By
construction, if $Y^{\leftarrow}_0\sim p_T$, then $Y^{\leftarrow}_s\sim p_{T-s}$ for every $s\in[0,T]$.

The drift in the reverse SDE has two pieces.  The term $-f_{T-s}$ reverses the
deterministic transport part of the forward dynamics.  The extra term
$g_{T-s}^2\,\sstar_{T-s}$ is exactly the score introduced in
Subsection~\ref{subsec:continuous-tweedie} as the denoising direction: it corrects the density
evolution caused by diffusion by nudging each sample back toward regions of
higher noised density, i.e.\ toward the posterior mean of the clean sample.
This is the precise sense in which the denoiser of the previous subsection
\emph{is} the engine of the reverse dynamics.  Replacing the true score
$\sstar_t$ by a learned score $\score_t$ yields the score-based reverse SDE
sampler, the continuous object that DDPM-type algorithms discretize.

It is instructive to specialize \eqref{eq:reverse-sde} to the two schedules of
the running example; we will revisit the same two processes from the
probability-flow viewpoint in the next subsection.

\begin{example}[Variance-exploding reverse SDE]
For $X_t=X_0+\sqrt t Z$, we have $f_t=0$ and $g_t=1$, so the reverse SDE is
\[
    \dd Y^{\leftarrow}_s=\sstar_{T-s}(Y^{\leftarrow}_s)\dd s+\dd B^{\leftarrow}_s,
    \qquad Y^{\leftarrow}_0\sim p_T.
\]
The whole drift is the score: starting from a high-noise draw, the sampler
repeatedly steps along the denoising direction $\sstar_{T-s}$ while fresh
Brownian noise is injected.
\end{example}

\begin{example}[Variance-preserving Ornstein--Uhlenbeck reverse SDE]
For
\[
    \dd X_t=-\frac12X_t\dd t+\dd B_t,
\]
we have $f_t=-\frac12x$ and $g_t=1$, so the reverse SDE is
\[
    \dd Y^{\leftarrow}_s=\left[\frac12 Y^{\leftarrow}_s+\sstar_{T-s}(Y^{\leftarrow}_s)\right]\dd s+\dd B^{\leftarrow}_s,
    \qquad Y^{\leftarrow}_0\sim p_T.
\]
The term $\frac12 Y^{\leftarrow}_s$ undoes the OU contraction toward the origin, while the
score term $\sstar_{T-s}$ supplies the denoising correction.
\end{example}

In practice the reverse SDE is run in discrete time, and this discretization is
exactly DDPM\@.  One fixes a grid in the sampling time $s$, replaces the true
score $\sstar_{T-s}$ by the learned model $\score_{T-s}$, and takes
Euler--Maruyama steps of \eqref{eq:reverse-sde}: each step nudges the current
sample along the denoising drift and then adds fresh Gaussian noise of the
appropriate variance.  This is the DDPM sampler of Ho, Jain, and
Abbeel~\cite{ho2020ddpm}, the stochastic counterpart of the deterministic DDIM
update we will read off from the probability-flow ODE in the next subsection.
The detailed Gaussian transition kernels and the resulting discretization error
are taken up in Section~\ref{sec:discretizing}.

\subsection{Probability flow ODE}

There is a deterministic ODE whose one-time marginals are the same as those of
the forward SDE \eqref{eq:forward-sde}.  This ODE is called the
\emph{probability flow ODE}. The probability flow ODE replaces
Brownian randomness by a transport field that has the same effect on
marginal densities. Define the velocity field
\begin{equation}\label{eq:pf-velocity}
    v_t(x)
    =
    f_t(x)-\frac12 g_t^2\sstar_t(x)
    =
    f_t(x)-\frac12 g_t^2\nabla\log p_t(x).
\end{equation}
The probability flow ODE is
\begin{equation}\label{eq:pf-ode}
    \frac{\dd X_t}{\dd t}=v_t(X_t).
\end{equation}
If $X_0\sim p_0$, then the solution of \eqref{eq:pf-ode} has marginal density
$p_t$ at every time $t$ for which the ODE is well posed.

\begin{proof}[Derivation from the Fokker--Planck equation]
The forward SDE \eqref{eq:forward-sde} has Fokker--Planck equation
\begin{equation}\label{eq:forward-fp-score-section}
    \partial_t p_t
    =
    -\nabla\cdot(f_tp_t)
    +\frac12 g_t^2\Delta p_t.
\end{equation}
Use the identity
\[
    \Delta p_t
    =
    \nabla\cdot(\nabla p_t)
    =
    \nabla\cdot(p_t\nabla\log p_t)
    =
    \nabla\cdot(p_t\sstar_t).
\]
Then \eqref{eq:forward-fp-score-section} becomes
\[
    \partial_t p_t
    =
    -\nabla\cdot(f_tp_t)
    +\frac12 g_t^2\nabla\cdot(p_t\sstar_t)
    =
    -\nabla\cdot\left[
    \left(f_t-\frac12g_t^2\sstar_t\right)p_t
    \right].
\]
This is exactly the continuity equation
\begin{equation}\label{eq:continuity-pf}
    \partial_t p_t+\nabla\cdot(p_tv_t)=0
\end{equation}
for the density transported by the deterministic flow $\dot X_t=v_t(X_t)$.
Since $p_t$ satisfies this equation with initial condition $p_0$, the ODE flow
pushes $p_0$ forward to $p_t$.
\end{proof}

\begin{example}[Variance-exploding probability flow]
For $X_t=X_0+\sqrt t Z$, we have $f_t=0$ and $g_t=1$.  The probability flow
velocity is
\[
    v_t(x)=-\frac12\nabla\log p_t(x).
\]
The score points toward higher density, so $-\frac12\nabla\log p_t$ pushes mass
outward and reproduces the smoothing effect of the variance-exploding noising
process.  To generate samples, one integrates the same ODE backward in time,
which reverses this velocity.
\end{example}

\begin{example}[Variance-preserving Ornstein--Uhlenbeck flow]
For
\[
    \dd X_t=-\frac12X_t\dd t+\dd B_t,
\]
the probability flow velocity is
\[
    v_t(x)=-\frac12x-\frac12\nabla\log p_t(x).
\]
The first term is the deterministic OU contraction toward the origin.  The
second term is the transport representation of the Brownian smoothing.
\end{example}

The reverse SDE and
the probability flow ODE are two different machines that produce the same
snapshots when the score is exact.  The SDE machine keeps injecting randomness;
the ODE machine deterministically transports particles (starting from random initial data).  Since many numerical
and control questions depend on paths, not only snapshots, one should not
freely replace one machine by the other without checking what quantity is being
analyzed.

For sampling, we use the same reverse-time convention as in
\eqref{eq:reverse-sde}: set $s=T-t$, start from a high-noise draw with law
$p_T$, and integrate from $s=0$ to $s=T$.  The reverse-time probability-flow
sampler is the ODE
\[
    \frac{\dd Z^{\leftarrow}_s}{\dd s}
    =
    -v_{T-s}(Z^{\leftarrow}_s)
    =
    -f_{T-s}(Z^{\leftarrow}_s)
    +\frac12 g_{T-s}^2\sstar_{T-s}(Z^{\leftarrow}_s),
    \qquad Z^{\leftarrow}_0\sim p_T.
\]
By the same continuity-equation calculation, $Z^{\leftarrow}_s$ has law $p_{T-s}$
when the score is exact.
Compare this with the reverse SDE drift in \eqref{eq:reverse-sde}: $ -f_{T-s}+g_{T-s}^2\sstar_{T-s}$.
The DDIM sampler of Song, Meng, and Ermon~\cite{song2021ddim} is the
standard sampler on this probability-flow side: it uses the same trained
denoising model as DDPM, but follows the deterministic update suggested by the
ODE viewpoint.

The deterministic nature of the ODE has an important consequence for
likelihoods.  Along a solution of $\dot X_t=v_t(X_t)$, the continuity equation
\eqref{eq:continuity-pf} implies
\begin{equation}\label{eq:instantaneous-change}
    \frac{\dd}{\dd t}\log p_t(X_t)
    =
    -\nabla\cdot v_t(X_t).
\end{equation}
Indeed,
\[
    \frac{\dd}{\dd t}\log p_t(X_t)
    =
    \partial_t\log p_t(X_t)+
    \ip{v_t(X_t)}{\nabla\log p_t(X_t)},
\]
and dividing \eqref{eq:continuity-pf} by $p_t$ gives
\[
    \partial_t\log p_t+\ip{v_t}{\nabla\log p_t}
    =
    -\nabla\cdot v_t.
\]
Therefore, if the map $X_0\mapsto X_T$ is obtained by integrating the
probability flow ODE forward,
\[
    \log p_0(X_0)
    =
    \log p_T(X_T)
    +
    \int_0^T \nabla\cdot v_t(X_t)\dd t.
\]
This is the continuous normalizing-flow identity used to compute likelihoods
from probability flow ODEs.

\begin{remark}
In practice the true score $\sstar_t$ is replaced by a learned score
$\score_t$.  Then neither the implemented reverse SDE nor the implemented
probability flow ODE has exactly the target marginals.  Their final sampling
error comes from score error together with numerical discretization error:
time-stepping error for the SDE, and ODE integration error for the probability
flow.  For likelihood computation with the probability flow ODE, there is one
more numerical issue: one must also approximate the divergence
$\nabla\cdot v_t$, often by automatic differentiation or stochastic trace
estimators.
\end{remark}

\section{Stochastic Localization and Polchinski Flow}
\label{sec:stochastic-localization-polchinski}

The previous section introduced diffusion models through Gaussian noising.  We
started from data, followed the noised marginals $p_t$, learned the score
$\nabla\log p_t$, and used that score to run either a reverse SDE or a
probability-flow ODE\@.  This section looks at the same Gaussian channel from two
complementary angles: stochastic localization and Polchinski flow.  All three
viewpoints start from a clean random variable and the noisy observation (here and below we take the variance-exploding one)
\[
    X_t=X_\star+\sqrt{t}Z,
\]
but they organize the information in different ways.  Diffusion models
emphasize reverse-time sampling.  Stochastic localization asks a Bayesian
question: as the noisy observation becomes more precise, how does the posterior
law of the hidden signal $X_\star$ evolve?  Polchinski flow asks a
density-level question: as the noise level changes, how do the smoothed density
$p_t$ and the effective potential $U_t=-\log p_t$ evolve?

These two changes of perspective are useful because they expose structure that
is less visible from the reverse SDE alone.  In stochastic localization, the
score appears as a posterior denoising correction.  In Polchinski flow, the
score is the negative gradient of a coarse-grained energy landscape.  Both
viewpoints lead naturally to quantitative tools, especially covariance
identities, that will be useful later.

We begin by treating stochastic localization as an observation model for the
same Gaussian channel.  This lets us compare diffusion time with localization
precision, derive the posterior law, and connect its mean to the score.  We
then record the martingale structure that makes localization useful.  Finally,
we return to the density $p_t$ itself and write the corresponding
Polchinski flow for the effective potential $U_t$.

\subsection{Stochastic localization}
\label{subsec:stochastic-localization}

Let $X_\star\sim\mu$ be the signal, or clean data point.  Here we write the law
as $\mu$ rather than $\pdata$, to emphasize that the construction applies to a
general distribution and not only to the data law.  Start with the
Gaussian smoothing channel
\[
    X_t=X_\star+\sqrt{t}Z,\qquad Z\sim\normal{0}{\Id}.
\]
Here $t$ is the noise variance, and $p_t=\Law(X_t)$ is the Gaussian
convolution of the data law with covariance $t\Id$.  As $t$ increases,
the density becomes smoother; as $t$ decreases, the channel approaches the
original data distribution.

Stochastic localization uses the same Gaussian channel, but parametrizes it by
precision rather than variance.  Observe the same unknown $X_\star$ through
the continuous-time Gaussian observation process
\begin{equation}\label{eq:sl-observation}
    \dd Y_u=X_\star\dd u+\dd W_u,\qquad Y_0=0.
\end{equation}
Here $W_u$ is Brownian motion independent of $X_\star$.  Conditional on
$X_\star=x$, we have $Y_u\sim\normal{ux}{u\Id}$, and thus after a rescaling:
\[
    \bar Y_u:=u^{-1}Y_u\sim\normal{x}{u^{-1}\Id}.
\]
Hence the localization time $u$ corresponds to the diffusion noise variance
$t=u^{-1}$:
\[
    \bar Y_u=X_\star+\sqrt{t}Z,
    \qquad t=\frac1u.
\]
The same noised observation can therefore be indexed either by variance
$t$ or by precision $u$.  Diffusion notation emphasizes smoothing as
$t$ increases.  Localization notation emphasizes Bayesian inference as
$u$ increases and the observation becomes more informative.  This precision
parametrization is useful because the posterior mean and covariance then obey
simple martingale identities, as we will discuss below.

We now write this posterior law explicitly.  At precision $u$, the
observation is $Y_u$, and the basic object is the conditional law of the hidden
signal $X_\star$.  Let
$\mu_u(\cdot \mid y)=\Law(X_\star\mid Y_u = y)$ be the posterior law after
observing $Y_u = y$.  Bayes' rule gives
\begin{equation}\label{eq:sl-posterior}
    \mu_{u}(\dd x \mid Y_u = y)
    =
    \frac{1}{Z_u(y)}
    \exp\left\{\ip{y}{x}-\frac{u}{2}\norm{x}^2\right\}\mu(\dd x).
\end{equation}
The normalizing constant is
\[
    Z_u(y)=
    \int
    \exp\left\{\ip{y}{x}-\frac{u}{2}\norm{x}^2\right\}\mu(\dd x).
\]
This is why the method is called localization: the posterior is the original
measure tilted by a random linear field and penalized by a growing quadratic
function.
As $u$ grows, the posterior becomes increasingly concentrated near the hidden
signal.

\begin{proof}[Derivation of \eqref{eq:sl-posterior}]
Given $X_\star=x$, the observation $Y_u$ has density proportional to
\[
    \exp\left(-\frac{\norm{y-ux}^2}{2u}\right).
\]
Expanding the square,
\[
    -\frac{\norm{y-ux}^2}{2u}
    =
    -\frac{\norm{y}^2}{2u}
    +\ip{y}{x}
    -\frac{u}{2}\norm{x}^2.
\]
The first term is independent of $x$ and is absorbed into the normalizing
constant.  Multiplying by the prior measure $\mu(\dd x)$ gives
\eqref{eq:sl-posterior}.
\end{proof}

Now put the random observation back in and write
\[
    \mu_u=\mu_u(\cdot\mid Y_u),
    \qquad
    M_u=\int x\,\mu_u(\dd x).
\]
The original stochastic-localization viewpoint is that the random measure
$(\mu_u)_{u\geq0}$ itself evolves by a measure-valued SDE, which we now derive.  Let
$\cF_u^Y=\sigma(Y_r:0\leq r\leq u)$ and define the innovation process
\[
    I_u=Y_u-\int_0^u M_s\dd s,
    \qquad
    \dd I_u=\dd Y_u-M_u\dd u.
\]
This subtracts the part of the next observation that is already predictable
from the current posterior.  Intuitively, $I_u$ records only the ``surprise'' in
the observation stream: after conditioning on $\cF_u^Y$, the drift of
$\dd Y_u$ is $M_u\dd u$, so subtracting it leaves an increment with no
predictable component and the same accumulated covariance as the original
Brownian noise.  More formally, one can check that $I_u$ is a continuous
$\cF_u^Y$-martingale whose accumulated covariance over $[0,u]$ is $u\Id$;
Levy's characterization then identifies it as a Brownian motion with respect
to $\cF_u^Y$.  The posterior law satisfies
\begin{equation}\label{eq:sl-measure-sde}
    \dd \mu_u(\dd x)
    =
    \ip{x-M_u}{\dd I_u}\,\mu_u(\dd x).
\end{equation}
Equivalently, for every bounded test function $f$,
\[
    \dd\int f(x)\,\mu_u(\dd x)
    =
    \ip{\int f(x)(x-M_u)\,\mu_u(\dd x)}{\dd I_u}.
\]
The same identity applies componentwise to vector-valued test functions when
the relevant moments are finite.
In particular, for every Borel set $A$, the process
$u\mapsto\mu_u(A)=\Pr(X_\star\in A\mid\cF_u^Y)$ is a martingale:
\[
    \E[\mu_v(A)\mid\cF_u^Y]=\mu_u(A),
    \qquad v\geq u.
\]
Thus, in this formulation, stochastic localization is a measure-valued
martingale whose random tilt concentrates the prior while preserving
conditional expectations.

\begin{proof}[Proof of \eqref{eq:sl-measure-sde}]
It is enough to compute the differential of the normalized posterior weight.
From \eqref{eq:sl-posterior}, along the random observation path,
\[
    \frac{\mu_u(\dd x)}{\mu(\dd x)}
    =
    \frac{1}{Z_u}
    \exp\left\{\ip{Y_u}{x}-\frac{u}{2}\norm{x}^2\right\}.
\]
First, It\^o's formula gives
\[
    \begin{aligned}
    \dd
    \exp\left\{\ip{Y_u}{x}-\frac{u}{2}\norm{x}^2\right\}
    &=
    \exp\left\{\ip{Y_u}{x}-\frac{u}{2}\norm{x}^2\right\}
    \left(\ip{x}{\dd Y_u}-\frac12\norm{x}^2\dd u\right)
    \\
    &\qquad
    +
    \frac12
    \exp\left\{\ip{Y_u}{x}-\frac{u}{2}\norm{x}^2\right\}
    \norm{x}^2\dd u
    \\
    &=
    \exp\left\{\ip{Y_u}{x}-\frac{u}{2}\norm{x}^2\right\}
    \ip{x}{\dd Y_u}.
    \end{aligned}
\]
The middle line contains the It\^o correction; it cancels the drift from
the term $-\frac u2\norm{x}^2$ in the exponent.  Integrating this identity
in $x$ gives
\[
    \frac{\dd Z_u}{Z_u}=\ip{M_u}{\dd Y_u},
    \qquad
    \dd\log Z_u=\ip{M_u}{\dd Y_u}-\frac12\norm{M_u}^2\dd u.
\]
Therefore the log-density of $\mu_u$ with respect to $\mu$ satisfies
\[
    \dd\log\frac{\mu_u(\dd x)}{\mu(\dd x)}
    =
    \ip{x-M_u}{\dd Y_u}
    -
    \frac12\bigl(\norm{x}^2-\norm{M_u}^2\bigr)\dd u.
\]
Applying It\^o's formula once more, now to the exponential of this
log-density, gives
\[
    \frac{\dd\mu_u(\dd x)}{\mu_u(\dd x)}
    =
    \ip{x-M_u}{\dd Y_u}
    -
    \ip{x-M_u}{M_u}\dd u
    =
    \ip{x-M_u}{\dd I_u}.
\]
This is exactly \eqref{eq:sl-measure-sde}; integrating against a bounded test
function $f$ gives the weak form.
\end{proof}

\subsection{Posterior means and martingale structure}

The previous subsection constructed the posterior law and its measure-valued
martingale equation.  We now pass to its first two moments.  First we freeze
the observation and connect the posterior mean to the diffusion-model score;
then we put the random observation back in and record the martingale dynamics.

For a deterministic observation $y$, define the posterior mean and covariance
\[
    m_u(y)=\E[X_\star\mid Y_u=y],
    \qquad
    \Sigma_u(y)=\Cov(X_\star\mid Y_u=y).
\]
The posterior mean is exactly a denoiser.  Indeed, since
$\bar Y_u=Y_u/u=X_\star+u^{-1/2}Z$, a deterministic observation $y$
corresponds to the noised sample $y/u$, and Tweedie's identity
\eqref{eq:continuous-tweedie} gives
\begin{equation}\label{eq:sl-tweedie}
    m_u(y)
    =
    \frac{y}{u}
    +
    \frac{1}{u}
    \nabla\log p_{1/u}\!\left(\frac{y}{u}\right),
\end{equation}
where $p_t$ is the density of $X_\star+\sqrt{t}Z$.  Equivalently,
\[
    \nabla\log p_{1/u}(y/u)
    =
    u\bigl(m_u(y)-y/u\bigr).
\]
Thus learning a score is the same problem as learning the
localization posterior mean.

\begin{proof}[Specializing Tweedie to the localization channel]
The rescaled observation $\bar Y_u=X_\star+u^{-1/2}Z$ is the forward channel
$X_t=X_\star+\sqrt{t}Z$ at time $t=1/u$, in the variance-exploding convention
$a_t=1$, $\sigma_t^2=t=1/u$.  Thus the continuous-time Tweedie identity
\eqref{eq:continuous-tweedie} reads, in the current notation,
\[
    \nabla\log p_{1/u}(y/u)
    =
    u\,\E[X_\star-y/u\mid Y_u=y]
    =
    u\bigl(m_u(y)-y/u\bigr),
\]
where we used $\E[X_\star\mid Y_u=y]=m_u(y)$.  Solving for $m_u(y)$ then yields
\eqref{eq:sl-tweedie}.
\end{proof}

The covariance, in turn, is the linear response of this denoiser:
differentiating the posterior tilt \eqref{eq:sl-posterior} with respect to the
observation gives
\begin{equation}\label{eq:posterior-mean-jacobian}
    \nabla_y m_u(y)=\Sigma_u(y).
\end{equation}
Combining this with Tweedie's identity \eqref{eq:sl-tweedie}, and writing
$t=1/u$, gives the score-Jacobian identity
\begin{equation}\label{eq:score-hessian-covariance}
    \nabla \sstar_t(x)
    =
    t^{-2}\Sigma_{1/t}(x/t)
    -
    t^{-1}\Id.
\end{equation}
Up to this time change, then, the Hessian of the noised log-density is precisely
the posterior covariance of the localization process.

Now return to the random localization path. We write
$M_u$ and $\Sigma_u$ below for the random evaluations
$m_u(Y_u)$ and $\Sigma_u(Y_u)$.  The measure-valued equation
\eqref{eq:sl-measure-sde} then turns the fixed-observation posterior moments
into stochastic dynamics.

Start with the mean.  Applying the weak form of
\eqref{eq:sl-measure-sde} componentwise to the test function $f(x)=x$ gives,
under standard integrability assumptions,
\begin{equation}\label{eq:sl-mean-sde}
    \dd M_u
    =
    \left(\int x(x-M_u)^\top \mu_u(\dd x)\right)\dd I_u
    =
    \Sigma_u\,\dd I_u.
\end{equation}
Thus the mean moves only in response to the innovation Brownian motion, and
the size and direction of that motion are governed by the current posterior
covariance.  In particular, the posterior mean is a martingale: before seeing
the next small piece of data, the current posterior mean is the best prediction
of the next posterior mean.

\begin{proof}[Why the posterior mean is a martingale]
First note that $M_u$, although defined by conditioning on the single
observation $Y_u$, is unchanged if we condition on the entire history
$\cF_u^Y=\sigma(Y_r:0\leq r\leq u)$.  The reason is that $Y_u$ is a sufficient
statistic for $X_\star$: by Girsanov's theorem
(Appendix~\ref{app:ito-girsanov}), the likelihood of the path
$(Y_r)_{0\leq r\leq u}$ under the signal value $X_\star=x$, relative to the Brownian motion, is
\[
    \exp\!\left(\int_0^u\ip{x}{\dd Y_r}-\frac{u}{2}\norm{x}^2\right)
    =
    \exp\!\left(\ip{x}{Y_u}-\frac{u}{2}\norm{x}^2\right),
\]
which depends on the path only through the endpoint $Y_u$.  This is exactly the
tilt appearing in \eqref{eq:sl-posterior}, so the posterior given the whole
history equals the posterior given $Y_u$, and in particular
\[
    \E[X_\star\mid\cF_u^Y]=\E[X_\star\mid Y_u]=M_u.
\]
The martingale property now follows from the tower property.  For $v\geq u$,
since $\cF_u^Y\subseteq\cF_v^Y$,
\[
    \E[M_v\mid \cF_u^Y]
    =
    \E[\E[X_\star\mid \cF_v^Y]\mid \cF_u^Y]
    =
    \E[X_\star\mid \cF_u^Y]
    =
    M_u. \qedhere
\]
\end{proof}

The covariance, which set the size of these mean fluctuations, has an evolution
of its own.  Unlike the mean, it carries a drift, and that drift is strictly
dissipative: in one dimension,
\[
    \frac{\dd}{\dd u}\E[\Sigma_u]=-\E[\Sigma_u^2]\leq0,
\]
and in multiple dimensions the trace satisfies
\begin{equation}\label{eq:localization-covariance-dissipation}
    \frac{\dd}{\dd u}\E[\Tr\Sigma_u]
    =
    -\E[\Tr(\Sigma_u^2)]\leq0.
\end{equation}
Integrate \eqref{eq:localization-covariance-dissipation} over
$u\in[0,\infty)$: the fundamental theorem of calculus gives
\[
    \int_0^\infty\E\Tr(\Sigma_u^2)\,\dd u
    =
    \E\Tr\Sigma_0-\lim_{u\to\infty}\E\Tr\Sigma_u
    \leq
    \E\Tr\Sigma_0,
\]
the inequality because $\Tr\Sigma_u\geq0$.  At $u=0$ no observation has yet been
made, so $\Sigma_0=\Cov(X_\star)$; hence if $\Cov(X_\star)\preceq\Id$,
\begin{equation}\label{eq:localization-hessian-budget}
    \int_0^\infty
    \E\Tr(\Sigma_u^2)\,\dd u
    \leq
    \E\Tr\Sigma_0
    \leq d,
\end{equation}
so the total squared covariance accumulated along the entire path is bounded by
the dimension.  This is the geometric content of localization: the random
posterior measure steadily concentrates, while its mean drifts by martingale
increments whose size is set by the uncertainty that remains.

\begin{proof}[One-dimensional covariance calculation]
In one dimension the posterior variance splits into a second moment and the
square of the mean,
\[
    \Sigma_u=\E[X_\star^2\mid Y_u]-M_u^2,
\]
and we track the two pieces separately.  For the second moment, the weak form
of \eqref{eq:sl-measure-sde} with test function $f(x)=x^2$ gives
\[
    \dd\,\E[X_\star^2\mid Y_u]
    =
    \Bigl(\int x^2(x-M_u)\,\mu_u(\dd x)\Bigr)\dd I_u
    =
    \Bigl(\E[X_\star^3\mid Y_u]-M_u\E[X_\star^2\mid Y_u]\Bigr)\dd I_u,
\]
where we used $\E[X_\star\mid Y_u]=M_u$.  For the squared mean,
$\dd M_u=\Sigma_u\dd I_u$ and It\^o's formula give
\[
    \dd(M_u^2)=2M_u\Sigma_u\,\dd I_u+\Sigma_u^2\,\dd u.
\]
Subtracting the two identities,
\[
    \dd\Sigma_u
    =
    \Bigl(\E[X_\star^3\mid Y_u]-M_u\E[X_\star^2\mid Y_u]-2M_u\Sigma_u\Bigr)\dd I_u
    -
    \Sigma_u^2\,\dd u.
\]
The $\dd I_u$ term is a martingale increment and has mean zero, so taking
expectations leaves only the drift,
\[
    \frac{\dd}{\dd u}\E[\Sigma_u]=-\E[\Sigma_u^2].
\]
In $d$ dimensions the same computation, now with the matrix-valued test
function $f(x)=xx^\top$, gives the covariance SDE, in components,
\begin{equation}\label{eq:localization-covariance-sde}
    \dd(\Sigma_u)_{ij}
    =
    \sum_k
    \E\bigl[(X_\star-M_u)_i(X_\star-M_u)_j(X_\star-M_u)_k\mid Y_u\bigr]\,(\dd I_u)_k
    -
    (\Sigma_u^2)_{ij}\,\dd u.
\end{equation}
The first term is a mean-zero martingale increment, so taking traces and
expectations recovers \eqref{eq:localization-covariance-dissipation},
$\frac{\dd}{\dd u}\E[\Tr\Sigma_u]=-\E[\Tr(\Sigma_u^2)]$.
\end{proof}

This martingale-and-covariance structure is what makes stochastic localization
useful as an analytical tool.  Rather than bounding the original
high-dimensional measure directly, one embeds it in a random path of
progressively tilted posteriors and argues along the path, where the mean is a
martingale and the covariance dissipates.  Many static questions thereby reduce
to estimating quantities averaged over the localization path, with the
covariance identities
\eqref{eq:localization-covariance-dissipation}--\eqref{eq:localization-hessian-budget}
supplying the bookkeeping.

The main early applications of stochastic localization~\cite{eldan2013thinshell}
were in high-dimensional geometry.  For an isotropic log-concave measure, the
Kannan--Lov\'asz--Simonovits (KLS) conjecture asks, in one formulation, for a
dimension-free lower bound on the Cheeger isoperimetric constant.  This problem
is closely tied to spectral gaps, Poincar\'e inequalities, thin-shell estimates,
and concentration of measure.
Eldan introduced stochastic localization to relate thin-shell and
spectral-gap/KLS bounds up to logarithmic factors~\cite{eldan2013thinshell}.
Lee and Vempala developed the method further for isoperimetry, concentration,
and mixing, improving the known KLS-type bounds for isotropic log-concave
measures~\cite{lee2017eldans}.  Later work of Chen obtained an almost constant
lower bound for the KLS isoperimetric coefficient using related localization
techniques~\cite{chen2021kls}.  Chen and Eldan also made the sampling connection
explicit: their localization-schemes framework associates Markov chains to
localized martingales of measures and proves mixing bounds by analyzing the
localization process~\cite{chen2022localizationschemes}.  For our purposes, the
lesson is that stochastic localization converts global geometric or sampling
questions into estimates on the random posterior covariance process.  The same
covariance-budget viewpoint is what later reappears in analyses of diffusion
models~\cite{montanari2023sampling,benton2024nearlydlinear}.

\subsection{Polchinski flow for the effective potential}

The localization picture describes the Gaussian channel through posterior
quantities.  Polchinski flow describes the same channel at the level of the
noised marginals.  In diffusion-model language, this is the variance-exploding
forward process from the previous section:
\[
    X_t=X_0+\sqrt{t}Z,
\]
whose marginal density is the Gaussian-smoothed density
\[
    p_t=p_0*\normal{0}{t\Id}.
\]
It solves the heat equation
\begin{equation}\label{eq:heat-flow-density}
    \partial_t p_t=\frac12\Delta p_t.
\end{equation}
Define the effective potential
\[
    U_t(x)=-\log p_t(x).
\]
The diffusion-model score at noise level $t$ is therefore
$\sstar_t=\nabla\log p_t=-\nabla U_t$.  Thus Polchinski flow is not a
new sampling process; it is a way to track how the same score field evolves as
the forward diffusion smooths the data distribution.
The heat equation becomes the nonlinear PDE
\begin{equation}\label{eq:polchinski-finite}
    \partial_t U_t
    =
    \frac12\Delta U_t
    -
    \frac12\norm{\nabla U_t}^2.
\end{equation}
This is the \emph{finite-dimensional
Polchinski equation}.  It is the isotropic Gaussian-convolution version of the
Polchinski renormalization
flow~\cite{polchinski1984renormalization,bauerschmidt2024polchinski}.
The density formulation is linear, but the effective potential formulation is
nonlinear because logarithms turn smoothing into a viscous Hamilton--Jacobi
equation.

\begin{proof}
Since $p_t=e^{-U_t}$,
\[
    \Delta p_t
    =
    \Delta(e^{-U_t})
    =
    e^{-U_t}\bigl(\norm{\nabla U_t}^2-\Delta U_t\bigr).
\]
Using \eqref{eq:heat-flow-density},
\begin{equation*}
    \partial_t U_t
    =
    -\frac{\partial_t p_t}{p_t}
    =
    -\frac12\frac{\Delta p_t}{p_t}
    =
    \frac12\Delta U_t-\frac12\norm{\nabla U_t}^2. \qedhere
\end{equation*}
\end{proof}

Because the score is the negative gradient of the effective potential,
\[
    \sstar_t(x)=\nabla\log p_t(x)=-\nabla U_t(x),
\]
the score satisfies the viscous Burgers equation
\begin{equation}\label{eq:burgers-score}
    \partial_t \sstar_t
    =
    \frac12\Delta \sstar_t
    +
    \nabla\left(\frac12\norm{\sstar_t}^2\right).
\end{equation}
Diffusion models are usually trained by score matching, rather than solving the Burgers equation.  The equation
nevertheless explains why regularity of the score improves at positive noise
levels and why small-noise estimates are delicate: as $t\downarrow0$, the
smoothing that controls derivatives disappears.

The Polchinski equation \eqref{eq:polchinski-finite} is useful because it shows what forward noising does to the energy landscape.  Sharp
wells, ridges, and nonsmooth features of $U_0=-\log p_0$ are averaged out by
convolution.  At positive noise, $U_t$ is smoother and its gradient is better
behaved.  The reverse model is trying to follow this family of smoothed
landscapes backward, where the geometry becomes progressively sharper.

This is the renormalization interpretation of diffusion models.  Historically,
renormalization entered statistical physics through the idea that one should
describe a system at a chosen resolution and then track how the effective
description changes when microscopic degrees of freedom are averaged out.
Kadanoff's block-spin
scaling picture~\cite{kadanoff1966scaling} gave the physical image of
coarse-graining near criticality.  Wilson's renormalization group made this
image into a systematic flow of effective theories and explained universality
at critical points~\cite{wilson1971rg1,wilson1971rg2}; the review of
Wilson and Kogut~\cite{wilson1974renormalization} is the classical reference.
Polchinski's exact renormalization group equation~\cite{polchinski1984renormalization}
then expressed this flow directly at the level of effective Lagrangians.  The
finite-dimensional equation \eqref{eq:polchinski-finite} is the diffusion-model
shadow of the same idea.

Renormalization group language says that forward noising integrates out
microscopic information.  In image models this statement is metaphorical but
useful: the data distribution is gradually blurred, high-frequency details
become harder to distinguish, and the effective potential $U_t$ describes
the log-density of the coarse-grained law.  In field-theoretic applications the
same idea can be literal: one chooses a covariance schedule that removes
degrees of freedom by scale, and the corresponding effective action follows a
Polchinski-type equation.
In the work of Bauerschmidt and collaborators, this flow is used precisely in
this rigorous sense: it is an evolution equation for scale-dependent effective
potentials, so estimates along the flow turn renormalization into analytic
control of stochastic dynamics and Gibbs measures~\cite{bauerschmidt2024polchinski}.

\subsection{A dictionary}

The same Gaussian smoothing operation can now be summarized in three languages:
diffusion dynamics, posterior localization, and renormalization by Gaussian
coarse-graining.
\begin{center}
\small
\begin{tabular}{p{0.30\textwidth}|p{0.30\textwidth}|p{0.30\textwidth}}
Diffusion-model view & Stochastic-localization view & Polchinski/RG view\\
\hline
Forward noising $X_t=X_0+\sqrt{t}Z$ &
Observation $\dd Y_u=X_\star\dd u+\dd W_u$ &
Coarse-graining by Gaussian convolution\\
Noised law $p_t$ &
Posterior law $\mu_u(\cdot\mid Y_u)$ with $t=1/u$ &
Effective potential $U_t=-\log p_t$\\
Score $\sstar_t=\nabla\log p_t$ &
Posterior mean $M_u=\E[X_\star\mid Y_u]$ &
Action gradient $-\nabla U_t$\\
Tweedie denoiser $x+t\sstar_t(x)$ &
Localization mean $m_u(y)$ &
Inverse RG drift for sampling\\
Reverse SDE or probability flow ODE &
Martingale posterior-mean process &
Hamilton--Jacobi/Polchinski PDE\\
\end{tabular}
\end{center}

\section{Discretizing Continuous Diffusion Models}\label{sec:discretizing}

We now pass from ideal continuous dynamics to algorithms.  A sampler must
choose a finite time grid and replace each exact reverse transition by a
computable transition kernel.  This section sets up the formulas needed
for the numerical analysis.

A useful way to keep the presentation organized is to separate the three
implementation questions.  First, the reverse chain must be initialized from a
simple high-noise law rather than the exact terminal law.  Second, the reverse
dynamics use learned scores or denoisers rather than the true ones.  Third, the
continuous reverse dynamics, or equivalently the exact Bayes reverse kernels on
a grid, must be discretized.  The analysis is built around exactly this
trichotomy: initialization error, statistical score error, and discretization
error.

\subsection{The DDPM grid, scores, and exact reverse kernels}
\label{subsec:ddpm-grid}

Let $X_0\sim\pdata$ on $\R^d$.  Choose a time grid
\[
    0=t_0<t_1<\cdots<t_K=T.
\]
The goal is to replace the continuous noising process by a finite Markov chain with the same prescribed one-time
marginals.  Recall from \eqref{eq:continuous-marginal} that these marginals are
\[
    X_t\mid X_0\sim\normal{a_tX_0}{\sigma_t^2\Id}.
\]
On the grid, we write
\[
   X_k = X_{t_k}, \qquad  a_k=a_{t_k}, \qquad \sigma_k=\sigma_{t_k}.
\]
The grid schedule determines the one-step noising parameters.  We choose a
Gaussian forward transition
\begin{equation}\label{eq:one-step-noising}
    X_{k+1}\mid X_k
    \sim
    \normal{\alpha_k X_k}{\alpha_k^2\eta_k\Id}.
\end{equation}
If
\[
    X_k=a_kX_0+\sigma_k Z_k,\qquad Z_k\sim\normal{0}{\Id},
\]
then \eqref{eq:one-step-noising} is equivalently
\[
    X_{k+1}=\alpha_kX_k+\alpha_k\sqrt{\eta_k}\xi_k,
    \qquad \xi_k\sim\normal{0}{\Id}
\]
and hence
\[
    X_{k+1}\mid X_0
    \sim
    \normal{
    \alpha_k a_kX_0
    }{
    \alpha_k^2(\sigma_k^2+\eta_k)\Id
    }.
\]
Matching this with the target marginal
$X_{k+1}\mid X_0\sim\normal{a_{k+1}X_0}{\sigma_{k+1}^2\Id}$ gives
\[
    \alpha_k=\frac{a_{k+1}}{a_k},
    \qquad
    \eta_k=\frac{\sigma_{k+1}^2}{\alpha_k^2}-\sigma_k^2.
\]
The variance-preserving case has $a_k^2+\sigma_k^2=1$.  The
variance-exploding case has $a_k\equiv1$ and an increasing noise scale $\sigma_k^2=t_k$.

The reverse sampler has to invert one grid step: given a noisy point at level $k+1$, it needs the
conditional law of the previous point at level $k$.  This is where denoising and
score information enter.  Let $p_k$ denote the density of $X_k$.  The true score at time $k$ is
\[
    \sstar_k(x)=\nabla\log p_k(x).
\]

The score determines the posterior mean of the clean sample.  This is the
grid-time specialization of the continuous-time Tweedie identity
\eqref{eq:continuous-tweedie}, with $a_t,\sigma_t$ read off on the grid:
\begin{equation}\label{eq:tweedie}
    \sstar_k(x)
    =
    \frac{1}{\sigma_k^2}
    \E[a_kX_0-X_k\mid X_k=x].
\end{equation}
Equivalently, the grid-time version of the optimal denoiser from
Subsection~\ref{subsec:continuous-tweedie} is
\[
    \dstar_k(x)
    :=
    \E[X_0\mid X_k=x]
    =
    a_k^{-1}\bigl(x+\sigma_k^2\sstar_k(x)\bigr).
\]
This is the grid-time notation corresponding to the posterior mean $m_u(y)$
from stochastic localization: in the variance-exploding normalization
$\sigma_t^2=t$, one has $\dstar_t(x)=x+t\sstar_t(x)=m_{1/t}(x/t)$.
A score model $\score_k$ is equivalently a denoiser model
\[
    \denoiser_k(x)
    :=
    a_k^{-1}\bigl(x+\sigma_k^2\score_k(x)\bigr),
    \qquad
    \denoiser_k(x)-\dstar_k(x)
    =
    a_k^{-1}\sigma_k^2\bigl(\score_k(x)-\sstar_k(x)\bigr).
\]
In the variance-exploding normalization $a_t\equiv1$ and $\sigma_t^2=t$, the
reverse drift can be written as
\[
    \sstar_t(x)
    =
    \frac{\dstar_t(x)-x}{t}.
\]
Thus one may either think of the sampler as using a score field, or as using an
optimal denoiser whose value is frozen over each discrete reverse step.

Before approximating the reverse step with a score model, let us write the
exact Bayes kernel in terms of the marginal density $p_k$.  Let
$P_k^{\leftarrow}(x' \to \cdot)$ be the density of the backward transition
$\Law(X_k\mid X_{k+1}=x')$, and let $P_k^{\to}(x \to x')$ denote the forward
transition density in \eqref{eq:one-step-noising}.  For fixed $x$,
\[
    P_k^{\to}(x \to x')
    =
    (2\pi\alpha_k^2\eta_k)^{-d/2}
    \exp\left(
    -\frac{\norm{x'-\alpha_kx}^2}{2\alpha_k^2\eta_k}
    \right).
\]
Bayes' rule gives
\[
    P_k^{\leftarrow}(x' \to x)
    =
    \frac{p_k(x)P_k^{\to}(x \to x')}{p_{k+1}(x')},
    \qquad
    p_{k+1}(x')=\int p_k(y)P_k^{\to}(y \to x')\,\dd y.
\]
So the exact
reverse kernel is
\begin{equation}\label{eq:backward-kernel}
    P_k^{\leftarrow}(x' \to x)
    \propto
    p_k(x)
    \exp\!\left(
    -\frac{\norm{x-\alpha_k^{-1}x'}^2}{2\eta_k}
    \right).
\end{equation}
Thus each reverse step is a Bayes update: $p_k$ is the prior at the previous
noise level, and the Gaussian factor is the likelihood of the rescaled
observation $\alpha_k^{-1}x'$ under the channel
$\alpha_k^{-1}X_{k+1}=X_k+\sqrt{\eta_k}\xi_k$.  If $p_k$ were known as a
density, \eqref{eq:backward-kernel} could be used directly.  In diffusion
models, the learned object is only the score $\nabla\log p_k$, so the next step
is to approximate this Bayes update locally.

\subsection{The Gaussian reverse update}

A common sampler replaces \eqref{eq:backward-kernel} by a Gaussian transition.
With the notation of this section,
\begin{equation}\label{eq:ddpm-approx}
    \wh P_k^{\leftarrow}(x' \to \cdot)
    =
    \normal{\alpha_k^{-1}x'+\eta_k\alpha_k\score_{k+1}(x')}
    {\eta_k\Id},
\end{equation}
where we use the same step-size parameter $\eta_k$ for the Gaussian variance.
This should be contrasted with the exact reverse kernel
\eqref{eq:backward-kernel}, which need not be Gaussian because of the prior
factor $p_k(x)$.

To see why the mean in \eqref{eq:ddpm-approx} has this form, apply Tweedie's
identity to the single forward step itself.  After rescaling the next
observation,
\[
    \alpha_k^{-1}X_{k+1}
    =
    X_k+\sqrt{\eta_k}\xi_k.
\]
Tweedie's identity for this channel gives
\[
\begin{aligned}
    \E[X_k\mid X_{k+1}=x']
    & =
    \alpha_k^{-1}x'
    +
    \alpha_k \eta_k\nabla\log\bigl( p_{k+1}(x')\bigr) \\
    & =
    \alpha_k^{-1}x'
    +
    \alpha_k \eta_k \sstar_{k+1}(x').
\end{aligned}
\]
Thus the center of \eqref{eq:ddpm-approx} is the exact posterior mean with
$\sstar_{k+1}$ replaced by the learned score $\score_{k+1}$.  The remaining
approximation is to replace the exact posterior covariance by the local Gaussian
scale $\eta_k\Id$: when the prior density $p_k$ is smooth, the likelihood in
\eqref{eq:backward-kernel} confines $X_k$ to a ball of radius
$\sqrt{\eta_k}$ around the observation, up to leading order.

Let us also present a complementary derivation from the numerical SDE viewpoint. Choose a
forward noising SDE on $[0,T]$ whose marginals agree with the channel
$X_t\mid X_0\sim\normal{a_tX_0}{\sigma_t^2\Id}$.  Assume the schedule is smooth,
$a_t>0$, with the usual normalization $a_0=1$ and $\sigma_0=0$, and define
\[
    \lambda_t=\frac{\dot a_t}{a_t},
    \qquad
    f_t(x)=\lambda_t x,
    \qquad
    g_t^2=\frac{\dd}{\dd t}\sigma_t^2-2\lambda_t\sigma_t^2.
\]
Assume the last quantity is non-negative, so that $g_t$ is well defined, then the linear SDE
\[
    \dd X_t=\lambda_tX_t\dd t+g_t\dd B_t
\]
has conditional mean $a_tX_0$ and conditional variance $\sigma_t^2\Id$:
the variance solves $\dot v_t=2\lambda_tv_t+g_t^2$ with $v_0=0$, hence
$v_t=\sigma_t^2$.  Applying \eqref{eq:reverse-sde} to this particular choice of
$f_t$ and $g_t$ gives the reverse SDE
\[
    \dd Y^{\leftarrow}_s
    =
    \left\{
    -\lambda_{T-s}Y^{\leftarrow}_s
    +
    g_{T-s}^2\nabla\log p_{T-s}(Y^{\leftarrow}_s)
    \right\}\dd s
    +
    g_{T-s}\dd B^{\leftarrow}_s.
\]
Now apply Euler--Maruyama to this reverse SDE on the grid.  For the
reverse step from time $t_{k+1}$ to time $t_k$, write
$h_k=t_{k+1}-t_k$ and start from $x'$ at noise level $t_{k+1}$.  Freezing the
coefficients and the score at the beginning of this reverse step, namely at
$t_{k+1}$ in forward time, and using the grid shorthand
$\lambda_{k+1}:=\lambda_{t_{k+1}}$ and $g_{k+1}:=g_{t_{k+1}}$, gives
\[
    x'
    -h_k\lambda_{k+1}x'
    +h_kg_{k+1}^2\score_{k+1}(x')
    +g_{k+1}\sqrt{h_k}\,\xi.
\]
The discrete parameters are exactly the grid versions of the same schedule:
\[
    \alpha_k=\frac{a_{k+1}}{a_k},
    \qquad
    \eta_k=\frac{\sigma_{k+1}^2}{\alpha_k^2}-\sigma_k^2.
\]
A Taylor expansion at $t_{k+1}$ gives
\[
    \alpha_k^{-1}
    =
    1-h_k\lambda_{k+1}+O(h_k^2),
    \qquad
    \alpha_k\eta_k
    =
    h_kg_{k+1}^2+O(h_k^2),
\]
which agrees with the Gaussian reverse update discussed above.
It is worth being explicit about the two step parameters here, since both recur
in the error analysis: $h_k=t_{k+1}-t_k$ is the time increment of the reverse
step, while $\eta_k$ is the variance of the one-step Gaussian kernel
\eqref{eq:one-step-noising}.  In a general schedule they are different
quantities, related by $\alpha_k\eta_k=h_kg_{k+1}^2+O(h_k^2)$, but they coincide,
$\eta_k=h_k$, in the variance-exploding normalization ($\alpha_k=1$,
$g_t\equiv1$) used in Section~\ref{sec:error-analysis}.
The update \eqref{eq:ddpm-approx} can therefore be viewed either as
a local approximation to the exact Bayes kernel or as an Euler--Maruyama
discretization of the reverse SDE associated with the same Gaussian noising
channel.  The discretization analysis asks how much error is introduced by
making this local replacement at every reverse step.

\section{Error Analysis for Diffusion Models}
\label{sec:error-analysis}

The previous section isolated the local numerical question: each step has an
exact Bayes kernel \eqref{eq:backward-kernel}, while a practical sampler uses an
implementable approximation such as \eqref{eq:ddpm-approx}.  We now turn this
local comparison into a global guarantee.  Beyond the early-stopping gap between
the smoothed target and the raw data law, the reverse-chain error splits into
the same three sources isolated in the previous section: how the chain is
initialized, how accurately the score has been learned, and how accurately each
reverse kernel is discretized.

Before getting into the details, it is worth laying out the roadmap of the whole analysis; each piece is made precise in the subsections that
follow.  Throughout, $P_k^{\leftarrow}(x_{k+1} \to \cdot)$ denotes the exact
reverse kernels \eqref{eq:backward-kernel}, and
$\wh P_k^{\leftarrow}(x_{k+1} \to \cdot)$ the implemented reverse kernels of a
sampler.  For each time $k$, the $L^2(p_k)$ score error is
\begin{equation}\label{eq:score-error}
    \epsscorek^2
    =
    \E_{X_k\sim p_k}
    \norm{\score_k(X_k)-\sstar_k(X_k)}^2.
\end{equation}
The implemented reverse chain is initialized at the high-noise endpoint with
error
\[
    \KL(p_K\|\wh p_K)\leq \epsinit^2,
\]
and the accumulated one-step kernel error splits into a pure discretization part
and a score-error part,
\[
    \sum_{k=1}^{K-1}
    \E_{X_{k+1}\sim p_{k+1}}
    \KL\!\left(
        P_k^{\leftarrow}(X_{k+1} \to \cdot)
        \middle\|
        \wh P_k^{\leftarrow}(X_{k+1} \to \cdot)
    \right)
    \leq
    \epsdisc^2
    +
    c\sum_{k=1}^{K-1}\eta_k\eps_{k+1,\mathsf{score}}^2.
\]
Putting these together, the error decomposition we will establish reads
\begin{equation}\label{eq:ddpm-error-decomposition}
    \KL(p_1\|\wh p_1)
    \leq
    \epsinit^2
    +
    \epsdisc^2
    +
    c\sum_{k=1}^{K-1}\eta_k\eps_{k+1,\mathsf{score}}^2,
\end{equation}
with the early-stopping comparison between $p_1$ and $\pdata$ handled separately
as a geometric smoothing estimate, naturally controlled in $W_2$.

This decomposition is mainly a bookkeeping statement.  Instead of comparing only the
final samples, we compare the exact and implemented reverse chains as full
paths.  Data processing then says that the KL between the final marginals is no
larger than the path-space KL\@.  The path-space chain rule breaks this latter
quantity into two pieces: the mismatch between the two initial laws at the
high-noise endpoint, and the sum of the one-step KL errors accumulated along the
reverse chain.

The factor multiplying the score error has a simple origin.  In Gaussian DDPM
or Euler--Maruyama updates, the learned score changes the mean of the one-step
Gaussian kernel by about $\eta_k(\score_{k+1}-\sstar_{k+1})$.  KL between
Gaussians with the same covariance $\eta_k\Id$ divides the squared mean error by
the variance, leaving a contribution of order
$\eta_k\eps_{k+1,\mathsf{score}}^2$.  Thus score errors at times with
larger reverse steps matter more, while errors on very small steps are naturally
downweighted.  Here $\eta_k$ is the one-step variance parameter of
Section~\ref{subsec:ddpm-grid}; in the variance-exploding chain analyzed in the
rest of this section it is exactly the reverse-time step $h_k=t_{k+1}-t_k$, which
is why these weights act as step sizes.

The term $\epsdisc^2$ collects the error that would remain even
with the exact score.  For Euler--Maruyama this is the cost of freezing the score
or denoiser during a step.  The stochastic-localization argument controls that
cost through posterior covariance budgets, and the high-accuracy correction later
reduces it by sampling the local Gaussian tilt more faithfully.  The subsections
below unpack the pieces in this order: early stopping, KL telescoping,
score-error contribution, Euler--Maruyama discretization and its localization
refinement, and finally the high-accuracy correction.

\subsection{Early stopping}

The first issue is the endpoint of the reverse process.  We stop at a small
positive noise level and target the smoothed law $p_1$.  This
early-stopped law is the right object for the KL analysis: path-space chain
rules and Girsanov-type estimates give direct control of $\KL(p_1\|\wh p_1)$
for the sampler output $\wh p_1$.  The raw data distribution may be singular,
supported on a lower-dimensional set, or otherwise mutually singular with any
smooth sampler output, so it cannot generally be compared to $\wh p_1$ in KL.

Thus we keep two error measurements separate.  The early-stopping error is a
geometric smoothing error and is controlled in $W_2$.  The algorithmic error
between $p_1$ and $\wh p_1$ is controlled later in KL.

The $W_2$ estimate is immediate from the forward noising coupling.  In the
grid notation of the previous section, $X_1=a_1X_0+\sigma_1Z$, so coupling
$X_1$ with the same clean sample $X_0$ gives
\begin{equation}\label{eq:early-stop-w2}
    W_2^2(\pdata,p_1)
    \leq
    (1-a_1)^2\E\norm{X_0}^2+\sigma_1^2d.
\end{equation}
Thus early stopping is harmless when $\sigma_1$ and $1-a_1$ are small relative
to the scale of the data.

\begin{proof}[Proof of the early-stopping estimate]
Couple $p_1$ and $\pdata$ by using the same clean data point $X_0$ and the
same Gaussian noise used to form
\[
    X_1=a_1X_0+\sigma_1Z,\qquad Z\sim\normal{0}{\Id}.
\]
Then
\[
    X_1-X_0=(a_1-1)X_0+\sigma_1Z.
\]
Using $\E Z=0$ and independence of $Z$ and $X_0$,
\[
    \E\norm{X_1-X_0}^2
    =
    (1-a_1)^2\E\norm{X_0}^2+\sigma_1^2\E\norm{Z}^2
    =
    (1-a_1)^2\E\norm{X_0}^2+\sigma_1^2d.
\]
Since $W_2^2$ is the infimum over all couplings, it is bounded by the cost
of this particular coupling.
\end{proof}

\subsection{KL telescoping transition error}

The algorithmic comparison in this section is
\[
    \KL(p_1\|\wh p_1),
\]
where $p_1$ is the law produced by the exact reverse chain and $\wh p_1$ is the
law produced by the implemented chain, both stopped at the first positive noise
level.  We control this final KL by first comparing the two reverse chains as
full paths and then projecting those paths to their endpoint.

The reason this helps is simple: a Markov chain path density is a product of an
initial density and transition densities.  Taking a logarithm turns that product
into a sum, and taking expectation gives an additive KL identity.  This is the
basic telescoping mechanism behind the error decomposition: a long reverse
sampling procedure is reduced to many one-step reverse-kernel comparisons.

\begin{lemma}[KL chain rule for reverse Markov kernels]\label{lem:kl-chain}
Let $\mu_KP_{K-1}\cdots P_1$ and $\nu_KQ_{K-1}\cdots Q_1$ be two path laws on
$(X_K,\ldots,X_1)$, where $P_k(x_{k+1} \to \cdot)$ and
$Q_k(x_{k+1} \to \cdot)$ transition from level $k+1$ to level $k$.  Let
$\mu_{k+1}$ be the marginal law of $X_{k+1}$ under the first path law.  Then
\[
    \KL(\mu_KP_{K-1}\cdots P_1\|\nu_KQ_{K-1}\cdots Q_1)
    =
    \KL(\mu_K\|\nu_K)
    +
    \sum_{k=1}^{K-1}
    \E_{X_{k+1}\sim\mu_{k+1}}
    \KL(P_k(X_{k+1} \to \cdot)\|Q_k(X_{k+1} \to \cdot)).
\]
\end{lemma}

\begin{proof}
Write the two path densities as
\[
    \mu_K(x_K)\prod_{k=1}^{K-1}P_k(x_{k+1} \to x_k),
    \qquad
    \nu_K(x_K)\prod_{k=1}^{K-1}Q_k(x_{k+1} \to x_k).
\]
The log-ratio is therefore
\[
    \log\frac{\mu_K(x_K)}{\nu_K(x_K)}
    +
    \sum_{k=1}^{K-1}
    \log\frac{P_k(x_{k+1} \to x_k)}{Q_k(x_{k+1} \to x_k)}.
\]
Taking expectation under the first path law gives the initial term
$\KL(\mu_K\|\nu_K)$.  For the $k$th transition term, condition on
$X_{k+1}$.  Under the first path law, the conditional law of $X_k$ given
$X_{k+1}$ is $P_k(X_{k+1} \to \cdot)$, while the law of $X_{k+1}$ itself is
$\mu_{k+1}$.  Thus
\[
    \E
    \log\frac{P_k(X_{k+1} \to X_k)}{Q_k(X_{k+1} \to X_k)}
    =
    \E_{X_{k+1}\sim\mu_{k+1}}
    \KL(P_k(X_{k+1} \to \cdot)\|Q_k(X_{k+1} \to \cdot)).
\]
Summing over $k$ gives the identity.  The general measure-theoretic statement
follows by replacing densities with Radon--Nikodym derivatives.
\end{proof}

\begin{remark}
The marginal KL at the final time is at most the path-space KL by data
processing; see Lemma~\ref{lem:data-processing} in Appendix~\ref{app:ito-girsanov}.
Therefore a path-space calculation gives a valid, often convenient upper bound
for the final sampling error.  This is the path-space substitute for the triangle
inequality that KL lacks: the additivity comes from comparing two full path laws
with factored initial distributions and transition kernels.
\end{remark}

\begin{prop}[From local reverse-kernel error to final error]
\label{prop:local-to-global}
Let the exact reverse chain start from $p_K$ and use the exact reverse kernels
$P_k^{\leftarrow}(x_{k+1} \to \cdot)$ from \eqref{eq:backward-kernel},
$k=1,\ldots,K-1$, producing final law $p_1$.
Let an approximate reverse chain start from $\wh p_K$ and use kernels
$\wh P_k^{\leftarrow}(x_{k+1} \to \cdot)$, producing final law $\wh p_1$.  Then
\[
    \KL(p_1\|\wh p_1)
    \leq
    \KL(p_K\|\wh p_K)
    +
    \sum_{k=1}^{K-1}
    \E_{X_{k+1}\sim p_{k+1}}
    \KL\!\left(
        P_k^{\leftarrow}(X_{k+1} \to \cdot)
        \middle\|
        \wh P_k^{\leftarrow}(X_{k+1} \to \cdot)
    \right).
\]
\end{prop}

\begin{proof}
Consider the exact path law of $(X_K,\ldots,X_1)$ and the approximate path law
of $(\wh X_K,\ldots,\wh X_1)$.  Apply Lemma~\ref{lem:kl-chain} to these path
laws.
The initial contribution is $\KL(p_K\|\wh p_K)$, and the transition
contribution at reverse step $k$ is
\[
    \E_{X_{k+1}\sim p_{k+1}}
    \KL\!\left(
        P_k^{\leftarrow}(X_{k+1} \to \cdot)
        \middle\|
        \wh P_k^{\leftarrow}(X_{k+1} \to \cdot)
    \right).
\]
Thus the KL divergence between the two full path laws is at most the right-hand
side above.  Projecting a path to its final coordinate is a measurable map, so
data processing for KL gives the claimed bound on $\KL(p_1\|\wh p_1)$.
\end{proof}

This proposition is the path-space part of the error decomposition
\eqref{eq:ddpm-error-decomposition}.  The remaining work is to estimate
the one-step terms in the sum.

\begin{remark}[Putting the metrics together]\label{rem:combine-metrics}
The main estimates above live in their natural metrics: $W_2$ for early
stopping and KL for the reverse-chain error.  If one wants a single weak
observable guarantee, one can introduce the bounded-Lipschitz distance
\[
    \BL(\mu,\nu)
    =
    \sup_{\norm f_\infty\leq1,\;\norm f_{\mathrm{Lip}}\leq1}
    \abs{\E_\mu f-\E_\nu f}.
\]
Then $\BL(\pdata,p_1)\leq W_2(\pdata,p_1)$ by the coupling argument in the
definition of $W_2$: for any $1$-Lipschitz $f$ and any coupling $(X,Y)$ of
$\mu$ and $\nu$,
\[
    \abs{\E f(X)-\E f(Y)}\leq \E\norm{X-Y}
    \leq \sqrt{\E\norm{X-Y}^2},
\]
and then one takes the infimum over couplings.  On the algorithmic side,
\[
    \BL(p_1,\wh p_1)\leq 2\TV(p_1,\wh p_1)
    \leq 2\sqrt{\KL(p_1\|\wh p_1)/2}.
\]
Thus a final $\BL$ statement can be obtained after the native $W_2$ and KL
estimates are proved, but it is not the metric driving either part of the
analysis.
\end{remark}

\subsection{How score error contributes to KL}

Proposition~\ref{prop:local-to-global} reduces the global KL error to
local KL comparisons between the exact reverse kernel $P_k^{\leftarrow}$ and the
implemented kernel $\wh P_k^{\leftarrow}$.  Each local comparison can contain several
sources of error.  One source is numerical: even with the exact score, the
implemented transition may only approximate the Bayes reverse kernel.  Another
source is statistical: the sampler uses a learned score $\score$ in place of
the true score $\sstar$.

In the rest of this section we specialize the DDPM grid of
Section~\ref{subsec:ddpm-grid} to the variance-exploding case
($\alpha_k=1$, $\sigma_k^2=t_k$), now run backward from a large noise level $T$
down to an early-stopping level $\delta>0$:
\[
    \delta=t_1<t_2<\cdots<t_K=T,\qquad h_k=t_{k+1}-t_k.
\]
The reverse step indexed by $k$ moves from noise level $t_{k+1}$ to $t_k$, for
$k=1,\ldots,K-1$, so the chain ends at $t_1=\delta$; the final law written $p_1$
in the earlier subsections is thus the early-stopped law $p_\delta$.  The one-step variance
parameter then reduces to the time increment, $\eta_k=h_k$, and we write $h_k$
from here on to emphasize that the reverse step is a genuine time step of an SDE
discretization.

This subsection isolates the statistical part.  The implemented reverse step is
the Gaussian update \eqref{eq:ddpm-approx} -- equivalently, an Euler--Maruyama
discretization of the reverse SDE -- in which the score enters only through the
mean of the kernel.  A score error therefore becomes a mean error, which we now
quantify.

To keep the notation separate, first fix the Gaussian discretization.  Let
\begin{equation}\label{eq:em-exact-score}
    P_k^{\mathsf{EM}}(x' \to \cdot)
    =
    \normal{x'+h_k\sstar_{k+1}(x')}{h_k\Id}
\end{equation}
be the exact-score Euler kernel, and let
\begin{equation}\label{eq:em-learned-score}
    \wh P_k^{\mathsf{EM}}(x' \to \cdot)
    =
    \normal{x'+h_k\score_{k+1}(x')}{h_k\Id}
\end{equation}
be the learned-score Euler kernel.  These are not the exact Bayes reverse
kernel $P_k^{\leftarrow}$; they are the true-score and learned-score versions of
the same Gaussian approximation \eqref{eq:ddpm-approx}.  The following KL
therefore isolates the statistical score error after the discretization has
been fixed.  For fixed $x'$, applying Lemma~\ref{lem:gaussian-kl} gives
\[
    \KL\!\left(
    P_k^{\mathsf{EM}}(x' \to \cdot)
    \middle\|
    \wh P_k^{\mathsf{EM}}(x' \to \cdot)
    \right)
    =
    \frac{1}{2h_k}
    \norm{h_k
    \bigl(\score_{k+1}(x')-\sstar_{k+1}(x')\bigr)}^2.
\]
Averaging over $X_{k+1}\sim p_{k+1}$ and using the definition of
$\eps_{k+1,\mathsf{score}}^2$ gives
\begin{equation}\label{eq:gaussian-score-step}
    \E_{X_{k+1}\sim p_{k+1}}
    \KL\!\left(
    P_k^{\mathsf{EM}}(X_{k+1} \to \cdot)
    \middle\|
    \wh P_k^{\mathsf{EM}}(X_{k+1} \to \cdot)
    \right)
    =
    \frac{h_k}{2}
    \eps_{k+1,\mathsf{score}}^2,
\end{equation}
which is the score-estimation contribution in the error decomposition
\eqref{eq:ddpm-error-decomposition}, after the discretization choice has been
separated, with $\eta_k=h_k$ and $c=\tfrac12$.

Summing these one-step terms over the reverse chain, the total score
contribution has the form
\begin{equation}\label{eq:score-sum}
    \sum_{k=1}^{K-1}h_k
    \eps_{k+1,\mathsf{score}}^2,
\end{equation}
the discrete analogue of the path-space estimate
\[
    \int_\delta^T
    \E_{X_t\sim p_t}
    \norm{\score_t(X_t)-\sstar_t(X_t)}^2
    \,\dd t.
\]
The weight $h_k$ is what makes this estimate informative: a score model need not
be equally accurate at all times, since an error matters less when the sampler
takes only a small step, while an error during a large reverse step has a larger
effect.

\subsection{Euler--Maruyama}

We now turn to the numerical error of the simplest reverse-time integrator, the
Euler--Maruyama scheme, in the variance-exploding setting fixed at the start of
this section.  Here the initialization term in the error decomposition
\eqref{eq:ddpm-error-decomposition} is explicit.  If the implemented
reverse chain is initialized from $\normal{0}{T\Id}$, then
\begin{equation}\label{eq:ve-init-kl}
    \epsinit^2
    :=
    \KL(p_T\|\normal{0}{T\Id})
    \leq
    \int
    \KL\!\left(\normal{x}{T\Id}\middle\|\normal{0}{T\Id}\right)
    \pdata(\dd x)
    =
    \frac{\E\norm{X_0}^2}{2T}.
\end{equation}
Here we used the representation
$p_T=\int\normal{x}{T\Id}\,\pdata(\dd x)$ and convexity of KL in its first
argument.

The other, and more substantial, contribution is the discretization error
accumulated over the reverse steps, whose size is governed by the step lengths
$h_k$.  Choosing the grid well is therefore the central design question, and the
relevant local scale against which to measure $h_k$ is the current noise level.
Indeed, the reverse drift contains
$\sstar_t=(\dstar_t-x)/t$, so the same absolute step size is much more
aggressive near the early-stopping level than it is at large noise.  A uniform
grid would therefore force all steps to be as small as the lowest noise scale
$\delta$, leading to a poor dependence on $T/\delta$.  The standard
non-uniform-grid remedy is to use a relative control, as in Chen, Lee, and Lu~\cite{chen2023improvedsgm}:
\begin{equation}\label{eq:relative-mesh}
    h_k\leq \bar h\,t_{k+1},
    \qquad
    \bar h\leq \frac12.
\end{equation}
The canonical example is the geometric grid
\begin{equation}\label{eq:geometric-ve-grid}
    t_k=(1-\bar h)^{K-k}T,
    \qquad
    \delta=t_1=(1-\bar h)^{K-1}T.
\end{equation}
For this grid,
\[
    K\asymp \bar h^{-1}\log\frac{T}{\delta},
    \qquad
    \bar h\asymp \frac{\log(T/\delta)}{K}.
\]

With this convention, the exact Bayes reverse kernel is
\begin{equation}\label{eq:ve-backward-kernel}
    P_k^{\leftarrow}(x' \to \dd x)
    \propto_x
    p_k(x)
    \exp\!\left(-\frac{\norm{x-x'}^2}{2h_k}\right)\dd x,
\end{equation}
where $x'$ is the state at noise level $t_{k+1}$.  A first-order Gaussian
approximation freezes the score at the beginning of the reverse step.  This is
the exact-score Euler kernel $P_k^{\mathsf{EM}}$ from
\eqref{eq:em-exact-score}; the learned sampler uses
$\wh P_k^{\mathsf{EM}}$ from \eqref{eq:em-learned-score}, replacing
$\sstar_{k+1}$ by $\score_{k+1}$.
The superscript $\mathsf{EM}$ is just a reminder that the score has been
frozen during the step.  In the general DDPM notation of
\eqref{eq:ddpm-approx}, these formulas are obtained by setting
$\alpha_k=1$ and $\eta_k=h_k$.

There are two local error sources.  The score-estimation part has already been
isolated in the previous subsection: in the present variance-exploding
normalization, \eqref{eq:gaussian-score-step} contributes
$h_k\eps_{k+1,\mathsf{score}}^2$ up to constants.  The new issue in this
subsection is the numerical error caused by Euler--Maruyama itself.  Thus we
first compare the exact reverse kernel $P_k^{\leftarrow}$ with the exact-score
Euler kernel $P_k^{\mathsf{EM}}$.  This is the error made by freezing the true
reverse drift during one time step.

\begin{prop}[Euler--Maruyama discretization error]
\label{prop:em-local-error}
Suppose that, over reverse step $k$, the exact score field is Lipschitz at
scale $L_k$ on the region where the process typically lies.  Then the
exact-score Euler kernel satisfies a one-step estimate of the form
\[
    \E_{X_{k+1}\sim p_{k+1}}
    \KL\!\left(
        P_k^{\leftarrow}(X_{k+1} \to \cdot)
        \middle\|
        P_k^{\mathsf{EM}}(X_{k+1} \to \cdot)
    \right)
    \lesssim
    d\,L_k^2\,h_k^2,
\]
up to lower-order schedule factors.
\end{prop}

\begin{proof}
We write the single reverse step in reverse time.  Condition on the starting
point $X_{k+1}=x'$ at noise level $t_{k+1}$, and let
$0\leq r\leq h_k$.  The exact reverse diffusion on this interval has drift
$\sstar_{t_{k+1}-r}$:
\[
    \dd Y^{\leftarrow}_r=\sstar_{t_{k+1}-r}(Y^{\leftarrow}_r)\dd r+\dd B^{\leftarrow}_r,
    \qquad
    Y^{\leftarrow}_0=x'.
\]
Its endpoint law is $P_k^{\leftarrow}(x' \to \cdot)$.  The exact-score
Euler--Maruyama interpolation freezes the drift at the beginning of the step:
\[
    \dd \bar Y^{\leftarrow}_r=\sstar_{k+1}(x')\dd r+\dd B^{\leftarrow}_r,
    \qquad
    \bar Y^{\leftarrow}_0=x',
\]
whose endpoint law is precisely
$P_k^{\mathsf{EM}}(x' \to \cdot)$.  By the Girsanov KL formula in
Appendix~\ref{app:ito-girsanov}, followed by data processing
(Lemma~\ref{lem:data-processing}), the endpoint comparison is bounded by the
path comparison
\[
    \KL\!\left(
        P_k^{\leftarrow}(x' \to \cdot)
        \middle\|
        P_k^{\mathsf{EM}}(x' \to \cdot)
    \right)
    \leq
    \frac12
    \E\int_0^{h_k}
    \norm{
        \sstar_{t_{k+1}-r}(Y^{\leftarrow}_r)-\sstar_{k+1}(x')
    }^2\dd r.
\]
This display is the basic Euler discretization estimate: the whole cost is the
cost of replacing the moving score along the ideal reverse path by the frozen
score at the left endpoint.

Under the local Lipschitz bound, the score difference is controlled by the
motion of the reverse path during the step, up to lower-order terms coming from
the deterministic change of the noise level.  Since the Brownian displacement
over time $r$ has squared size of order $dr$, the local estimate is
\[
    \E\norm{
        \sstar_{t_{k+1}-r}(Y^{\leftarrow}_r)-\sstar_{k+1}(x')
    }^2
    \lesssim
    L_k^2\,d\,r,
\]
again suppressing schedule-dependent lower-order terms.  Integrating in
$r\in[0,h_k]$ gives
\[
    \frac12L_k^2d\int_0^{h_k} r\,\dd r
    \lesssim
    d\,L_k^2h_k^2.
\]
Finally average this conditional bound over $X_{k+1}\sim p_{k+1}$.
\end{proof}

Proposition~\ref{prop:em-local-error} is only the exact-score discretization
estimate.  When the sampler uses $\score_{k+1}$ instead of $\sstar_{k+1}$, the
extra contribution is the statistical Gaussian mean-error term computed in
\eqref{eq:gaussian-score-step}, which becomes
$h_k\eps_{k+1,\mathsf{score}}^2$ in the variance-exploding
normalization.

Combining Proposition~\ref{prop:em-local-error}, the score-error contribution
\eqref{eq:gaussian-score-step}, and Proposition~\ref{prop:local-to-global}
gives the standard Euler--Maruyama diffusion-sampling bound:
\begin{equation}\label{eq:em-global-bound}
    \KL(p_\delta\|\wh p_\delta)
    \lesssim
    \epsinit^2
    +
    \sum_{k=1}^{K-1} d\,L_k^2h_k^2
    +
    \sum_{k=1}^{K-1}
    h_k\eps_{k+1,\mathsf{score}}^2.
\end{equation}
The three terms are initialization at the high-noise endpoint, Euler
discretization, and score estimation.  This is the KL telescoping estimate of
Proposition~\ref{prop:local-to-global} with the local Gaussian calculation
inserted, so it is the Euler--Maruyama specialization of the error decomposition
\eqref{eq:ddpm-error-decomposition}.

\begin{remark}[Comparison with ULA]
The discretization term in \eqref{eq:em-global-bound} has the same origin as
the ULA error in \eqref{eq:ula-vw-kl}.  In both cases, the numerical scheme
freezes a drift over a short time interval.  The Brownian displacement during
that interval produces a local KL cost of order $dL^2h^2$, or $dL_k^2h_k^2$ on
a nonuniform diffusion grid, and summing over the grid gives an accumulated
discretization bias proportional to the mesh size: $dL^2h$ for ULA and
$\sum_k dL_k^2h_k^2$ here.  The difference is in the surrounding argument.  For
ULA, \eqref{eq:ula-vw-kl} is a mixing theorem for a Markov chain targeting a
fixed density, so it needs a global contraction mechanism such as a log-Sobolev
inequality, together with smoothness to control the Euler error.  For diffusion
models, the reference path $p_t$ is already supplied by the forward noising
process, and the reverse sampler is compared to this path by KL telescoping and
Girsanov.  Thus global structural assumptions on the data distribution, such as
convexity, log-Sobolev or isoperimetric inequalities, or warm starts, are not
needed; they are replaced by initialization and score-estimation terms.
\end{remark}

In fact, the discretization error in estimate \eqref{eq:em-global-bound} can be sharpened, as the next subsection shows.  Once
initialization and score estimation have been separated, the exact-score Euler
sampler pays only for freezing the true score
during each reverse step.  Let $t^+$ denote the next grid point at or above the
current noise level, so $t^+=t_{k+1}$ for $t\in[t_k,t_{k+1}]$.  The
discretization cost is measured by the pathwise quantity
\begin{equation}\label{eq:cll-bottleneck}
    \int_\delta^T
    \E\norm{
        \sstar_t(X_t)-\sstar_{t^+}(X_{t^+})
    }^2\dd t,
\end{equation}
up to deterministic schedule factors in more general schedules.  This is the
pathwise version of the finite-grid $\sum_k dL_k^2h_k^2$ term:
Euler--Maruyama freezes the score at the high-noise endpoint $t^+$ of each
reverse step, and the discretization error is governed by the change in the
true score along the ideal reverse path as the noise time decreases from
$t^+$ to $t$.

In the error analysis above, the passage from this score difference to the bound
$dL_k^2h_k^2$ used a worst-case control: the score was bounded by a Lipschitz
constant $L_k$, or equivalently by a pointwise Hessian bound for $\log p_t$.
The factor $d$ comes from the Brownian displacement during one step; the
potential loss comes from multiplying it by a uniform operator-norm bound on
the Hessian.  Under weak assumptions, this worst-case curvature can be much
larger than the average curvature seen by the reverse process, especially near
the early-stopping time.  The next subsection replaces this pointwise control
by an averaged posterior-covariance quantity, leading to the
nearly linear-in-$d$ discretization bound.

\subsection{Hessian control}
\label{subsec:denoiser-D-delta-T}

The difficulty left by the previous subsection is the pathwise score difference
in \eqref{eq:cll-bottleneck}.  The aim of this subsection is to replace the
worst-case Hessian estimate by the more averaged bound
\[
    \text{score-freezing cost}
    \;\lesssim\;
    \bar h
    \times
    \text{posterior-covariance budget}.
\]
We first derive the local score-freezing estimate where the factor $\bar h$
appears, and then rewrite the remaining weighted quadratic variation in terms
of posterior covariance.

For $t\in[t_k,t_{k+1}]$, write $t^+=t_{k+1}$.  The Girsanov
score-freezing cost in \eqref{eq:cll-bottleneck} is, up to the constant
$\frac12$ and schedule factors, the sum over grid intervals of
\[
    \int_{t_k}^{t_{k+1}}
    \E\norm{
        \sstar_t(X_t)-\sstar_{t^+}(X_{t^+})
    }^2\,\dd t.
\]

The useful variable is not the score itself but the optimal denoiser.  In the
variance-exploding normalization,
\[
    \dstar_t(x):=\E[X_0\mid X_t=x]=x+t\sstar_t(x),
    \qquad
    \sstar_t(x)=\frac{\dstar_t(x)-x}{t}.
\]
The stochastic-localization martingale identity
\eqref{eq:sl-mean-sde}, rewritten in the noise-time coordinate, says that for
$t\leq t^+$,
\[
    \E\norm{
    \dstar_t(X_t)-\dstar_{t^+}(X_{t^+})
    }^2
    =
    \int_t^{t^+}
    \E\norm{\nabla \dstar_r(X_r)}_{\mathrm F}^2\,\dd r.
\]
This is the replacement for a pointwise Hessian bound: it measures the actual
quadratic variation of the denoiser along the noising path.

Now estimate one grid interval.  Using
$\sstar_t=(\dstar_t-x)/t$ and keeping only the curvature term explicitly, the
score increment is bounded by the denoiser increment with the natural $t^{-2}$
weight:
\[
    \E\norm{
        \sstar_t(X_t)-\sstar_{t^+}(X_{t^+})
    }^2
    \lesssim
    \frac{1}{t^2}
    \E\norm{
        \dstar_t(X_t)-\dstar_{t^+}(X_{t^+})
    }^2.
\]
Here we have omitted some error terms coming from the explicit factor $x/t$, which can be handled easily.  Substituting the martingale identity and applying Fubini gives
\[
\begin{aligned}
    &\int_{t_k}^{t_{k+1}}
    \E\norm{
        \sstar_t(X_t)-\sstar_{t^+}(X_{t^+})
    }^2\,\dd t  \\
    &\qquad\lesssim
    \int_{t_k}^{t_{k+1}}
    \frac{1}{t^2}
    \int_t^{t_{k+1}}
    \E\norm{\nabla \dstar_r(X_r)}_{\mathrm F}^2\,\dd r\,\dd t  \\
    &\qquad=
    \int_{t_k}^{t_{k+1}}
    \E\norm{\nabla \dstar_r(X_r)}_{\mathrm F}^2
    \left(\int_{t_k}^r\frac{\dd t}{t^2}\right)\dd r  \\
    &\qquad=
    \int_{t_k}^{t_{k+1}}
    \E\norm{\nabla \dstar_r(X_r)}_{\mathrm F}^2
    \frac{r-t_k}{r\,t_k}\,\dd r.
\end{aligned}
\]
Since $r-t_k\leq h_k$ and the mesh condition
\eqref{eq:relative-mesh} with $\bar h\leq1/2$ gives
$t_k\geq t_{k+1}/2$, we get
\[
    \int_{t_k}^{t_{k+1}}
    \E\norm{
        \sstar_t(X_t)-\sstar_{t^+}(X_{t^+})
    }^2\,\dd t
    \lesssim
    \frac{h_k}{t_{k+1}}
    \int_{t_k}^{t_{k+1}}
    \frac{\E\norm{\nabla \dstar_r(X_r)}_{\mathrm F}^2}{r}\,\dd r.
\]
This is the point at which $\bar h$ enters: the ratio $h_k/t_{k+1}$ is the
relative step size on the interval, and \eqref{eq:relative-mesh} says it is at
most $\bar h$.  Summing over $k$ therefore gives
\[
    \sum_k
    \int_{t_k}^{t_{k+1}}
    \E\norm{
        \sstar_t(X_t)-\sstar_{t^+}(X_{t^+})
    }^2\,\dd t
    \lesssim
    \bar h
    \int_\delta^T
    \frac{\E\norm{\nabla \dstar_t(X_t)}_{\mathrm F}^2}{t}\,\dd t.
\]

It remains to express the weighted quadratic variation in a more intrinsic
form.  Let
\[
    \Sigma_t(x)=\Cov(X_0\mid X_t=x)
\]
be the posterior covariance in the variance-exploding noising channel.  The
linear-response identity \eqref{eq:posterior-mean-jacobian} and the
score-Jacobian identity \eqref{eq:score-hessian-covariance}, translated to this
notation, give
\[
    \nabla \dstar_t(x)
    =
    \Id+t\nabla\sstar_t(x)
    =
    t^{-1}\Sigma_t(x).
\]
The covariance-dissipation identity
\eqref{eq:localization-covariance-dissipation},
$\frac{\dd}{\dd u}\E[\Tr\Sigma_u] = -\E[\Tr(\Sigma_u^2)]$, after the change of
variables $t=1/u$, gives
\[
    \E\norm{\nabla \dstar_t(X_t)}_{\mathrm F}^2
    =
    \partial_t\,\E\Tr\Sigma_t(X_t).
\]
Thus
\[
    \int_\delta^T
    \frac{\E\norm{\nabla \dstar_t(X_t)}_{\mathrm F}^2}{t}
    \,\dd t
    =
    \int_\delta^T
    \frac{\partial_t\,\E\Tr\Sigma_t(X_t)}{t}\,\dd t.
\]
Integrating by parts,
\[
    \int_\delta^T
    \frac{\partial_t\,\E\Tr\Sigma_t(X_t)}{t}\,\dd t
    =
    \frac{\E\Tr\Sigma_T(X_T)}{T}
    -
    \frac{\E\Tr\Sigma_\delta(X_\delta)}{\delta}
    +
    \int_\delta^T
    \frac{\E\Tr\Sigma_t(X_t)}{t^2}\,\dd t.
\]
The middle term is nonpositive, so the first and third terms give an upper
bound.  Following the notation of
Chewi~\cite[Chapter~12]{chewi2024logconcave}, define the covariance budget
\begin{equation}\label{eq:D-delta-T}
    \mathfrak D_{\delta,T}(\pdata)
    :=
    \frac{\E_{X_T\sim p_T}\Tr\Sigma_T(X_T)}{T}
    +
    \int_\delta^T
    \frac{\E_{X_t\sim p_t}\Tr\Sigma_t(X_t)}{t^2}\,\dd t.
\end{equation}
Combining the previous displays gives the clean summary
\[
    \sum_k
    \int_{t_k}^{t_{k+1}}
    \E\norm{
        \sstar_t(X_t)-\sstar_{t^+}(X_{t^+})
    }^2\,\dd t
    \lesssim
    \bar h\,\mathfrak D_{\delta,T}(\pdata).
\]
This is the key identification: the discretization error is governed by the
weighted quadratic variation of the optimal denoiser, and
$\mathfrak D_{\delta,T}$ packages that variation after the DDPM time weights
have been accounted for.  A careful version of this estimate, together with the
score-error term, gives
\begin{equation}\label{eq:chewi-ddpm-D}
    \KL(p_\delta\|\wh p_\delta)
    \leq
    \epsinit^2
    +
    4\sum_{k=1}^{K-1} h_k\,
    \E_{p_{k+1}}
    \norm{\score_{k+1}-\sstar_{k+1}}^2
    +
    2\bar h\,\mathfrak D_{\delta,T}(\pdata),
\end{equation}
where the middle term is the time-weighted score error.
The important lesson is that $\mathfrak D_{\delta,T}(\pdata)$ replaces a global
Lipschitz constant for the score as the data-dependent factor in the
discretization complexity.

The remaining question is how large this covariance budget can be.  A useful
ambient-dimensional estimate comes directly from its posterior-variance form:
for any distribution with finite second moment,
\[
    \E\Tr\Sigma_t(X_t)
    =
    \E\norm{X_0-\E[X_0\mid X_t]}^2
    \leq
    \E\norm{X_0-X_t}^2
    =
    dt,
\]
because $X_t$ itself is an admissible estimator of $X_0$.  Substituting this
into \eqref{eq:D-delta-T} gives
\[
    \mathfrak D_{\delta,T}(\pdata)
    \leq
    d\left(1+\log\frac{T}{\delta}\right).
\]
Combining this estimate with \eqref{eq:chewi-ddpm-D} and the geometric-grid
relation \eqref{eq:geometric-ve-grid}, the exact-score discretization term is
controlled by
\[
    O\!\left(
        \frac{d\log^2(T/\delta)}{K}
    \right).
\]
Thus, when the initialization and score-estimation errors are also
$O(\eps^2)$, it is enough to take
\[
    K=\tO\!\left(
        \frac{d\log^2(T/\delta)}{\eps^2}
    \right)
\]
reverse steps.  This is the nearly $d$-linear dependence proved by Benton,
De Bortoli, Doucet, and Deligiannidis~\cite{benton2024nearlydlinear} under a
finite second-moment assumption. Related KL guarantees under finite
Fisher-information assumptions were obtained by Conforti, Durmus, and
Gentiloni Silveri~\cite{conforti2024klscore}.  The covariance budget can be bounded
more sharply when the data are intrinsically low-dimensional: entropy or
covering-number estimates for the smoothed law $p_\delta$ replace the ambient
dimension by an intrinsic dimension at scale $\sqrt\delta$.

\subsection{First-order rejection sampling and high accuracy}

The Hessian-control bound explains why Euler--Maruyama can be analyzed under
weak assumptions, but it also exposes a limitation of the sampler itself.
Euler--Maruyama does not try to sample the exact one-step Bayes kernel; it
linearizes the log density and relies on the step size being small enough that
the nonlinear remainder is negligible.  Thus, as for ULA before the
Metropolis correction, high accuracy is bought by a fine grid and hence by a
polynomial dependence on the target accuracy.  The natural MALA-like question
is whether this dependence can be reduced to logarithmic, or at least
polylogarithmic, dependence on $1/\eps$.

Here the analogy with MALA also reveals the main obstruction.  MALA obtains its
correction from a log-density ratio.  In a diffusion model, the one-step target
involves the noised density $p_k$, but $p_k(x)$ and $\log p_k(x)$ are not
available as numerical quantities; the learned object is only the score
$\sstar_k=\nabla\log p_k$.  Thus the question is sharper than simply applying a
Metropolis step: can a score-only correction emulate the missing
log-density-ratio test and correct the local Gaussian proposal toward the exact
Bayes kernel?

For the variance-exploding chain, this question has a useful local structure.
The exact backward kernel \eqref{eq:ve-backward-kernel} is not merely a
Gaussian with a shifted mean; it is a Gaussian tilt of the form
\begin{equation}\label{eq:gaussian-tilt}
    p_k(x)\exp\!\left(-\frac{\norm{x-x'}^2}{2h_k}\right),
\end{equation}
where $x'$ is the state at noise level $t_{k+1}$.  Euler--Maruyama corresponds
to replacing this tilt by the Gaussian obtained from a first-order expansion of
$\log p_k$.  The high-accuracy idea is to sample the tilted law
\eqref{eq:gaussian-tilt} more faithfully while still using only first-order
information.

This is where the algorithmic innovation enters.  First-order rejection
sampling, denoted $\AlgFORS$~\cite{chen2026highaccuracy}, replaces direct
log-density-ratio evaluation by randomized first-order estimates.  Although
$\log p_k(x)$ and $\log p_k(y)$ are unavailable separately, their difference
can be written as a path integral involving only the score: for any smooth path
$\gamma$ with $\gamma_0=y$ and $\gamma_1=x$,
\[
    \log p_k(x)-\log p_k(y)
    =
    \int_0^1
    \ip{\dot\gamma_r}{\sstar_k(\gamma_r)}\,\dd r.
\]
Sampling a random point along the path turns this integral into an unbiased
estimate of the log-density difference.  Thus, even though the endpoint log
densities themselves are unavailable, score queries can produce a randomized
estimate of exactly the log tilt that a rejection correction needs.

There is still one conversion to make.  Rejection sampling needs an acceptance
weight proportional to the exponential of the log tilt, not just an estimate of
the log tilt itself.  Exponentiating an unbiased estimate would not usually
give the right average, so $\AlgFORS$ uses a Poisson product identity to build
the exponential tilt from first-order random estimates.  In abstract form,
suppose a proposal is drawn from a law $Q$ and the desired law is the
exponential tilt of $Q$ with Radon--Nikodym derivative proportional to
$e^{w(x)}$.  Assume that, given the proposal $x$, one can sample bounded random
variables $W_1,W_2,\ldots$ with $\E[W_1\mid x]=w(x)$.  In the simplest bounded
form,
assume $|W_j|\leq R$ almost surely for a deterministic envelope $R$.  Take
$J\sim\Poi(2R)$ and accept a proposal $x$ with probability
\[
    \prod_{j=1}^{J}\frac{R+W_j}{2R}.
\]
The probability generating function of a Poisson random variable makes the
average acceptance factor proportional to $e^{w(x)}$.  Hence the accepted
proposals follow the exponentially tilted law.  If $R=\Theta(1)$, the number of
first-order queries is constant in expectation and logarithmic with high
probability.

\begin{proof}[Poisson-product calculation]
Condition on the proposed point $x$ and write
\[
    A(x)
    =
    \E\left[
    \prod_{j=1}^{J}\frac{R+W_j}{2R}
    \middle| x
    \right].
\]
Given $J$, the factors are independent and have conditional mean
$(R+w(x))/(2R)$.  Therefore
\[
    A(x)
    =
    \E\left[
    \left(\frac{R+w(x)}{2R}\right)^J
    \middle| x
    \right].
\]
The generating function of $J\sim\Poi(2R)$ gives
\[
    A(x)
    =
    \exp\!\left(2R\left(\frac{R+w(x)}{2R}-1\right)\right)
    =
    e^{w(x)-R}.
\]
The extra factor $e^{-R}$ is independent of $x$.  Hence the joint law of an
accepted proposal is proportional to $e^{w(x)}Q(\dd x)$, which is exactly the
exponential tilt.  The number of score queries is $J$, so its mean is $2R$ and a
standard Poisson tail bound gives logarithmic high-probability control when
$R$ is bounded.
\end{proof}

Thus $\AlgFORS$ is the score-only analogue of a Metropolis correction.  MALA
uses density ratios to correct a Gaussian proposal; $\AlgFORS$ uses randomized
first-order estimates to correct a Gaussian tilt without evaluating the
density.  Instead of accepting the Gaussian linearization as the sampler, the
algorithm corrects that local proposal toward the true Gaussian tilt.

In the notation of the error decomposition \eqref{eq:ddpm-error-decomposition}, $\AlgFORS$
changes the size of the exact-score algorithmic term
$\epsdisc^2$.  By sampling each local Gaussian tilt accurately
enough, the accumulated local algorithmic error can be made $O(\eps^2)$
using only polylogarithmically many reverse steps in the target accuracy
$\eps$.  In a simplified ambient-dimension form, the result of Chen,
Chewi, Daskalakis, and Rakhlin~\cite{chen2026highaccuracy} says that,
suppressing schedule constants and taking the initialization and local
algorithmic budgets of order $\eps^2$, there is a reverse sampler using
score evaluations for which
\begin{equation}\label{eq:high-acc-prototype}
    \KL(p_\delta\|\wh p_\delta)
    \lesssim
    \eps^2
    +
    \sum_{k=1}^{K-1}
    h_k\eps_{k+1,\mathsf{score}}^2
\end{equation}
with
\[
    K
    =
    \tO\!\left(d\,\polylog(1/\eps)\right)
\]
in a basic ambient-dimension reading.  The full theorem contains sharper
dimension measures and refinements under additional structure, but the key
message here is simpler: the expensive dependence on the accuracy parameter is
replaced by a polylogarithmic one.  Together with the variance-exploding
early-stopping estimate $W_2^2(\pdata,p_\delta)\leq \delta d$, this separates
the two error measurements: the initial smoothing is controlled in $W_2$, while
the implemented reverse chain is controlled in KL\@.
Note that the same
time-integrated score-error term appears in \eqref{eq:high-acc-prototype}:
$\AlgFORS$ improves the local sampling step, but it does not remove the need
for an accurate learned score.

\section{Discrete Diffusion Models}
\label{sec:discrete}

The development so far built and analyzed diffusion models on continuous state
spaces.  We now turn to a different modeling question: what if the state space
itself is not continuous?  Reverse-time Bayes rules, denoising objectives, and
path-space error decompositions survive on finite state spaces; Euclidean
gradients and Brownian motion do not.  Discrete diffusion models therefore force
us to state the sampling ideas in a form that does not rely on calculus.

\subsection{Forward and reverse kernels on a finite state space}

Images in pixel space can often be treated as continuous, but language,
symbolic sequences, molecules, graphs, and many scientific configurations are
inherently discrete.  If $x$ is a token sequence, an expression such as
$x+\sqrt t Z$ is meaningless.  There is no small Gaussian perturbation of the
word ``cat'' inside the vocabulary.  Instead of adding Brownian noise, a
discrete diffusion model applies a sequence of Markov kernels on a finite state
space.

The
probabilistic skeleton is exactly the one we have used throughout: a forward
chain that degrades the data toward a reference law, and a reverse chain
recovered by Bayes' rule.  We therefore keep the same symbols as in the
continuous case.  The forward transition $P_k^{\to}$ and the backward
transition $P_k^{\leftarrow}$ of Subsection~\ref{subsec:ddpm-grid} are now read
as \emph{stochastic matrices} on the finite set, rather than as transition
densities on $\R^d$.  The only genuine loss is the calculus: there
is no gradient $\nabla\log p_k$, so the score must be replaced by a
finite-difference object.

Let $\cX$ be a finite state space.  The data law $\pdata$ may be
supported on a subset $\cX_0\subseteq\cX$; we view it as a
law on $\cX$ by assigning zero mass outside $\cX_0$.  For ordinary
categorical sequences one can take $\cX_0=\cX=\mc{V}^L$, while
for the masked diffusions treated below the noisy state space $\cX$ also
contains mask symbols.  A
discrete forward diffusion is a Markov chain
\begin{equation}\label{eq:discrete-forward}
    X_0\sim \pdata,\qquad
    X_{k+1}\sim P_k^{\to}(X_k \to \cdot),
    \qquad k=0,\ldots,K-1,
\end{equation}
where each $P_k^{\to}$ is a stochastic matrix on $\cX$:
\[
    P_k^{\to}(x \to y)\geq0,\qquad
    \sum_{y\in\cX}P_k^{\to}(x \to y)=1.
\]
Let $p_k$ be the law of $X_k$.  In row-vector notation,
\[
    p_{k+1}=p_kP_k^{\to},\qquad
    p_k=\pdata P_0^{\to}P_1^{\to}\cdots P_{k-1}^{\to}.
\]
The forward kernels are chosen so that $p_K$ is close to a simple reference
distribution $\pi_{\mathsf{ref}}$, such as all-mask tokens, independent uniform
tokens, or a product distribution.

The design principle is the same as in continuous diffusion.  The
forward process should be simple enough that we can sample from it and know its
transition probabilities, and destructive enough that after many steps the
distribution is close to a reference law.  The reverse process is then the
learned part.  What changes is the nature of corruption: tokens are replaced,
masked, or resampled rather than moved by small Gaussian increments.

\subsection{Exact reverse kernels and ratio scores}

The exact reverse kernel is again obtained from Bayes' rule.  For fixed
$x'\in\cX$,
\begin{equation}\label{eq:discrete-reverse}
    P_k^{\leftarrow}(x' \to x)
    =
    \Pp(X_k=x\mid X_{k+1}=x')
    =
    \frac{p_k(x)P_k^{\to}(x \to x')}{p_{k+1}(x')}.
\end{equation}
This is the discrete analogue of \eqref{eq:backward-kernel}, in the same
$P_k^{\leftarrow}$ notation: there the continuous backward kernel equalled $p_k$
times a Gaussian likelihood, renormalized; here it equals $p_k$ times a
transition likelihood $P_k^{\to}(x \to x')$, renormalized.

If we knew $p_k$ exactly, \eqref{eq:discrete-reverse} would give an exact
reverse sampler.  In learning, we approximate either the reverse kernel
$P_k^{\leftarrow}(x' \to \cdot)$ directly or an object from which it can be computed,
such as the posterior distribution of $X_0$ given $X_k$.

The reverse-kernel identity \eqref{eq:discrete-reverse} shows exactly what must
be learned.  The forward transition $P_k^{\to}$ is known by design, but the
intermediate law $p_k$ contains information about the data distribution and is
unknown.  Learning a discrete diffusion model means learning enough about these
intermediate laws to simulate the reverse Bayes kernels.  One way to represent
this missing information is through ratios of the unknown law.  Indeed, for two
candidate predecessors $x$ and $y$ of the same current state $x'$, the
normalizing factor $p_{k+1}(x')$ in \eqref{eq:discrete-reverse} cancels:
\[
    \frac{P_k^{\leftarrow}(x' \to y)}{P_k^{\leftarrow}(x' \to x)}
    =
    \frac{p_k(y)P_k^{\to}(y \to x')}{p_k(x)P_k^{\to}(x \to x')}.
\]
Since $P_k^{\to}$ is known, the unknown part is the same-time density ratio
\begin{equation}\label{eq:discrete-ratio-score}
    \sstar_k(x,y)=\frac{p_k(y)}{p_k(x)},\qquad x,y\in\cX.
\end{equation}
In continuous space, the score tells us the infinitesimal log-density change
when we move from $x$ to $x+\dd x$.  On a graph or finite set, the ratio
$p_k(y)/p_k(x)$ tells us the finite log-density change when we move from
$x$ to a neighboring state $y$.  If the forward kernel permits only local
changes, these ratios are enough to specify local reverse transition
probabilities.

For many modern discrete diffusion models, however, the denoising posterior is
the more convenient target:
\[
    \dstar_k(x_0\mid x)=\Pp(X_0=x_0\mid X_k=x),
    \qquad x_0\in\cX_0,\ x\in\cX.
\]
This is the discrete counterpart of the continuous posterior mean
$\E[X_0\mid X_k=x]$ in Tweedie's identity.

The ratio-score and denoiser views are closely related by Bayes' rule:
\[
    \dstar_k(x_0\mid x)
    =
    \frac{\pdata(x_0)\Pp(X_k=x\mid X_0=x_0)}{p_k(x)}.
\]
Thus comparing the posterior probabilities at two noisy states recovers the
density ratio $p_k(y)/p_k(x)$ up to the known forward-likelihood ratio.  The
denoiser therefore contains the ratio-score information, but in a form that is
often easier to learn by supervised clean-token prediction.

\subsection{Masked diffusion}

The cleanest example is masked diffusion for token sequences.  Let the
vocabulary be $\mc{V}$, and add a special mask symbol $[\mathsf M]$.
Here $\cX_0=\mc{V}^L$, while the noisy state space is the augmented
sequence space
\[
    \cX_{\mathsf M}:=(\mc{V}\cup\{[\mathsf M]\})^L.
\]
For $x,y\in\cX_{\mathsf M}$, write $x_i$ and $y_i$ for their $i$th
coordinates.  Define the forward step by
\begin{equation}\label{eq:mask-forward}
    P_k^{\to}(x\to y)
    =
    \prod_{i=1}^L
    \begin{cases}
    1-\beta_k, & y_i=x_i,\ x_i\in\mc{V},\\
    \beta_k, & y_i=[\mathsf M],\ x_i\in\mc{V},\\
    1, & x_i=y_i=[\mathsf M],\\
    0, & \text{otherwise}.
    \end{cases}
\end{equation}
Thus an unmasked token either stays unchanged or becomes masked, and once a
token is masked it remains masked.  Variants may use coordinate-dependent
masking rates, but the same sequence-level form applies.

Let
\[
    \bar\alpha_k=\prod_{s=0}^{k-1}(1-\beta_s).
\]
For each coordinate $i$, conditional on $X_{0,i}=a\in\mc{V}$,
\[
    \Pp(X_{k,i}=a\mid X_{0,i}=a)=\bar\alpha_k,\qquad
    \Pp(X_{k,i}=[\mathsf M]\mid X_{0,i}=a)=1-\bar\alpha_k.
\]
This is analogous to the Gaussian formula
$X_k\mid X_0\sim\normal{a_kX_0}{\sigma_k^2\Id}$: the scalar attenuation
measures how much information about $X_0$ remains.

Masked diffusion is especially transparent because corruption is irreversible
in the forward direction.  Once a token is masked, the forward process has
forgotten its identity.  The reverse model must therefore use the surrounding
context and the learned data distribution to infer plausible clean tokens.
This makes the Bayesian nature of denoising visible without any calculus.  At
the coordinate level, the reverse step has only two cases.
\begin{itemize}
    \item If $X_{k+1,i}=a\in\mc{V}$ is unmasked, then necessarily
    $X_{k,i}=a$.  There is nothing to sample at that coordinate.
    \item If $X_{k+1,i}=[\mathsf M]$, then $X_{k,i}$ may either already be
    masked or may be the original clean token.  The reverse sampler must decide
    whether to unmask, and if so which token to place.
\end{itemize}
Specializing the reverse kernel \eqref{eq:discrete-reverse} to
\eqref{eq:mask-forward} gives the sequence-level formula
\begin{equation}\label{eq:mask-reverse-sequence}
    P_k^{\leftarrow}(x'\to x)
    =
    \frac{p_k(x)}{p_{k+1}(x')}
    P_k^{\to}(x\to x').
\end{equation}
This is still a coordinatewise masking rule through the known factor
$P_k^{\to}(x\to x')$, but the unknown data-dependent weight is the full
sequence law $p_k(x)$.  Thus the distribution of an unmasked token generally
depends on the entire partially masked sequence.

\subsection{Training objective for masked diffusion}

The model is usually trained to predict masked clean tokens from a corrupted
sequence.  Let $M_k\subseteq\{1,\ldots,L\}$ be the set of masked coordinates at
time $k$, and let $x_k$ be the masked sequence.  A denoising model
$\denoiser_k(i,\cdot\mid x_k)$ outputs a distribution over vocabulary
tokens for coordinate $i$.  The population loss is
\begin{equation}\label{eq:masked-loss}
    \mathsf L_{\mathsf{mask}}(\denoiser)
    =
    \E\left[
    \sum_{i\in M_k}
    -\log \denoiser_k(i,X_{0,i}\mid X_k)
    \right].
\end{equation}
This is cross-entropy with the clean token as the label.  The population
minimizer is the coordinate marginal of the posterior,
\[
    \dstar_k(i,a\mid x_k)
    =
    \Pp(X_{0,i}=a\mid X_k=x_k).
\]

Thus the cross-entropy objective is not an arbitrary language-modeling loss.
It is the discrete analogue of denoising score matching.  In the Gaussian case,
the optimal prediction is a posterior mean and can be converted into a score by
Tweedie's identity.  In the masked-token case, the optimal prediction is the
full posterior distribution of the hidden clean token.

\begin{proof}
For fixed $(i,x_k,k)$, the conditional contribution to the loss is
\[
    \sum_{a\in\mc{V}}
    \Pp(X_{0,i}=a\mid X_k=x_k)\,[-\log \denoiser_k(i,a\mid x_k)].
\]
This is the cross-entropy between the true posterior distribution and the
model distribution $\denoiser_k(i,\cdot\mid x_k)$.  It is minimized uniquely by
setting the model distribution equal to the true posterior.
\end{proof}

\subsection{Sampling from a masked model}
\label{subsec:masked-sampling}

A trained denoiser turns the posterior predictions above into a reverse-time
update rule.  In the absorbing case the terminal state is the all-mask
sequence.  Starting from that state, a sampler repeatedly chooses a set
$A\subseteq M_k$ of currently masked coordinates and fills those coordinates by
drawing
\[
    \wh x_i\sim \denoiser_k(i,\cdot\mid x_k),
    \qquad i\in A.
\]
Coordinates outside $A$ are left unchanged, unless the chosen algorithm
deliberately remasks or resamples them.  Thus the schedule is the rule that
chooses $A$ and the time label $k$ at each stage of the reverse process.

This viewpoint makes the any-order character of masked diffusion explicit.  The
model is not tied to a left-to-right factorization: it is trained to denoise
from partially masked contexts, so at inference time the sampler may reveal
coordinates in any order.  This order-agnostic view goes back to the
order-agnostic training of NADE~\cite{uria2014nade}, and was tied directly to
absorbing/masked diffusion by the autoregressive diffusion models of Hoogeboom
et al.~\cite{hoogeboom2022ardm}.  For an ordering or sequence of blocks
$A_1,A_2,\ldots,A_m\subseteq\{1,\ldots,L\}$, the sampler updates the
still-masked coordinates in $A_j$ using posterior predictions conditioned on
the current partial sequence.  Singleton blocks give an autoregressive-like
sampler with a chosen order; larger blocks give parallel decoding; adaptive
blocks chosen from model confidence give a coarse-to-fine or easy-to-hard
reveal schedule.

The reason this freedom is legitimate is that training never fixes an order in
the first place.  The objective \eqref{eq:masked-loss} masks a \emph{random}
subset of coordinates and asks the model to predict the clean tokens at the
masked positions from the unmasked ones.  Across training examples, noise
levels, and random masks, the masked set $M_k$ ranges over essentially all
subsets of $\{1,\ldots,L\}$, so the model is implicitly trained to approximate
the entire family of conditionals
\[
    \Pp\!\left(X_{0,i}=a \mid X_{k,M_k^c}=x_{M_k^c}\right),
    \qquad M_k\subseteq\{1,\ldots,L\},\ i\in M_k,
\]
of a clean coordinate $i$ given the visible entries of the partially masked
sequence.  Two features of masked diffusion make this exact.  First, the forward
process only masks tokens, never alters them, so
$X_{k,M_k^c}=X_{0,M_k^c}$: conditioning on the visible part of $X_k$ is the same
as conditioning on the corresponding clean values.  Second, whether a coordinate
is masked is decided independently of the token values, so the masking pattern
carries no information beyond which coordinates are observed.  Hence the
population optimum $\dstar_k(i,\cdot\mid x_k)$ of \eqref{eq:masked-loss} is
exactly the above conditional.

An inference order is therefore just a rule for choosing which conditional to
query next.  Revealing coordinate $i$ from the currently observed set $M_k^c$
uses
$\Pp(X_{0,i}\mid X_{k,M_k^c}=x_{M_k^c})$, a member of the family the model
already learned.  Every order factorizes the joint into
conditionals drawn from this one family, which is why any order is supported by
the same trained model.  The choice of block size, however, is still a sampling
choice.  If a block contains several coordinates and the sampler draws them
independently from the same context, it is using a factorized approximation to
the joint conditional for that block.  The next two subsections make this
distinction precise: first by writing the general KL accounting for approximate
reverse kernels, and then by isolating the extra error created by parallel block
updates.

\subsection{Error analysis parallel to the continuous case}

A schedule together with a denoiser defines approximate reverse kernels
$\wh P_k^{\leftarrow}(x' \to \cdot)$, while the exact kernels are
$P_k^{\leftarrow}(x' \to \cdot)$ from \eqref{eq:discrete-reverse}.  Let
$\wh p_0$ be the final distribution generated by the approximate reverse chain,
and let $\wh p_K$ be the law used to initialize that chain at time $K$.  A KL
path-space argument gives
\begin{equation}\label{eq:discrete-error}
    \KL(p_0\|\wh p_0)
    \leq
    \KL(p_K\|\wh p_K)
    +
    \sum_{k=0}^{K-1}
    \E_{X_{k+1}\sim p_{k+1}}
    \KL\!\left(
    P_k^{\leftarrow}(X_{k+1} \to \cdot)
    \middle\|
    \wh P_k^{\leftarrow}(X_{k+1} \to \cdot)
    \right).
\end{equation}
This is the exact analogue of the continuous diffusion decomposition:
\[
\begin{array}{c|c|c}
\text{continuous diffusion} & \text{discrete diffusion} & \text{meaning}\\
\hline
p_K\approx\normal{0}{\Id} & p_K\approx\pi_{\mathsf{ref}} &
\text{initialization at noise}\\
P_k^{\leftarrow}\propto p_k\times\text{Gaussian likelihood} &
P_k^{\leftarrow}\propto p_k\times P_k^{\to} &
\text{exact reverse kernel}\\
\score_k\approx\sstar_k & \wh P_k^{\leftarrow}\approx P_k^{\leftarrow}
\text{ or }\denoiser_k\approx\dstar_k &
\text{learned denoising information}\\
\sum_k\eta_k\epsscorek^2 &
\sum_k\E\,\KL(P_k^{\leftarrow}\|\wh P_k^{\leftarrow}) &
\text{statistical/model error}
\end{array}
\]

The justification is verbatim the KL telescoping argument of
Section~\ref{sec:error-analysis}: apply Lemma~\ref{lem:kl-chain} to the exact
and approximate reverse path laws, then project the paths to their final state
using data processing.

The decomposition \eqref{eq:discrete-error} isolates the same two error sources
as in the continuous case: the initialization gap $\KL(p_K\|\wh p_K)$ and the
per-step reverse-kernel KL summed over the reverse steps.  In absorbing masked
diffusion, finiteness of these KL terms is already informative.  If the sampler
starts from the deterministic all-mask state, then
$\wh p_K=\delta_{[\mathsf M]^L}$, so the initialization term is finite only when
the forward terminal law $p_K$ is also supported on the all-mask state.  For the
independent masking chain \eqref{eq:mask-forward}, this requires
$\bar\alpha_K=0$, for instance through a terminal step with $\beta_k=1$.

Once the supports are compatible so that the KL terms are finite, the second term in
\eqref{eq:discrete-error} is tied to the quantity the model actually trains on.
For masked diffusion the model does not learn the reverse kernel directly; it
learns posterior clean-token distributions through the cross-entropy objective
\eqref{eq:masked-loss}.  Schematically, we have the bound
\[
    \E\,\KL(P_k^{\leftarrow}(X_{k+1} \to \cdot)\|\wh P_k^{\leftarrow}(X_{k+1} \to \cdot))
    \lesssim
    \E\sum_{i\in M_{k+1}}
    \KL\!\left(
    \dstar_{k+1}(i,\cdot\mid X_{k+1})
    \middle\|
    \denoiser_{k+1}(i,\cdot\mid X_{k+1})
    \right).
\]
This bound is the cross-entropy regret of the posterior predictions, the part
that training controls: it vanishes as the learned posterior
$\denoiser_{k+1}$ approaches the true posterior $\dstar_{k+1}$.  It is the
\emph{entire} per-step error when the sampler unmasks one coordinate at a time,
because each single-coordinate reverse step queries a true conditional (as
explained in Subsection~\ref{subsec:masked-sampling}), so the approximate reverse
chain reproduces the exact one.  With a perfect denoiser, single-coordinate
decoding is therefore exact in any order.  Its drawback is cost---revealing a
single coordinate per step takes as many network evaluations as there are
coordinates---which raises the question of the next subsection: can we instead
reveal several coordinates at once, and what accuracy does that sacrifice?

\subsection{Parallel masked inference}
\label{subsec:parallel-inference}

To avoid spending one network evaluation per coordinate, practical masked
samplers often reveal a whole block of coordinates in a single reverse step.
This can mean two different things.  The simplest implementation is
\emph{factorized block decoding}: if $A$ is the set of coordinates to reveal from
the current partially observed sequence $x_k$, it samples
\[
    a_A\sim \prod_{i\in A}\denoiser_k(i,a_i\mid x_k).
\]
The exact object, however, is the joint conditional law
\[
    \Pp(X_{0,A}=a_A\mid X_k=x_k).
\]
Equivalently, for any ordering $A=\{i_1,\ldots,i_m\}$, the chain rule writes this
joint conditional as
\[
    \prod_{\ell=1}^{m}
    \Pp\!\left(
    X_{0,i_\ell}=a_{i_\ell}
    \mid
    X_k=x_k,\,
    X_{0,i_1}=a_{i_1},\ldots,X_{0,i_{\ell-1}}=a_{i_{\ell-1}}
    \right).
\]
Thus a block can be sampled exactly by sequentially querying updated
conditionals, but the product of one-coordinate posteriors from the same context
is exact only under conditional independence.  The gap between these two laws is
the factorization error.  The inference schedule therefore sets a
speed--accuracy tradeoff: larger blocks shorten the reverse chain but introduce
conditional-independence bias, while singleton blocks remove the bias at the cost
of as many reverse evaluations as there are coordinates.  Lavenant and
Zanella~\cite{lavenant2025masked} analyze this error directly for masked
diffusions with factorized approximations, decomposing the relative-entropy error
into learning and factorization terms and relating the optimal block-size
schedule to an information profile of the data distribution.

This is not the only possible approach of parallel inference.  The block
sampler above is a factorized approximation to the joint conditional; a
different question is whether one can parallelize the exact sequential
conditional-sampling procedure itself.  In an oracle model with access to exact
conditional marginals of a target law on $\mc{V}^L$, Anari, Gao, and
Rubinstein~\cite{anari2024parallelcounting} show how to organize the queries to
sample arbitrary product-space laws in $\tO(L^{2/3})$ parallel time.
Anari et al.~\cite{anari2025autospeculation} improve the parallel time to
$\tO(L^{1/2})$ using autospeculative rejection sampling.  These results are best
read as a corrected-parallel counterpart to the simple block update: they do not
replace a joint conditional by independent marginals, but instead use parallel
queries to simulate the sequential conditionals with a correction.

\subsection{CTMC setup and learning}
\label{subsec:ctmc-setup-learning}

The discussion so far used a discrete-time masked chain because it makes the
Bayes rule and the denoising posterior easy to see.  For analyzing practical
samplers, however, the sharper language is continuous time: it lets us describe
the exact reverse dynamics as one coordinate jump at a time, and then view
parallel updates as a numerical approximation to those dynamics.

This is one reason that most modern discrete diffusion models are phrased not with
the stochastic matrices $P_k^{\to}$ but as a \emph{continuous-time Markov chain}
(CTMC), the finite-state analogue of the forward SDE\@.  The forward corruption is
run as a CTMC whose rate matrix factorizes over coordinates, so each coordinate
is noised independently.  This has a decisive consequence for the reverse
process: at any instant only one coordinate changes, because the probability
that two independent coordinates jump in the same infinitesimal interval is
$o(\dd t)$.  The exact reverse dynamics are therefore unambiguous and update a single
coordinate at a time, and the parallel-update bias described above
reappears in a controlled form, as the error of approximating these dynamics on
a finite time grid (\emph{$\tau$-leaping}).  Continuous time thus separates the
posterior/score error from the discretization error cleanly, produces the
cleanest analogue of the score, and admits a flexible family of reverse
samplers~\cite{campbell2022continuous,lou2024sedd}.

A CTMC on a finite set $\cX$ is generated by a time-dependent forward
\emph{rate matrix} $R_t^{\to}$ with
\[
    R_t^{\to}(x,y)\geq0\ (y\neq x),
    \qquad
    R_t^{\to}(x,x)=-\sum_{y\neq x}R_t^{\to}(x,y).
\]
Thus the off-diagonal entries are jump rates, and the diagonal entry makes each
row sum to zero.  Over an infinitesimal step,
\[
    \Pp(X_{t+\dd t}=y\mid X_t=x)
    =
    \begin{cases}
    R_t^{\to}(x,y)\dd t+o(\dd t), & y\neq x,\\
    1+R_t^{\to}(x,x)\dd t+o(\dd t), & y=x.
    \end{cases}
\]
Equivalently,
\[
    \Pp(X_{t+\dd t}=y\mid X_t=x)
    =
    \delta_{xy}+R_t^{\to}(x,y)\dd t+o(\dd t).
\]
If $p_t(x)=\Pp(X_t=x)$ is written as a row vector, then its evolution is
\begin{equation}\label{eq:ctmc-forward-equation}
    \partial_t p_t=p_tR_t^{\to},
    \qquad
    \partial_t p_t(y)=\sum_{x\in\cX}p_t(x)R_t^{\to}(x,y).
\end{equation}
This is the finite-state counterpart of the Fokker--Planck equation
\eqref{eq:fokker-planck}: a rate matrix replaces the drift-and-diffusion
operator acting on continuous densities.

The rate matrix is also where the corruption family is specified.  An
\emph{absorbing} (masked) generator sends each token to $[\mathsf M]$ at a
schedule-dependent rate and is the continuous-time version of the masked
diffusion of the previous subsections; a \emph{uniform} generator pushes each
token toward the uniform distribution over $\mc{V}$; other choices encode
structured token graphs.  Masked diffusion is therefore a special case of this CTMC picture.

The reverse process is again a CTMC, and its rates are fixed by a Bayes-ratio
formula that is the time-continuous version of the reverse
kernel~\eqref{eq:discrete-reverse}.  For $y\neq x$,
\begin{equation}\label{eq:ctmc-reverse-rate}
    R_t^{\leftarrow}(y,x)
    =
    R_t^{\to}(x,y)\,\frac{p_t(x)}{p_t(y)}.
\end{equation}
With the diagonal chosen so that rows sum to zero, these rates are exactly the
ones that track the same marginal in reverse time.  Indeed, for each
$y$ with $p_t(y)>0$,
\[
\begin{aligned}
    (p_tR_t^{\leftarrow})(y)
    &=
    \sum_{x\neq y}p_t(x)R_t^{\leftarrow}(x,y)
    -
    p_t(y)\sum_{x\neq y}R_t^{\leftarrow}(y,x)  \\
    &=
    \sum_{x\neq y}p_t(x)R_t^{\to}(y,x)\frac{p_t(y)}{p_t(x)}
    -
    p_t(y)\sum_{x\neq y}R_t^{\to}(x,y)\frac{p_t(x)}{p_t(y)}  \\
    &=
    p_t(y)\sum_{x\neq y}R_t^{\to}(y,x)
    -
    \sum_{x\neq y}p_t(x)R_t^{\to}(x,y)  \\
    &=
    -\partial_t p_t(y),
\end{aligned}
\]
where the last equality is the forward equation
\eqref{eq:ctmc-forward-equation}.  Thus
$-\partial_t p_t=p_tR_t^{\leftarrow}$, which is the forward equation for the
chain when it is simulated from time $T$ down to time $0$.  Starting the reverse
CTMC from $p_T$ therefore gives marginal $p_t$ at every intermediate time.
As in the continuous case, the forward rates $R_t^{\to}$ are known by design and
the only unknown ingredient is the collection of same-time density ratios over
neighboring states.  We write the exact score as
\[
    \sstar_t(x,y):=\frac{p_t(y)}{p_t(x)},
    \qquad y\neq x,\ R_t^{\to}(y,x)>0.
\]
This is the continuous-time form of the ratio
score~\eqref{eq:discrete-ratio-score}, the discrete analogue of
$\nabla\log p_t$ that records the finite relative change of the noised law along
each admissible move.  Let $\score_t(x,y)>0$ be the learned prediction,
$\score_t(x,y)\approx\sstar_t(x,y)$.
The orientation is chosen for reverse sampling: when the reverse chain is at
$x$, a candidate predecessor $y$ is admissible exactly when the forward CTMC can
jump from $y$ to $x$.  With this convention the learned reverse rate is
\[
    \wh R_t^{\leftarrow}(x,y)
    =
    R_t^{\to}(y,x)\,\score_t(x,y),
    \qquad y\neq x.
\]
At first the training target looks inaccessible, since the ratio contains the
unknown marginal law $p_t$.  The key denoising identity is that this unknown
ratio can be written as a posterior average of known forward likelihood ratios.
\begin{equation}\label{eq:discrete-ratio-denoising-identity}
    \sstar_t(x,y)
    =
    \E\!\left[
    \frac{\Pp(X_t=y\mid X_0)}{\Pp(X_t=x\mid X_0)}
    \,\middle|\, X_t=x
    \right].
\end{equation}
Thus training can sample clean data $X_0$, corrupt it to $X_t=x$, and use the
known forward likelihood ratio inside the loss.

Score-entropy training~\cite{lou2024sedd} turns this identity into a supervised
loss for positive ratios by first choosing a scalar discrepancy for one reverse
edge.  Let $\sstar>0$ denote the exact ratio, let $\score>0$ be the learned
prediction, and set $F(u)=-\log u$ on $u>0$.  Score entropy uses the scaled
Bregman divergence
\[
    \sstar D_F(\score,\sstar)
    =
    \sstar\bigl(F(\score)-F(\sstar)-F'(\sstar)(\score-\sstar)\bigr)
    =
    \score-\sstar\log\score+\sstar\log\sstar-\sstar.
\]
This quantity is nonnegative, has derivative $1-\sstar/\score$, and is minimized
at $\score=\sstar$; note that it is defined only for positive ratios, so the
learned $\score_t(x,y)$ must be constrained to remain positive.  In the learning objective the terms
$\sstar\log\sstar-\sstar$ are independent of the learned score.  Thus the
model-dependent part of one edge is
$\score_t(x,y)-\sstar_t(x,y)\log \score_t(x,y)$.  In training, the unknown
ratio $\sstar_t(x,y)$ is replaced by the computable forward likelihood ratio
from the noising process.
A schematic denoising score-entropy loss is
\begin{equation}\label{eq:score-entropy-loss}
    \E_{t,X_0}
    \sum_{x\in\cX}\Pp(X_t=x\mid X_0)
    \sum_{y:\,R_t^{\to}(y,x)>0}
    w_t(x,y)
    \left[
    \score_t(x,y)
    -
    \frac{\Pp(X_t=y\mid X_0)}{\Pp(X_t=x\mid X_0)}
    \log \score_t(x,y)
    \right],
\end{equation}
up to terms independent of the learned score.  Here $w_t(x,y)\geq0$ is a chosen
weight, often the incoming forward rate $R_t^{\to}(y,x)$ or a multiple of it.

The reason this objective has the correct population target is exactly the
denoising identity \eqref{eq:discrete-ratio-denoising-identity} above.
Since the displayed loss is affine in the likelihood ratio appearing in \eqref{eq:discrete-ratio-denoising-identity}, averaging
over the unknown clean data simply replaces it by its conditional mean.  The
resulting edgewise risk is, up to constants,
$
    \score_t(x,y)
    -
    \sstar_t(x,y)\log \score_t(x,y),
$
and is minimized at the exact score
$\score_t(x,y)=\sstar_t(x,y)$ on each reverse-admissible edge.

The special choice $w_t(x,y)=R_t^{\to}(y,x)$ gives this denoising loss its
sampling interpretation.  On the edge from the current state $x$ back to a possible
predecessor $y$, the true reverse rate is
$R_t^{\to}(y,x)\sstar_t(x,y)$, while the learned reverse rate is
$R_t^{\to}(y,x)\score_t(x,y)$.  The corresponding per-edge contribution to the
CTMC path-space relative entropy is (as will be shown in \eqref{eq:ctmc-score-entropy-error})
\[
    R_t^{\to}(y,x)\left[
    \score_t(x,y)
    -
    \sstar_t(x,y)
    +
    \sstar_t(x,y)
    \log
    \frac{\sstar_t(x,y)}{\score_t(x,y)}
    \right].
\]
After dropping terms independent of the learned score, this is exactly the
objective in \eqref{eq:score-entropy-loss} with $w_t(x,y)=R_t^{\to}(y,x)$.  Thus
the rate-weighted choice matches the path-space error quantity used in the CTMC
sampling analysis below.

\subsection{CTMC sampling}
\label{subsec:ctmc-sampling}

Once the learned ratios have specified the reverse rates, the remaining problem
is numerical simulation of the resulting jump process.
$\tau$-leaping is a standard acceleration of the stochastic simulation
algorithm for chemical reaction networks, introduced by
Gillespie~\cite{gillespie2001tau}.  The idea is to group many small jumps over a
short interval, while pretending that the jump rates are essentially constant
during that interval.

For the learned reverse CTMC, keep the same forward-time grid convention as
before,
\[
    0=t_0<t_1<\cdots<t_K=T,
    \qquad h_k=t_{k+1}-t_k.
\]
The reverse sampler moves from $t_{k+1}$ down to $t_k$.  If the current state is
$x$ at time $t_{k+1}$, the simplest tau-leap freezes the learned reverse rates
at $(t_{k+1},x)$ and uses the first-order kernel
\[
    \Pp(\wh X_{t_k}=z\mid \wh X_{t_{k+1}}=x)
    =
    h_k\,\wh R_{t_{k+1}}^{\leftarrow}(x,z)
    +O(h_k^2),
    \qquad z\neq x,
\]
with the remaining probability assigned to staying at $x$.  Equivalently, for
product spaces one draws independent Poisson clocks
$N_z\sim\Poi(h_k\wh R_{t_{k+1}}^{\leftarrow}(x,z))$ for the admissible local
moves and applies the proposed coordinate changes in parallel, with a fixed
tie-breaking convention if several incompatible clocks ring.  This is the
discrete analogue of an Euler--Maruyama step: it replaces the time-varying
reverse generator by a frozen one and, in parallel implementations, introduces
a local independence error by allowing several coordinates to update during one
step.  We refer to \cite[Chapter~13]{e2019appliedstochastic} for a more detailed
discussion of $\tau$-leaping and related algorithms.

Tau-leaping is convenient, but it is still an approximation because the rates
are frozen over a whole interval.  A closely related exact construction is
\emph{uniformization}, originally introduced by
Grassmann~\cite{grassmann1977transient}, applied to chemical reaction
simulation by Beentjes and Baker~\cite{beentjes2019uniformisation}, and applied
to discrete diffusion models by Chen and Ying~\cite{chenying2024discrete}.

Let us describe how uniformization works.  On the interval
$[t_k,t_{k+1}]$, choose a clock rate
\[
    \Lambda_k
    \geq
    \sup_{\substack{s\in[t_k,t_{k+1}]\\ x\in\cX}}
    \sum_{z\neq x}\wh R_s^{\leftarrow}(x,z).
\]
Starting from the state at time $t_{k+1}$, draw Poisson event times in
$[t_k,t_{k+1}]$ with rate $\Lambda_k$ and process them in decreasing time.  At
an event time $s$, if the current state is $x$, jump to $z\neq x$ with
probability $\wh R_s^{\leftarrow}(x,z)/\Lambda_k$ and otherwise stay at $x$.
The stay-put events are virtual jumps.  Since the clock dominates all outgoing
rates, this thinning construction has exactly the jump law of the learned
reverse CTMC on the interval, without a frozen-rate approximation.

\subsection{Error analysis for CTMC samplers}
\label{subsec:ctmc-error-analysis}

Having described the two CTMC simulation schemes, we now ask how their output
laws differ from the desired data law.
The numerical analysis can be organized in the same way as in
Section~\ref{sec:error-analysis}.  There are three contributions.  First, the
exact reverse process should start from $p_T$, whereas the sampler is initialized
from some implementable law $\wh p_T$.  This gives the common initialization
gap.  Define
\[
    \epsinit^2:=\KL(p_T\|\wh p_T).
\]
Two common choices of $\wh p_T$ should be interpreted differently.  If
$\wh p_T$ is a point mass, then $\epsinit^2$ is mainly a support condition.  In
absorbing masked diffusion, for instance, one often takes
$\wh p_T=\delta_{[\mathsf M]^L}$; the KL is finite only if the forward terminal
law $p_T$ is also supported on the all-mask state.  With a finite integrated
masking rate this may fail, so one should either force exact terminal
absorption or measure the endpoint error in a weaker metric.

If $\wh p_T$ is full-support, the same term is an ordinary initialization
error.  For example, if $\wh p_T$ is uniform and the forward CTMC contracts KL
to $\wh p_T$ at rate $\rho$, then
\[
    \epsinit^2
    \leq e^{-\rho T}\KL(p_0\|\wh p_T)
    \leq e^{-\rho T}\log|\cX|.
\]
On product token spaces $\cX=\mc{V}^L$, this term is of order
$L(\log|\mc{V}|)e^{-T}$ for the usual independent-coordinate corruption.

After initialization, fix a simulator and first imagine running that simulator
with the true reverse rates, and only afterwards replace the true ratio score by
the learned ratio score:
\[
    p_0
    \quad\longrightarrow\quad
    p_0^{\mathsf{num},*}
    \quad\longrightarrow\quad
    \wh p_0^{\mathsf{num}}.
\]
Here $p_0^{\mathsf{num},*}$ denotes the output law of the chosen numerical
simulator when it is run with the true reverse rates.  The first comparison
isolates numerical error, while the second comparison isolates score-estimation
error.  Strictly speaking, KL has no triangle inequality, so the rigorous proof
usually performs this split on path space, or directly in the log-intensity
integrand, rather than by adding two marginal KL divergences.  Nevertheless,
this bookkeeping is useful for organizing the error analysis.

The score-estimation term has a path-space form analogous to Girsanov's theorem
for diffusions.  We write it in the same epsilon-squared convention as the
continuous score error:
\begin{align}\label{eq:ctmc-score-entropy-error}
    \epsSE^2
    :=
    \int_0^T
    \E_{X_t\sim p_t}
    \sum_{y:\,R_t^{\to}(y,X_t)>0}
    R_t^{\to}(y,X_t)
    \left[
    \score_t(X_t,y)
    -\sstar_t(X_t,y)
    +\sstar_t(X_t,y)
    \log
    \frac{\sstar_t(X_t,y)}{\score_t(X_t,y)}
    \right]\dd t.
\end{align}
This is the same score-entropy Bregman loss that appears in the learning
objective \eqref{eq:score-entropy-loss}, which explains why it is a natural
training loss.

\begin{proof}[Proof of \eqref{eq:ctmc-score-entropy-error}]
Compare the exact reverse CTMC with rates
$R_t^{\leftarrow}(x,y)=R_t^{\to}(y,x)\sstar_t(x,y)$ to the learned reverse CTMC
with rates
$\wh R_t^{\leftarrow}(x,y)=R_t^{\to}(y,x)\score_t(x,y)$, and assume the learned
rates are positive wherever the exact rates are positive.  Over a small interval
of length $\dd t$, condition on the current state $x$.  Using the diagonal for
the stay-put probability, the one-step KL is
\[
\begin{aligned}
    &\sum_{y\neq x}R_t^{\leftarrow}(x,y)\dd t
    \log
    \frac{R_t^{\leftarrow}(x,y)\dd t}
         {\wh R_t^{\leftarrow}(x,y)\dd t}
    +
    \bigl(1+R_t^{\leftarrow}(x,x)\dd t\bigr)
    \log
    \frac{1+R_t^{\leftarrow}(x,x)\dd t}
         {1+\wh R_t^{\leftarrow}(x,x)\dd t}
    +o(\dd t)  \\
    &\qquad=
    \dd t
    \sum_{y\neq x}
    \left[
    \wh R_t^{\leftarrow}(x,y)-R_t^{\leftarrow}(x,y)
    +R_t^{\leftarrow}(x,y)
    \log
    \frac{R_t^{\leftarrow}(x,y)}
         {\wh R_t^{\leftarrow}(x,y)}
    \right]
    +o(\dd t).
\end{aligned}
\]
Substituting the two reverse rates cancels the common factor
$R_t^{\to}(y,x)$ inside the logarithm and gives
\[
    \sum_{y:\,R_t^{\to}(y,x)>0}
    R_t^{\to}(y,x)
    \left[
    \score_t(x,y)-\sstar_t(x,y)
    +\sstar_t(x,y)\log\frac{\sstar_t(x,y)}{\score_t(x,y)}
    \right]\dd t
    +o(\dd t).
\]
The exact reverse process has marginal $p_t$ at noise time $t$.  Averaging over
$X_t\sim p_t$ and integrating over $t\in[0,T]$ gives
\eqref{eq:ctmc-score-entropy-error}; if the two reverse chains start from
different laws, the additional contribution is the initialization KL error.
\end{proof}

For $\tau$-leaping, the exact-score numerical law is not the true reverse CTMC
law.  On each interval $[t_k,t_{k+1}]$ it replaces the time-dependent true
reverse rates by the frozen rates $R_{t_{k+1}}^{\leftarrow}$.  Replacing these
frozen true rates by $\wh R_{t_{k+1}}^{\leftarrow}$ then adds the
score-estimation part.  Let $\bar h$ denote the mesh parameter controlling the
reverse step sizes, and let $\bar R$ be an upper bound on the total outgoing
reverse rate,
$\sum_{z\neq x}R_s^{\leftarrow}(x,z)\leq\bar R$ on $[0,T]$.  A
stochastic-integral analysis of the combined comparison gives, under
bounded-rate and regularity assumptions, a bound of the schematic form
\begin{equation}\label{eq:tau-leaping-kl-general}
    \KL\!\left(p_0\middle\|\wh p_0^\tau\right)
    \lesssim
    \epsinit^2
    +\epsSE^2
    +\bar R^{\,2}\bar h T.
\end{equation}
The three terms are the initialization error, the score-estimation error, and
the exact-score tau-leaping discretization error.

For product token spaces $\cX=\mc{V}^L$, more recent work gives a sharper
version of the same message with explicit dependence on the vocabulary size.
If the learned ratios are clipped to $[B^{-1},B]$ and the reverse grid is
controlled by mesh parameter $\bar h$, then the standard tau-leaping sampler
obeys, suppressing logarithmic factors and endpoint regularity constants,
\begin{equation}\label{eq:tau-leaping-kl-product}
    \KL\!\left(p_0\middle\|\wh p_0^\tau\right)
    \lesssim
    \epsinit^2
    +\epsSE^2
    +\bar h\, L^2|\mc{V}|\,T,
\end{equation}
as in the analysis of Liang et al.~\cite{liang2025sampler}.  For the usual
independent-coordinate corruption one may substitute
$\epsinit^2\lesssim L(\log|\mc{V}|)e^{-T}$.  Thus making the tau-leaping
discretization error of order $\eps$ requires
$\bar h$ of order $\tO(\eps/(L^2|\mc{V}|T))$, and hence a deterministic grid
with roughly $\tO(L^2|\mc{V}|/\eps)$ reverse steps in this analysis.

We can achieve high-accuracy sampling using uniformization.  With the true
ratio score, it simulates the true reverse CTMC exactly, so
$p_0^{\mathsf{uni},*}=p_0$ if the reverse process is initialized from
$p_T$.  Thus, for the learned score, the error estimate of uniformization reduces to
\begin{equation}\label{eq:ctmc-exact-kl}
    \KL\!\left(p_0\middle\|\wh p_0^{\rm uni}\right)
    \leq
    \epsinit^2
    +\epsSE^2.
\end{equation}
On product token spaces $\cX=\mc{V}^L$ with the usual independent-coordinate
corruption, one may substitute
$\epsinit^2\lesssim L(\log|\mc{V}|)e^{-T}$ in \eqref{eq:ctmc-exact-kl}, giving
$\KL(p_0\|\wh p_0^{\rm uni})\lesssim
L(\log|\mc{V}|)e^{-T}+\epsSE^2$.
Choosing $T\asymp\log(L\log|\mc{V}|/\eps^2)$ and learning the ratio score so
that $\epsSE^2\lesssim\eps^2$ gives
$\KL(p_0\|\wh p_0^{\rm uni})\lesssim\eps^2$.  Chen and
Ying~\cite{chenying2024discrete} further show that, with adaptive dominating
rates, the expected number of Poisson events is nearly linear in the hypercube
dimension.  Their theorem is stated for $\{0,1\}^d$; for an alphabet of size
$|\mc{V}|$, a fixed binary encoding uses
$q=\lceil\log_2|\mc{V}|\rceil$ bits per token, so $d=Lq$.  In token notation
this gives an expected event count
$\tO\!\left(L\log|\mc{V}|\right)$
for $\eps^2$ KL error.

\section{Guidance, Reward Tilting, and Inference-Time RL}
\label{sec:guidance}

So far we have built diffusion samplers, continuous and discrete, that aim to
reproduce the data distribution, or a slightly smoothed version of it.  In many
applications one instead wants a controlled modification of that base sampler:
generate samples satisfying a condition, prefer high-reward outputs, or adapt a
model at inference time.  Guidance, reward tilting, and inference-time RL are all
mechanisms for biasing the base model toward preferred outputs while staying
close to the pretrained distribution, and their common mathematical language is
the KL-regularized change of measure developed in this section.

\subsection{Reward-tilted targets and KL-regularized optimization}

Let $p_0$ denote the clean-output distribution produced by the pretrained
sampler; we keep the diffusion convention that clean samples are written as
$x_0$.  At this point we specify only the desired law on clean outputs; the
corresponding reverse-time dynamics will be derived below.  Given a reward
$r:\R^d\to\R$ and inverse temperature $\beta\geq0$, define the exponentially
tilted target
\begin{equation}\label{eq:reward-tilt}
    p_0^\beta(x_0)
    =
    \frac{1}{Z_\beta}p_0(x_0)e^{\beta r(x_0)}.
\end{equation}
Here
\[
    Z_\beta=\int p_0(x_0)e^{\beta r(x_0)}\dd x_0
    =
    \E_{X_0\sim p_0}e^{\beta r(X_0)}
\]
is the normalizing constant, assumed finite.  Equivalently, for any test
function $f$,
\[
    \E_{p_0^\beta}f(X_0)
    =
    \frac{\E_{p_0}\!\left[f(X_0)e^{\beta r(X_0)}\right]}
    {\E_{p_0} e^{\beta r(X_0)}}.
\]
Thus the target is obtained from the base model by reweighting samples
according to their clean-output reward.  The case $\beta=0$ recovers $p_0$, while
larger $\beta$ places more mass on high-reward regions and therefore trades
diversity under the base model for reward improvement.

Conditional generation fits the same form.  If $y$ is an observation, class
label, or prompt and $p(y\mid x_0)$ is the corresponding likelihood or
compatibility model, then Bayes' rule gives
\[
    p_0(x_0\mid y)\propto p_0(x_0)p(y\mid x_0).
\]
This is an exponential tilt with reward function
$x_0\mapsto\log p(y\mid x_0)$ and $\beta=1$; a classifier score or learned
preference model plays the same role when an explicit likelihood is unavailable.
The inference-time problem is to modify the reverse sampler so that its
clean-output law approximates such a tilted or conditional target while still
reusing the pretrained base dynamics.

The same tilted law has a useful variational meaning.  Among all distributions
$q$ over final samples, it is the optimizer of
\begin{equation}\label{eq:gibbs-variational}
    \sup_q
    \left\{
    \beta\E_q r-\KL(q\|p_0)
    \right\}.
\end{equation}
Equivalently, if $\beta>0$,
\[
    p_0^\beta
    =
    \argmax_q
    \left\{
    \E_q r-\frac1\beta\KL(q\|p_0)
    \right\}.
\]
Thus $\beta$ controls the tradeoff between
reward improvement and staying close to the base model.  Large $\beta$ pushes
hard toward high reward and risks mode collapse or reward hacking; small
$\beta$ preserves the base distribution but gives weaker alignment.

\begin{proof}
Let $p_0^\beta(x)=Z_\beta^{-1}p_0(x)e^{\beta r(x)}$.  For any $q$,
\[
    \KL(q\|p_0^\beta)
    =
    \int q(x)\log\frac{q(x)}{p_0(x)e^{\beta r(x)}/Z_\beta}\dd x
    =
    \KL(q\|p_0)-\beta\E_q r+\log Z_\beta.
\]
Rearranging,
\[
    \beta\E_q r-\KL(q\|p_0)
    =
    \log Z_\beta-\KL(q\|p_0^\beta)
    \leq \log Z_\beta,
\]
with equality iff $q=p_0^\beta$.
\end{proof}

\subsection{Guidance as score tilting}
\label{subsec:guidance-score}

To run a diffusion sampler for the tilted target $p_0^\beta$ we need the score of
its noised marginals, so the question is how that score differs from the base
score we have already learned.  Let $p_t$ and $p_t^\beta$ be the noised marginals
of $p_0$ and $p_0^\beta$.  Write $P_{0,t}^{\to}(x_0\to x)$ for the density of
the clean-to-noisy marginal forward kernel $\Law(X_t\mid X_0=x_0)$.  Because the reward acts only on the clean
sample $x_0$, noising the tilted law gives
\[
    p_t^\beta(x)
    =
    \frac{1}{Z_\beta}\int P_{0,t}^{\to}(x_0\to x)\,
    e^{\beta r(x_0)}p_0(x_0)\dd x_0
    =
    \frac{p_t(x)}{Z_\beta}\,h_t^\beta(x),
    \qquad
    h_t^\beta(x):=\E\!\left[e^{\beta r(X_0)}\mid X_t=x\right].
\]
The noised tilted law is thus the base noised law reweighted by the
\emph{posterior tilt factor} $h_t^\beta$, which reads off, from the current noisy
state, the expected exponential reward of the clean sample it will denoise to.
Because this marginal is a product, its log-gradient splits into the base score
plus a correction,
\[
    \nabla\log p_t^\beta(x)
    =
    \nabla\log p_t(x)+\nabla\log h_t^\beta(x),
\]
so exact guidance merely adds the tilt gradient $\nabla\log h_t^\beta$ to the
base score, and practical methods differ only in how they approximate it.

For conditional generation, the reward function is
$x_0\mapsto\log p(y\mid x_0)$.  The same posterior-tilt notation gives the noisy
likelihood
\[
    h_t(x\mid y)
    :=
    \E[p(y\mid X_0)\mid X_t=x].
\]
This is the quantity a noisy classifier estimates.  It is not the noised
conditional density itself; rather, it tilts the base noisy density.  If
the condition is fixed, the noised conditional density is
\[
    p_t(x\mid y)\propto p_t(x)h_t(x\mid y),
\]
where the missing normalizing constant is independent of $x$.  Taking a
log-gradient at noise level $t$ gives
\begin{equation}\label{eq:classifier-guidance}
    \nabla\log p_t(x\mid y)
    =
    \nabla\log p_t(x)
    +
    \nabla\log h_t(x\mid y).
\end{equation}
This is classifier guidance~\cite{dhariwal2021guidance}: train or use a
classifier on noisy inputs, then add its gradient to the unconditional score.
Classifier-free guidance~\cite{ho2022cfg} estimates the same increment without
a separate classifier.  We use the already established score notation in the
form
\[
    \score_t(x)\approx\nabla\log p_t(x),
    \qquad
    \score_t(x\mid y)\approx\nabla\log p_t(x\mid y).
\]
Thus \eqref{eq:classifier-guidance} suggests
\[
    \score_t(x\mid y)-\score_t(x)
    \approx
    \nabla\log h_t(x\mid y).
\]
With guidance strength $\beta\geq0$, classifier-free guidance replaces the
unconditional score by
\begin{equation}\label{eq:cfg}
    \score_t(x)
    \quad\longmapsto\quad
    \score_t(x)
    +
    \beta\bigl(\score_t(x\mid y)-\score_t(x)\bigr).
\end{equation}
The choice $\beta=1$ gives the conditional score estimate $\score_t(x\mid y)$, while
$\beta>1$ extrapolates the conditional-score increment; empirically this often
improves condition satisfaction, at the cost of moving farther from the base
distribution.
If the two learned scores were exact and $\beta=1$, the approximation \eqref{eq:cfg}
would recover the conditional score in \eqref{eq:classifier-guidance}.  For
$\beta\neq1$, it instead gives the score of the noise-level power tilt
$p_t(x)h_t(x\mid y)^\beta$, whereas the clean reward tilt
\eqref{eq:reward-tilt} would involve the posterior factor
$h_t^\beta(x)=\E[p(y\mid X_0)^\beta\mid X_t=x]$.  Powering the noisy likelihood is
not the same as tilting by the noisy expectation of the powered clean
likelihood, so the two coincide only at $\beta=1$.

For small tilts, the posterior tilt factor has a linear approximation.  Since
$h_t^\beta(x)=1+\beta V_t(x)+O(\beta^2)$ with
$V_t(x)=\E[r(X_0)\mid X_t=x]$, we have
\begin{equation}\label{eq:value-guidance}
    \nabla\log p_t^\beta(x)
    \approx
    \nabla\log p_t(x)
    +
    \beta\nabla V_t(x),
    \qquad
    V_t(x)=\E[r(X_0)\mid X_t=x].
\end{equation}
In the small-tilt regime, then, guidance perturbs the score by the gradient of
the posterior expected reward $V_t$, so a reward model need only capture this
first-order landscape.  The simplification is that $V_t$ averages the reward $r$
itself, whereas the exact factor $h_t^\beta$ averages $e^{\beta r}$; the two
agree only to first order in $\beta$.  Exact guidance uses the full $h_t^\beta$,
and the small-tilt form trades it for a cheaper reward-gradient correction.

\subsection{Reward tilting as a Polchinski flow}
\label{subsec:polchinski-tilt}

Stepping back from the practical methods, the tilt factor $h_t^\beta$ ties
guidance to the renormalization viewpoint of
Section~\ref{sec:stochastic-localization-polchinski}.  Recall the
variance-exploding normalization $X_t=X_0+\sqrt t\,Z$ used there, in which
$p_t=p_0*\normal{0}{t\Id}$ and the effective potential $U_t=-\log p_t$ solves the
finite-dimensional Polchinski equation~\eqref{eq:polchinski-finite}.
Reading the integral that defines $h_t^\beta$ the other way, the
\emph{unnormalized} tilted marginal
$h_t^\beta\,p_t=(e^{\beta r}p_0)*\normal{0}{t\Id}$ is the same forward
channel applied to the reward-reweighted data measure $e^{\beta r}p_0$---noising
the tilted data and tilting the noised data by $h_t^\beta$ coincide---so it is a
heat flow and solves~\eqref{eq:heat-flow-density} just as $p_t$ does.  Dividing
out $p_t$ leaves a forward Kolmogorov equation for the tilt factor,
\begin{equation}\label{eq:tilt-pde}
    \partial_t h_t^\beta
    =
    \tfrac12\Delta h_t^\beta
    +
    \sstar_t\cdot\nabla h_t^\beta,
\end{equation}
so $h_t^\beta$ is transported by the score-driven generator
$\tfrac12\Delta+\sstar_t\cdot\nabla$.  Equivalently, the tilted effective
potential $U_t^\beta:=-\log(h_t^\beta\,p_t)$ obeys the \emph{same} Polchinski equation,
\begin{equation}\label{eq:polchinski-tilt}
    \partial_t U_t^\beta
    =
    \tfrac12\Delta U_t^\beta
    -
    \tfrac12\norm{\nabla U_t^\beta}^2,
\end{equation}
only with the bare potential shifted from $U_0=-\log p_0$ to $U_0^{\beta} = -\log p_0-\beta r$.
Reward tilting therefore runs the renormalization flow from a
reward-shifted initial potential, and the guided score
$-\nabla U_t^\beta=\sstar_t+\nabla\log h_t^\beta$ is exactly the additive
guidance correction of the previous subsection.

It should be emphasized that this is a reformulation, not a recipe.  Solving the
tilt equation~\eqref{eq:tilt-pde} or the Polchinski
equation~\eqref{eq:polchinski-tilt} is no easier than evaluating the posterior
expectation $h_t^\beta(x)=\E[e^{\beta r(X_0)}\mid X_t=x]$ that defines it: the PDE
carries the same computational difficulty as the conditional expectation.

\subsection{Path-space control and the Doob transform}

The score-tilting, classifier-guidance, and Polchinski-flow descriptions of the
previous subsections are all exact, but they all rest on the same intractable
object, the posterior tilt $h_t^\beta$.  To see how it is approximated in
practice it helps to pass from marginals to whole reverse trajectories: on path
space the tilt becomes a stochastic control problem, and its exact
solution---the Doob transform of the base reverse chain---exposes precisely which
conditional expectations must be estimated.

Let us write a continuous reverse sampler in generation time as
\[
    \dd Y^{\leftarrow}_s=b_s(Y^{\leftarrow}_s)\dd s+\sigma_s\dd B^{\leftarrow}_s,
    \qquad 0\leq s\leq T,
\]
and let $\Pp^0$ be its path law on full reverse trajectories
$(Y^{\leftarrow}_s)_{0\leq s\leq T}$, with clean output $Y^{\leftarrow}_T$.

Reward-tilted sampling trades terminal reward against the cost of departing from
this base law.  At the level of path laws it is the KL-regularized optimization
\begin{equation}\label{eq:path-gibbs}
    \sup_{\mathbb Q}
    \left\{
    \beta\E_{\mathbb Q}r(Y^{\leftarrow}_T)-\KL(\mathbb Q\|\Pp^0)
    \right\}
    =
    \log\E_{\Pp^0}e^{\beta r(Y^{\leftarrow}_T)},
\end{equation}
the path-space version of the Gibbs variational principle
\eqref{eq:gibbs-variational}.  Its optimizer is the Gibbs path law
\begin{equation}\label{eq:path-gibbs-optimizer}
    \frac{\dd\mathbb Q^\star}{\dd\Pp^0}
    \!\left((Y^{\leftarrow}_s)_{0\leq s\leq T}\right)
    =
    \frac{e^{\beta r(Y^{\leftarrow}_T)}}
    {\E_{\Pp^0}e^{\beta r(Y^{\leftarrow}_T)}}.
\end{equation}

\begin{proof}
The proof is identical to the finite-dimensional reward-tilting proof.  Define
$\dd\mathbb Q^\star\propto e^{\beta r(Y^{\leftarrow}_T)}\dd\Pp^0$.  Then
\[
    \KL(\mathbb Q\|\mathbb Q^\star)
    =
    \KL(\mathbb Q\|\Pp^0)-\beta \E_{\mathbb Q}r(Y^{\leftarrow}_T)
    +\log\E_{\Pp^0}e^{\beta r(Y^{\leftarrow}_T)}.
\]
Rearranging and using nonnegativity of KL proves \eqref{eq:path-gibbs}, with the
supremum attained at $\mathbb Q^\star$.
\end{proof}

To turn this into a sampler we express it as a control problem on the drift.  A
controlled sampler changes only the drift,
\[
    \dd Y^{\leftarrow}_s=\bigl(b_s(Y^{\leftarrow}_s)+\sigma_s u_s(Y^{\leftarrow}_s)\bigr)\dd s+\sigma_s\dd B^{\leftarrow}_s,
\]
with path law $\Pp^u$; using the same diffusion coefficient means the two path
laws differ only through the drift.  Under the usual absolute-continuity and
integrability assumptions, the Girsanov KL formula in
Appendix~\ref{app:ito-girsanov} gives
\begin{equation}\label{eq:girsanov-control-cost}
    \KL(\Pp^u\|\Pp^0)
    =
    \frac12
    \E_{\Pp^u}\int_0^T\norm{u_s(Y^{\leftarrow}_s)}^2\dd s,
\end{equation}
with $u_s(Y^{\leftarrow}_s)$ a progressively measurable control, so the quadratic
control energy is exactly the KL cost of steering the sampler away from its base
law.  Restricting $\mathbb Q$ to controlled diffusions $\Pp^u$ and substituting
this cost into \eqref{eq:path-gibbs} makes the Gibbs problem the equivalent
stochastic control problem over drifts
\begin{equation}\label{eq:control-objective}
    \sup_u
    \E_{\Pp^u}\left[
    \beta r(Y^{\leftarrow}_T)
    -
    \frac12\int_0^T\norm{u_s(Y^{\leftarrow}_s)}^2\dd s
    \right].
\end{equation}
The two are the same optimization stated at the level of measures and of drifts.

Reading the optimizer \eqref{eq:path-gibbs-optimizer} back at the level of drifts
recovers the slogan ``guidance is control'': the optimal control's drift
increment is exactly the value-gradient term of the ideal guided reverse dynamics
of Subsection~\ref{subsec:guidance-score}, now in this generation-time
parametrization.  The KL cost \eqref{eq:girsanov-control-cost}, which measured
discretization error in Section~\ref{sec:error-analysis}, is here a deliberately
chosen guidance budget, spent where it raises expected terminal reward most
efficiently.

From here to the end of the section we adopt the discrete-time point of view,
working with the time-discretized reverse chain $(X_K,\ldots,X_0)$ of
Section~\ref{sec:discretizing} rather than the continuous SDE\@.  The reason is
that the objects we now build---the Doob transform below, the sequential Monte
Carlo weights of Subsection~\ref{subsec:feynman-kac}, and the policies of
Subsection~\ref{subsec:inference-rl}---all act step by step, and a finite chain
of reverse kernels states them without stochastic-calculus bookkeeping; the
continuous formulas above are recovered in the small-step limit.  In this
notation the clean output is indexed by $0$, so a terminal reward is written as
$r(X_0)$.  The path-space optimizer has an explicit Markov-kernel form: each
base reverse proposal is reweighted by the tilt factor of its continuations and
normalized by the tilt factor at the current state.  This kernel reweighting is
the Doob transform of the base reverse chain.

\begin{theorem}[Optimal path tilt and Doob transform]
\label{thm:path-doob-transform}
Let $\Pp^0$ be a base reverse Markov chain on
$(X_K,X_{K-1},\ldots,X_0)$ with initial law $p_K^0$ and reverse kernels
$P_k^{0,\leftarrow}(x_{k+1}\to \dd x_k)$.  For a terminal reward
$r(X_0)$, define
\[
    h_k(x_k)
    =
    \E_{\Pp^0}\!\left[e^{\beta r(X_0)}\mid X_k=x_k\right].
\]
Then the optimizer of the path-space Gibbs problem
\eqref{eq:path-gibbs} is the tilted path law
\[
    \frac{\dd\mathbb Q^\star}{\dd\Pp^0}
    =
    \frac{e^{\beta r(X_0)}}{\E_{\Pp^0}e^{\beta r(X_0)}}.
\]
Moreover, $\mathbb Q^\star$ is Markov, and its reverse kernels are
\begin{equation}\label{eq:guided-reverse-kernel}
    P_k^{\star,\leftarrow}(x_{k+1}\to \dd x_k)
    =
    P_k^{0,\leftarrow}(x_{k+1}\to \dd x_k)
    \frac{h_k(x_k)}{h_{k+1}(x_{k+1})},
    \qquad
    h_{k+1}(x_{k+1})
    =
    \int h_k(y)P_k^{0,\leftarrow}(x_{k+1}\to \dd y).
\end{equation}
\end{theorem}

\begin{proof}
The variational optimality of the tilted path law follows from nonnegativity of
KL\@.  It remains to compute its kernels.  Condition on the current reverse-time
state $X_{k+1}=x_{k+1}$.  Under the tilted law, the
conditional distribution of the next state $X_k$ is proportional to the base
conditional distribution multiplied by the expected future weight:
\[
    P_k^{\star,\leftarrow}(x_{k+1}\to \dd x_k)
    \propto
    P_k^{0,\leftarrow}(x_{k+1}\to \dd x_k)
    \E_{\Pp^0}\!\left[e^{\beta r(X_0)}\mid X_k=x_k\right].
\]
By definition, the expectation is $h_k(x_k)$.  The normalizing constant is
\[
    \int h_k(y)P_k^{0,\leftarrow}(x_{k+1}\to \dd y)
    =
    \E_{\Pp^0}\!\left[e^{\beta r(X_0)}\mid X_{k+1}=x_{k+1}\right]
    =
    h_{k+1}(x_{k+1}),
\]
where the middle equality uses the Markov property.  This gives the displayed
Doob-transform kernel and shows that the tilted path law is again Markov.
\end{proof}

The value function $h_k$ is intractable, and the two ways of coping with this
organize the rest of the section.  One route uses an approximate value only as a
proposal and removes its bias by reweighting, giving an \emph{exact} sampler in
the limit (Subsection~\ref{subsec:feynman-kac}); the other approximates the value
and accepts the resulting bias, the inference-time-RL view
(Subsection~\ref{subsec:inference-rl}).

\subsection{Feynman--Kac correction and sequential Monte Carlo}
\label{subsec:feynman-kac}

Theorem~\ref{thm:path-doob-transform} gives the exact guided sampler, but its
kernels require the value functions
$h_k(x_k)=\E_{\Pp^0}[e^{\beta r(X_0)}\mid X_k=x_k]$, which are exactly as
unavailable as the posterior reward-to-go in the score picture.  This subsection
takes the first route of the fork above: use an approximate value only as a
\emph{proposal} and remove its bias by reweighting.  That is the
Feynman--Kac / sequential Monte Carlo route, and it is the mechanism behind
inference-time scaling of guided samplers.

The name comes from the value function.  Along the base reverse process, $h$ is
a conditional expectation of a terminal functional, hence a martingale,
\[
    h_{k+1}(x_{k+1})=\int h_k(y)\,
    P_k^{0,\leftarrow}(x_{k+1}\to \dd y),
\]
the same identity that appeared in Theorem~\ref{thm:path-doob-transform}.  In
continuous time it solves the backward equation
$\partial_s h_s+\mathcal L_s h_s=0$ with terminal data $h=e^{\beta r}$, where
$\mathcal L_s$ is the generator of the base reverse diffusion; this is the
reverse-time companion of the noising-time equation~\eqref{eq:tilt-pde} for the
same tilt factor in Subsection~\ref{subsec:polchinski-tilt}.  At the clean end
$k=0$ the state is $X_0$ itself, so the terminal value is exactly
$h_0(x_0)=e^{\beta r(x_0)}$.

The key fact is that the entire telescoping of the Doob transform collapses to a
single terminal weight.

\begin{prop}[Untwisted reweighting]
\label{prop:untwisted-weight}
Let $\Pp^0$ be the base reverse chain of
Theorem~\ref{thm:path-doob-transform} and $\mathbb Q^\star$ the tilted path law
with $\dd\mathbb Q^\star/\dd\Pp^0=e^{\beta r(X_0)}/Z_\beta$.  Then on
trajectories,
\[
    \frac{\dd\mathbb Q^\star}{\dd\Pp^0}(x_K,\ldots,x_0)
    =
    \frac{e^{\beta r(x_0)}}{Z_\beta},
    \qquad
    Z_\beta=\E_{\Pp^0}e^{\beta r(X_0)}.
\]
\end{prop}

\begin{proof}
Write the tilted path law from its initial law and Doob kernels.  Its initial
law is $p_K^\star(\dd x_K)=p_K^0(\dd x_K)\,h_K(x_K)/Z_\beta$, and its kernels
are
$P_k^{\star,\leftarrow}
=P_k^{0,\leftarrow}\,h_k(x_k)/h_{k+1}(x_{k+1})$.  Multiplying,
\[
    \frac{\dd\mathbb Q^\star}{\dd\Pp^0}
    =
    \frac{h_K(x_K)}{Z_\beta}
    \prod_{k=0}^{K-1}\frac{h_k(x_k)}{h_{k+1}(x_{k+1})}
    =
    \frac{h_K(x_K)}{Z_\beta}\cdot\frac{h_0(x_0)}{h_K(x_K)}
    =
    \frac{h_0(x_0)}{Z_\beta},
\]
and $h_0(x_0)=e^{\beta r(x_0)}$.
\end{proof}

This suggests the simplest exact algorithm: draw $N$ trajectories from the base
model, weight each by $e^{\beta r(x_0)}$, and resample.  As $N\to\infty$ the
weighted empirical law converges to $\mathbb Q^\star$, so the clean marginal
converges to the tilted target $p_0^\beta$.  The catch is variance: if the base
model rarely produces high-reward outputs, almost all weight lands on a few
trajectories and the effective sample size collapses.

Sequential Monte Carlo fixes this by \emph{twisting} the proposal with the
approximate value and resampling along the way.  Replace the base kernel by a
proposal that follows an approximate Doob transform,
\begin{equation}\label{eq:twisted-proposal}
    \wh P_k^{\leftarrow}(x_{k+1}\to \dd x_k)
    =
    \frac{P_k^{0,\leftarrow}(x_{k+1}\to \dd x_k)\,\wh h_k(x_k)}
    {\wh g_{k+1}(x_{k+1})},
    \qquad
    \wh g_{k+1}(x_{k+1})=\int \wh h_k(y)\,
    P_k^{0,\leftarrow}(x_{k+1}\to \dd y),
\end{equation}
where $\wh h_k\approx h_k$ is any tractable value estimate (a classifier, a
reward model, a learned value, or the linearized small-tilt value
of~\eqref{eq:value-guidance}) with the exact terminal value
$\wh h_0=e^{\beta r}$.  If the particles are initialized from the base noisy
law $p_K^0$, they also carry the initial weight $\wh h_K(X_K)$; if one can
instead sample from the twisted noisy law proportional to $p_K^0\wh h_K$, this
initial weight is constant.  Running $N$ particles through $\wh P_k^{\leftarrow}$ and
carrying an incremental importance weight at each reverse step, then resampling
when the effective sample size drops, gives a consistent estimator of
$\mathbb Q^\star$.

At each reverse step the particle picks up an incremental importance weight: the
ratio of the base transition weighted by the current value $\wh h_k(x_k)$ to the
twisted proposal weighted by the next value $\wh h_{k+1}(x_{k+1})$.  Substituting
the proposal~\eqref{eq:twisted-proposal}, this collapses to
\[
    \frac{
    P_k^{0,\leftarrow}(x_{k+1}\to \dd x_k)\,\wh h_k(x_k)}
    {
    \wh P_k^{\leftarrow}(x_{k+1}\to \dd x_k)\,\wh h_{k+1}(x_{k+1})}
    =
    \frac{\wh g_{k+1}(x_{k+1})}{\wh h_{k+1}(x_{k+1})}.
\]
Thus the incremental weight is
\begin{equation}\label{eq:smc-incremental-weight}
    w_k(x_{k+1})
    =
    \frac{\wh g_{k+1}(x_{k+1})}{\wh h_{k+1}(x_{k+1})}.
\end{equation}
It measures the one-step mismatch between the value $\wh h_{k+1}$ assumed at the
current state and the value $\wh g_{k+1}$ obtained by propagating $\wh h_k$ one
base step.  If the twist is exact, $\wh h=h$, then $\wh g_{k+1}=h_{k+1}$ by the
martingale identity and all incremental weights are one; the only nonconstant
weight left is the initial twist weight if the sampler starts from $p_K^0$.
All of the approximation error is thus pushed into nonuniform weights, which
resampling removes in the large-$N$ limit; a good $\wh h$ keeps the weights near
uniform and the variance low.

For conditional generation, take the terminal reward to be
$\log p(y\mid x_0)$.  Then
$h_k(x_k)=\Pp^0(y\mid X_k=x_k)$, and a noisy classifier or conditional
likelihood model gives a tractable $\wh h_k$.  When this twist is realized as a
guidance gradient $\nabla\log\wh h_k$ that only moves the particles, the proposal
is exactly a classifier-guided sampler, and the weights above are the SMC
correction that restores the conditional target in the large-particle limit.

\begin{remark}
This is the ``Feynman--Kac steering'' or ``twisted diffusion sampler'' family of
methods~\cite{wu2023twisted}.  It separates two roles cleanly: guidance
supplies an approximate value $\wh h$ that biases the proposal toward
high-reward regions, while the weights~\eqref{eq:smc-incremental-weight} and
resampling guarantee that the limiting law is the \emph{exact} tilt rather than
a biased approximation.  Spending more inference-time compute---more particles,
more frequent resampling, or a more accurate twist---reduces the variance of the
correction rather than changing its target.
\end{remark}

\subsection{Inference-time reinforcement learning}
\label{subsec:inference-rl}

The previous subsections describe the ideal sampler for a reward tilt: if the
posterior value functions $h_k$ were known exactly, the reverse kernels would be
the Doob transforms in Theorem~\ref{thm:path-doob-transform}.  The setting here
is genuinely inference-time: the pretrained reverse sampler stays fixed as the
reference law, and a policy changes only how that sampler is run at test time.
The RL problem is therefore over sampler choices during generation, not over
retraining the base score model.

In a finite-horizon formulation, the reverse sampler takes actions during
sampling.  Let $\pi=(\pi_k)$ be a possibly randomized policy: after observing
$X_{k+1}$ at step $k$, it chooses an action $a_k\sim\pi_k(\cdot\mid X_{k+1})$,
and the reverse step uses the corresponding kernel,
\[
    a_k\sim\pi_k(\cdot\mid X_{k+1}),
    \qquad
    X_k\sim P_k^{a_k,\leftarrow}(X_{k+1}\to\cdot),
    \qquad
    a_k\in\cA.
\]
The action may be a guidance scale, a timestep choice, an injected-noise level,
a rejection or resampling decision, or an additive drift correction.  In a
masked language diffusion sampler, it may choose how many tokens to unmask,
which positions to update, or how sharply to sample from a token posterior.
The terminal reward is measured only after the final object $X_0$ is produced.

The objective is the same KL-regularized change of measure, restricted to the
family of sampler modifications available at inference time:
\begin{equation}\label{eq:rl-kl-objective}
    J(\pi)
    =
    \E_{\pi}
    \left[
    r(X_0)
    -
    \frac1\beta\sum_k
    \KL\!\left(
    P_k^{a_k,\leftarrow}(X_{k+1}\to\cdot)
    \middle\|
    P_k^{0,\leftarrow}(X_{k+1}\to\cdot)
    \right)
    \right].
\end{equation}
Here $\beta>0$ is the same inverse temperature as in the variational
principle~\eqref{eq:gibbs-variational}; larger $\beta$ means weaker KL
regularization and stronger reward seeking.  The KL term keeps the
policy-induced sampler close to the pretrained sampler at each reverse step.
This mirrors the variational identity: reward improvement is meaningful only
relative to a reference distribution.  Without this reference cost, policy
optimization can exploit flaws in the reward model or collapse to a narrow set
of high-scoring samples.

If the action class is rich enough to choose the whole next-step kernel, this
optimization does not produce a new sampler: it recovers the Doob-transformed
kernel of Theorem~\ref{thm:path-doob-transform}, now reached by dynamic
programming rather than by tilting the path law.  To see the same object appear,
define the regularized value-to-go from $x_k$ by
\[
    v_k(x_k)
    =
    \sup_{\pi_0,\ldots,\pi_{k-1}}
    \E_{\pi}
    \left[
    r(X_0)
    -
    \frac1\beta\sum_{j=0}^{k-1}
    \KL\!\left(
    P_j^{a_j,\leftarrow}(X_{j+1}\to\cdot)
    \middle\|
    P_j^{0,\leftarrow}(X_{j+1}\to\cdot)
    \right)
    \,\middle|\, X_k=x_k
    \right],
\]
with the empty-sum convention $v_0(x_0)=r(x_0)$.  The local Gibbs variational
principle gives
\[
    P_k^{\star,\leftarrow}(x_{k+1}\to \dd x_k)
    \propto
    P_k^{0,\leftarrow}(x_{k+1}\to \dd x_k)
    \exp\left(\beta v_k(x_k)\right).
\]
Writing $h_k(x_k)=\exp(\beta v_k(x_k))$ makes this exactly the normalization
in \eqref{eq:guided-reverse-kernel}: the same reweighting we already have, only
now the value is presented as a Bellman value-to-go rather than a
posterior tilt factor.  The two letters are the two coordinates of one object:
the multiplicative tilt factor $h_k$ that enters the Doob kernel and the SMC
weights, and its log-domain \emph{soft value} $v_k=\beta^{-1}\log h_k$, so that
$\beta$ times its gradient (in the continuous picture) is added to the score,
matching $\nabla\log h_t^\beta=\beta\nabla v_t$.  The
posterior expected reward $V_t$ of \eqref{eq:value-guidance} is the small-$\beta$
linearization of this soft value.  Soft dynamic programming is therefore the finite-state
computational form of the path-space tilt; what is genuinely new here is not the
optimal kernel but how its value is obtained, since guidance methods approximate
this value information while RL methods estimate it from sampled rewards.

For a parameterized policy $\pi_\theta$ with reverse-path density
$p_\theta(x_K,\ldots,x_0)$, sampled rewards can be turned into updates by the
likelihood-ratio identity
\[
    \nabla_\theta\E_{(X_K,\ldots,X_0)\sim p_\theta}[r(X_0)]
    =
    \E_{(X_K,\ldots,X_0)\sim p_\theta}
    \left[
    r(X_0)\nabla_\theta\log p_\theta(X_K,\ldots,X_0)
    \right],
\]
In implementations one often subtracts a baseline from the reward to reduce
variance, but that is not part of the change-of-measure identity.  Along a
reverse diffusion chain,
\[
    \log p_\theta(x_K,\ldots,x_0)
    =
    \log p_{\theta,K}(x_K)
    +
    \sum_k \log p_\theta(x_k\mid x_{k+1}),
\]
with the initial-law term disappearing from the gradient when the noisy
initialization is fixed.  Thus the update decomposes across the guidance
decisions made along the trajectory.
For the KL-regularized objective \eqref{eq:rl-kl-objective}, the
return is replaced by the regularized return inside the brackets, with any
explicit derivative of the KL penalty added when the transition family is
differentiable.

This perspective is also how to read current inference-time RL methods.
One can combine the base sampler with a reward-aligned proposal, spend particles
on resampling and trajectory correction, or estimate a drift/Doob correction
more directly.  These are different approximations to the same ideal controlled
path law.  The useful message for these notes is not that one particular recent
algorithm has replaced the others; it is that reward tilting, KL-regularized
control, and Doob transforms provide the common language for comparing ways of
changing the sampler at inference time.

\subsection{Guidance and reward tilting in the discrete case}

Everything in this section carries over almost unchanged to the discrete
diffusion models of Section~\ref{sec:discrete}, because the target and objective
never used continuous structure.  The
reward-tilted target \eqref{eq:reward-tilt} and the KL-regularized variational
principle \eqref{eq:gibbs-variational} are defined on any space; the
Feynman--Kac/twisted-SMC correction of Subsection~\ref{subsec:feynman-kac}
reweights particles by the same potentials; and inference-time RL
\eqref{eq:rl-kl-objective} differentiates the trajectory
likelihood---for a discrete chain a sum
$\sum_k\log P_{k,\theta}^{\leftarrow}(x_{k+1}\to x_k)$ of log transition
probabilities---with respect to $\theta$ rather than the state, so it transfers
verbatim.

The one step that changes form is guidance.  The optimal tilt is again the Doob
$h$-transform of Theorem~\ref{thm:path-doob-transform}, now of a jump process:
with the soft value $v_t(x)=\beta^{-1}\log\E[\exp(\beta r(X_0))\mid X_t=x]$, the
guided reverse rates are reweighted by value differences,
\[
    R_t^{\leftarrow,\mathrm{guided}}(y\to x)
    =R_t^{\leftarrow}(y\to x)\,\exp\!\big(\beta(v_t(x)-v_t(y))\big),
\]
the exact analogue of adding $\beta\nabla v_t$ to the reverse drift, with the gradient
replaced by a finite difference of $v_t$ across admissible moves---the same
score-to-ratio substitution as in \eqref{eq:discrete-ratio-score}.  As before
$v_t$ is intractable and is approximated, by a learned predictor of the reward,
a predictor-free interpolation, or a Taylor expansion around the predicted clean
token~\cite{nisonoff2024discreteguidance}.

\appendix

\section{It\^o Calculus and Girsanov Theorem}\label{app:ito-girsanov}

This appendix collects the stochastic-calculus facts used in the notes.  The
statements below are written in the smooth, non-explosive setting in which the
formal calculations in the text are valid.  For rigorous hypotheses and
proofs, see the texts of E, Li and Vanden-Eijnden~\cite{e2019appliedstochastic}, Karatzas and Shreve~\cite{karatzas1991brownian}
or {\O}ksendal~\cite{oksendal2003stochastic}.

\subsection*{It\^o formula, generator, and adjoint}

Let $X_t$ solve
\[
    \dd X_t=b_t(X_t)\dd t+\sigma_t(X_t)\dd B_t,
\]
where $B_t$ is an $m$-dimensional Brownian motion and
$\sigma_t(x)\in\R^{d\times m}$.  Write
$a_t(x)=\sigma_t(x)\sigma_t(x)^\top$.  For a smooth test function $\varphi$,
It\^o's formula gives
\[
    \dd\varphi(X_t)
    =
    \ip{\nabla\varphi(X_t)}{b_t(X_t)}\dd t
    +\frac12\Tr\!\left(a_t(X_t)\nabla^2\varphi(X_t)\right)\dd t
    +\ip{\nabla\varphi(X_t)}{\sigma_t(X_t)\dd B_t}.
\]
The last term is a martingale increment and has mean zero under the usual
integrability assumptions.  Thus
\[
    \frac{\dd}{\dd t}\E[\varphi(X_t)]
    =
    \E[(\mathcal L_t\varphi)(X_t)],
\]
where the infinitesimal generator is
\[
    \mathcal L_t\varphi
    =
    \ip{b_t}{\nabla\varphi}
    +\frac12\Tr(a_t\nabla^2\varphi).
\]
If $X_t$ has density $q_t$, then the Fokker--Planck equation is the adjoint
equation
\[
    \partial_tq_t=\mathcal L_t^\ast q_t,
\]
where, in coordinates,
\[
    \mathcal L_t^\ast q
    =
    -\nabla\cdot(b_tq)
    +\frac12\sum_{i,j=1}^d
    \partial_i\partial_j\!\left((a_t)_{ij}q\right).
\]
For overdamped Langevin, $b=-\nabla U$ and $a=2\Id$, so this reduces to
\[
    \partial_tq_t=\nabla\cdot(q_t\nabla U)+\Delta q_t.
\]

\subsection*{Girsanov's theorem}

Girsanov's theorem compares two diffusions that share the same initial law and
the same diffusion coefficient but may carry different drifts.  We fix the
coefficient $\sqrt2\,\Id$, the Langevin convention used in the notes.  Let $\Pp$
and $\mathbb{Q}$ be the laws on path space $C([0,T];\R^d)$ of
\[
    \dd X_t=b^\Pp_t(X_t)\dd t+\sqrt2\,\dd B_t
    \qquad\text{and}\qquad
    \dd X_t=b^{\mathbb{Q}}_t(X_t)\dd t+\sqrt2\,\dd B_t.
\]
Under the usual absolute-continuity and integrability assumptions, the two laws
are equivalent on $\cF_T$, the terminal $\sigma$-algebra of the Brownian filtration, with Radon--Nikodym density
\begin{equation}\label{eq:girsanov-density}
    \frac{\dd \Pp}{\dd \mathbb{Q}}\bigg|_{\cF_T}
    =
    \exp\!\left(
    \frac12\int_0^T \ip{b^\Pp_t-b^{\mathbb{Q}}_t}{\dd X_t-b^{\mathbb{Q}}_t\,\dd t}
    -
    \frac14\int_0^T \norm{b^\Pp_t-b^{\mathbb{Q}}_t}^2\,\dd t
    \right),
\end{equation}
where all integrands are evaluated at $X_t$.

To obtain the KL divergence, take logarithms in \eqref{eq:girsanov-density} and
average under $\Pp$:
\[
    \KL(\Pp\|\mathbb{Q})
    =
    \E_\Pp\!\left[\log\frac{\dd \Pp}{\dd \mathbb{Q}}\right]
    =
    \frac12\E_\Pp\int_0^T \ip{b^\Pp_t-b^{\mathbb{Q}}_t}{\dd X_t-b^{\mathbb{Q}}_t\,\dd t}
    -
    \frac14\E_\Pp\int_0^T\norm{b^\Pp_t-b^{\mathbb{Q}}_t}^2\dd t.
\]
The second term is already a plain expectation of a path integral, so it is left
as is.  For the first term, recall that under $\Pp$, we have $\dd X_t-b^{\mathbb{Q}}_t\,\dd t=(b^\Pp_t-b^{\mathbb{Q}}_t)\dd t+\sqrt2\,\dd B_t$, and
\[
    \frac12\E_\Pp\int_0^T \ip{b^\Pp_t-b^{\mathbb{Q}}_t}{\dd X_t-b^{\mathbb{Q}}_t\,\dd t}
    =
    \frac12\E_\Pp\int_0^T\norm{b^\Pp_t-b^{\mathbb{Q}}_t}^2\dd t
    +
    \frac{1}{\sqrt2}\,\E_\Pp\int_0^T\ip{b^\Pp_t-b^{\mathbb{Q}}_t}{\dd B_t}.
\]
The expectation of the final integral vanishes, and we arrive at
\begin{equation}\label{eq:girsanov-kl}
    \KL(\Pp\|\mathbb{Q}) = \frac14 \E_{\Pp}\int_0^T \norm{b^\Pp_t(X_t)-b^{\mathbb{Q}}_t(X_t)}^2\dd t.
\end{equation}
So far the two processes share the same initial law.  If instead they start from
different initial laws, the density \eqref{eq:girsanov-density} acquires the
extra factor $\dd\Law_\Pp(X_0)/\dd\Law_{\mathbb{Q}}(X_0)$ at $t=0$, and the KL gains the
corresponding initial term:
\[
    \KL(\Pp\|\mathbb{Q})
    =
    \KL\bigl(\Law_\Pp(X_0)\,\big\|\,\Law_{\mathbb{Q}}(X_0)\bigr)
    +
    \frac14\E_\Pp\int_0^T\norm{b^\Pp_t(X_t)-b^{\mathbb{Q}}_t(X_t)}^2\dd t.
\]

\begin{lemma}[Data processing]\label{lem:data-processing}
Let $\mu$ and $\nu$ be probability laws on a measurable space, and let
$T$ be a measurable map into another measurable space.  Then
\[
    \KL(T_\#\mu\|T_\#\nu)\leq \KL(\mu\|\nu)
    \qquad\text{and}\qquad
    \TV(T_\#\mu,T_\#\nu)\leq \TV(\mu,\nu).
\]
\end{lemma}

\begin{proof}
If $\mu$ is not absolutely continuous with respect to $\nu$, the KL bound is
trivial.  Otherwise let $Z=\dd\mu/\dd\nu$.  Under $T_\#\nu$, the
Radon--Nikodym derivative of $T_\#\mu$ with respect to $T_\#\nu$ is
\[
    \E_\nu[Z\mid T].
\]
Therefore Jensen's inequality gives
\[
    \KL(T_\#\mu\|T_\#\nu)
    =
    \E_\nu\!\left[\E_\nu[Z\mid T]\log \E_\nu[Z\mid T]\right]
    \leq
    \E_\nu[Z\log Z]
    =
    \KL(\mu\|\nu).
\]
For total variation, take the supremum over events $B$ in the target space:
\[
    \abs{T_\#\mu(B)-T_\#\nu(B)}
    =
    \abs{\mu(T^{-1}B)-\nu(T^{-1}B)}
    \leq
    \TV(\mu,\nu). \qedhere
\]
\end{proof}

Thus, for any measurable map $F$ on path space,
\[
    \KL(F_\#\Pp^u\|F_\#\Pp^0)\leq \KL(\Pp^u\|\Pp^0).
\]
Taking $F(\omega)=\omega_T$ gives the endpoint KL bound used in the
one-step ULA calculation.

\section{Gaussian Toolbox}\label{app:gaussian-toolbox}

This short toolbox collects facts used throughout the notes.  These are good
warm-up exercises for students who have not recently worked with Gaussian
densities.

\begin{lemma}[Affine Gaussian maps and sums]\label{lem:gaussian-affine-sum}
Let $X\sim\normal{m}{\Sigma}$ in $\R^d$.  For a matrix $A$ and vector $b$,
\[
    AX+b\sim\normal{Am+b}{A\Sigma A^\top}.
\]
If $X$ and $Y$ are independent Gaussians, then $X+Y$ is Gaussian and
\[
    \Cov(X+Y)=\Cov(X)+\Cov(Y).
\]
In particular, in one dimension, if $X$ and $Y$ are independent and centered,
then
\[
    \operatorname{Var}(aX+bY)
    =
    a^2\operatorname{Var}(X)+b^2\operatorname{Var}(Y).
\]
\end{lemma}

\begin{proof}
The affine statement follows from the Gaussian characteristic function:
\[
    \E e^{i\ip{t}{AX+b}}
    =
    \exp\!\left(i\ip{t}{Am+b}-\frac12 t^\top A\Sigma A^\top t\right).
\]
For independent Gaussians, characteristic functions multiply, so the means and
covariances add.  The one-dimensional variance formula is the corresponding
special case.
\end{proof}

\begin{lemma}[KL between Gaussians with equal covariance]\label{lem:gaussian-kl}
For $m,\wh m\in\R^d$ and positive definite $\Sigma$,
\[
    \KL\!\left(\normal{m}{\Sigma}\middle\|\normal{\wh m}{\Sigma}\right)
    =
    \frac12\norm{m-\wh m}_{\Sigma^{-1}}^2,
    \qquad
    \norm v_{\Sigma^{-1}}^2=v^\top\Sigma^{-1}v.
\]
In particular, if $\Sigma=\eta\Id$, the KL is
$\norm{m-\wh m}^2/(2\eta)$.
\end{lemma}

\begin{proof}
The log density ratio is
\[
    -\frac12\norm{x-m}_{\Sigma^{-1}}^2
    +\frac12\norm{x-\wh m}_{\Sigma^{-1}}^2.
\]
Taking expectation under $x\sim\normal{m}{\Sigma}$ cancels the covariance
terms and leaves $\frac12\norm{m-\wh m}_{\Sigma^{-1}}^2$.
\end{proof}

\begin{lemma}[Gaussian convolution score]
Let $X_t=X_0+\sqrt t Z$, where $Z\sim\normal{0}{\Id}$ is independent of $X_0$.
If $p_t$ is the density of $X_t$, then
\[
    \nabla\log p_t(x)
    =
    \frac{1}{t}\left(\E[X_0\mid X_t=x]-x\right).
\]
\end{lemma}

\begin{proof}
This is the special case $a_t=1$, $\sigma_t^2=t$ of the continuous-time Tweedie
identity \eqref{eq:continuous-tweedie}, obtained there by differentiating the
Gaussian mixture for $p_t$ and dividing by $p_t(x)$.
\end{proof}

\bibliographystyle{plain}
\bibliography{references}

\end{document}